\documentclass{article}

\usepackage[T1]{fontenc}
\usepackage[utf8]{inputenc}
\usepackage[english]{babel}

\usepackage[hyphens]{url}
\usepackage{fullpage}
\usepackage{bookmark}
\usepackage[toc,page]{appendix}

\usepackage{xcolor}
\usepackage{graphicx}
\usepackage[justification=centering]{caption}
\usepackage[justification=centering]{subcaption}
\usepackage{wrapfig}
\usepackage{float}
\usepackage{adjustbox}

\usepackage{amsmath}
\usepackage{amssymb}
\usepackage{amsthm}
\usepackage{bm}
\usepackage{bbm}
\usepackage{stmaryrd}
\usepackage{empheq}
\usepackage{fancybox}

\usepackage{listings}
\usepackage[linenumbers=false, charsperline=92]{jlcode}

\usepackage[textcolor=blue, color=purple!50]{todonotes} % add `disable`' option to disable

\usepackage{booktabs}
\usepackage{multirow}
\usepackage{array}
\usepackage{makecell}
\usepackage{csquotes}
\usepackage{authblk}

\usepackage[ruled]{algorithm2e}

\usepackage[
    backref=true,
    backend=biber,
    giveninits=true,
    style=authoryear,
    uniquename=init,
    maxbibnames=99,
    date=year,
    ]{biblatex}
    
\usepackage{subfiles}

\AtEveryBibitem{
    \clearfield{urlyear}
    \clearfield{urlmonth}
}

\graphicspath{{images/}}

\hypersetup{colorlinks}

\newtheorem{proposition}{Proposition}[section]
\newtheorem{remark}{Remark}[section]

\DeclareMathOperator*{\argmin}{argmin}
\DeclareMathOperator*{\argmax}{argmax}

\DeclareMathOperator{\conv}{conv}
\DeclareMathOperator{\proj}{proj}
\DeclareMathOperator{\dom}{dom}
\DeclareMathOperator{\softmax}{softmax}

\newcommand{\onevector}{\bm{1}}
\newcommand{\oneindicator}{\mathbbm{1}}
\newcommand{\FW}{\mathrm{FW}}
\newcommand{\bbE}{\mathbb{E}}
\newcommand{\bbP}{\mathbb{P}}
\newcommand{\bbR}{\mathbb{R}}

\newcommand{\calL}{\mathcal{L}}
\newcommand{\calN}{\mathcal{N}}

\newcommand{\calP}{\mathcal{P}}
\newcommand{\calR}{\mathcal{R}}
\newcommand{\calV}{\mathcal{V}}

\newcommand{\rmd}{\mathrm{d}}
\newcommand{\rme}{\mathrm{e}}

\newcommand{\delaycost}{c_{\text{delay}}}
\newcommand{\vehiclecost}{c_{\text{vehicle}}}

\def\jlbasicfont{\ttfamily\footnotesize\bfseries\selectfont}

\newcommand{\inline}{\jlinl}

\title{
    Learning with Combinatorial Optimization Layers: \\
    a Probabilistic Approach
}

\author[1]{Guillaume Dalle}
\author[1]{Léo Baty}
\author[1]{Louis Bouvier}
\author[1,*]{Axel Parmentier}

\affil[1]{\small CERMICS, Ecole des Ponts, Marne-la-Vallée, France}
\affil[*]{\small Corresponding author: \href{mailto:axel.parmentier@enpc.fr}{\texttt{axel.parmentier@enpc.fr}}}

\date{\today}

\begin{document}

\maketitle

\begin{abstract}
    Combinatorial optimization (CO) layers in machine learning (ML) pipelines are a powerful tool to tackle data-driven decision tasks, but they come with two main challenges.
    First, the solution of a CO problem often behaves as a piecewise constant function of its objective parameters.
    Given that ML pipelines are typically trained using stochastic gradient descent, the absence of slope information is very detrimental.
    Second, standard ML losses do not work well in combinatorial settings.
    A growing body of research addresses these challenges through diverse methods.
    Unfortunately, the lack of well-maintained implementations slows down the adoption of CO layers.

    In this paper, building upon previous works, we introduce a probabilistic perspective on CO layers, which lends itself naturally to approximate differentiation and the construction of structured losses.
    We recover many approaches from the literature as special cases, and we also derive new ones.
    Based on this unifying perspective, we present \texttt{InferOpt.jl}, an open-source Julia package that~1) allows turning any CO oracle with a linear objective into a differentiable layer, and~2) defines adequate losses to train pipelines containing such layers.
    Our library works with arbitrary optimization algorithms, and it is fully compatible with Julia's ML ecosystem.
    We demonstrate its abilities using a pathfinding problem on video game maps as guiding example, as well as three other applications from operations research.

\end{abstract}

\clearpage

\setcounter{tocdepth}{2}
\tableofcontents

\clearpage

\section{Introduction}

Machine learning (ML) and combinatorial optimization (CO) are two essential ingredients of modern industrial processes.
While ML extracts meaningful information from noisy data, CO enables decision-making in high-dimensional constrained environments.
But in many situations, combining both of these tools is necessary: for instance, we might want to generate predictions from data, and then use those predictions to make optimized decisions.
To do that, we need \emph{pipelines} that contain two types of \emph{layers}: ML layers and CO layers.

Due to their many possible applications, hybrid ML-CO pipelines currently attract a lot of research interest.
The recent reviews by \textcite{bengioMachineLearningCombinatorial2021} and \textcite{kotaryEndtoEndConstrainedOptimization2021} are excellent resources on this topic.
Unfortunately, relevant software implementations are scattered across paper-specific repositories, with few tests, minimal documentation and sporadic code maintenance.
Not only does this make comparison and evaluation difficult for academic purposes, it also hurts practitioners wishing to experiment with such techniques on real use cases.

Let us discuss a generic hybrid ML-CO pipeline, which includes a CO oracle amid several ML layers:
\begin{equation} \label{eq:hybrid_pipeline_0}
    \xrightarrow[]{\text{Input~$x$}}
    \ovalbox{\begin{Bcenter}ML layers \end{Bcenter}}
    \xrightarrow[]{\text{Objective~$\theta$}}
    \doublebox{\begin{Bcenter} CO oracle\end{Bcenter}}
    \xrightarrow[]{\text{Solution~$y$}}
    \ovalbox{More ML layers}
    \xrightarrow[]{\text{Output}}
\end{equation}
The inference problem consists in predicting an output from a given input.
It is solved online, and requires the knowledge of the parameters (weights) for each ML layer.
On the other hand, the learning problem aims at finding parameters that lead to \enquote{good} outputs during inference.
It is solved offline based on a training set that contains several inputs, possibly complemented by target outputs.

In Equation~\eqref{eq:hybrid_pipeline_0}, we use the term \emph{CO oracle} to emphasize that any algorithm may be used to solve the optimization problem, whether it relies on an existing solver or a handcrafted implementation.
Conversely, when we talk about a \emph{layer}, it is implied that we can compute meaningful derivatives using automatic differentiation (AD).
Since it may call black box subroutines, an arbitrary CO oracle is seldom compatible with AD.
And even when it is, its derivatives are zero almost everywhere, which gives us no exploitable slope information.
Therefore, according to our terminology, \emph{a CO oracle is not a layer (yet)}, and the whole point of this paper is to turn it into one.

\medskip

Modern ML libraries provide a wealth of basic building blocks that allow users to assemble and train complex pipelines.
We want to leverage these libraries to create hybrid ML-CO pipelines, but we face two main challenges.
First, while ML layers are easy to construct, it is not obvious how to transform a CO oracle into a usable layer.
Second, standard ML losses are ill-suited to our setting, because they often ignore the underlying optimization problem.

Our goal is to remove these difficulties.
We introduce \texttt{InferOpt.jl}\footnote{\url{https://github.com/axelparmentier/InferOpt.jl}}, a Julia package which 1) can turn any CO oracle into a layer with meaningful derivatives, and 2) provides structured loss functions that work well with the resulting layers.
It contains several state-of-the-art methods that are fully compatible with Julia's AD and ML ecosystem, making CO layers as easy to use as any ML layer.
To describe the available methods in a coherent manner, we leverage the unifying concept of probabilistic CO layer, hence the name of our package.

\subsection{Motivating example} \label{sec:motivation}

Let us start by giving an example of hybrid ML-CO pipeline.
Suppose we want to find shortest paths on a map, but we do not have access to an exact description of the underlying terrain.
Instead, all we have are images of the area, which give us a rough idea of the topography and obstacles.
To solve our problem, we need a pipeline comprising two layers of very different natures.
First, an image processing layer, which is typically implemented as a convolutional neural network (CNN).
The CNN is tasked with translating the images into a weighted graph.
Second, a CO layer performing shortest path computations on said weighted graph (\textit{e.g.} using Dijkstra's algorithm).

This pipeline is exactly the one considered by \textcite{vlastelicaDifferentiationBlackboxCombinatorial2020} and \textcite{berthetLearningDifferentiablePerturbed2020} for pathfinding on video game maps.
We illustrate it on Figure~\ref{fig:warcraft_pipeline}, and we describe it in more detail in Section~\ref{sec:warcraft}.
The goal is to learn appropriate weights for the CNN, so that it feeds accurate cell costs to Dijkstra's algorithm.
This is done by minimizing a loss function, such as the distance between the true optimal path and the one we predict.

\begin{figure}
    \centering
    \includegraphics[width=\textwidth]{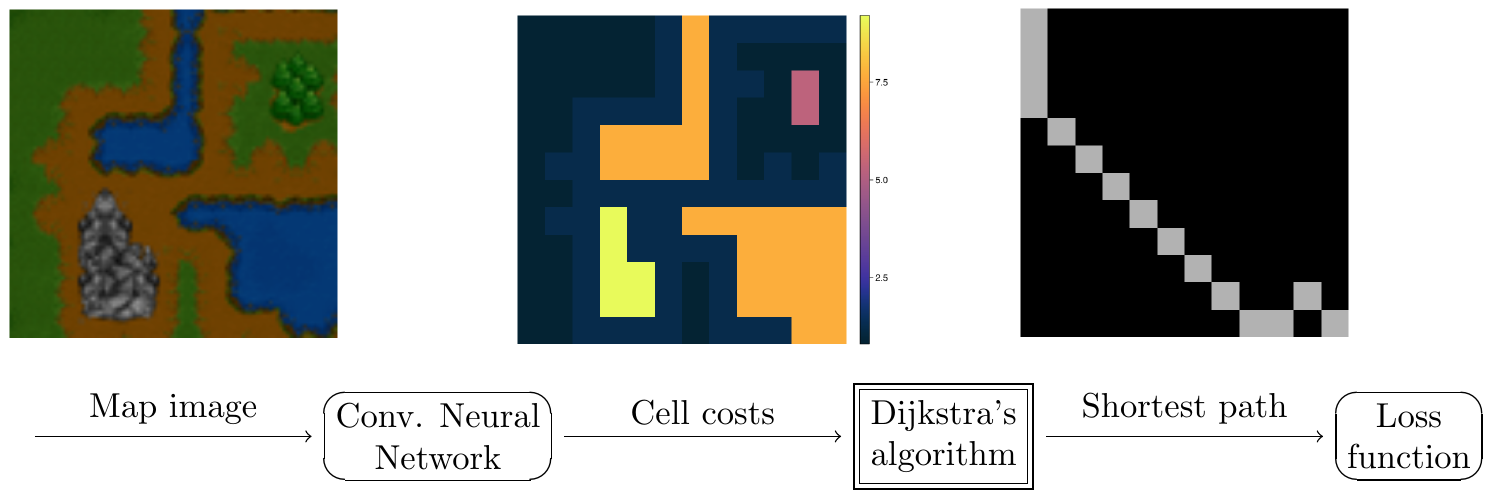}
    \caption{Pipeline for computing shortest paths on Warcraft maps -- data from \textcite{vlastelicaDifferentiationBlackboxCombinatorial2020}}
    \label{fig:warcraft_pipeline}
\end{figure}

\subsection{Our setting}

In our hybrid ML-CO pipelines, we consider CO oracles~$f$ that solve the following kind of problem:
\begin{equation} \label{eq:co_oracle}
    f: \theta \longmapsto \argmax_{v \in \calV} \theta^\top v
\end{equation}
Here, the input~$\theta \in \bbR^d$ is the \emph{objective direction}.
Meanwhile,~$\calV \subset \bbR^d$ (for \emph{vertices}) denotes a finite set of feasible solutions -- which may be exponentially large in~$d$ -- among which the optimal solution~$f(\theta)$ shall be selected.
For simplicity, we assume that~$f$ is single-valued, \textit{i.e.} that the optimal solution is unique.

The feasible set~$\calV$ and its dimension may depend on the instance.
For instance, if Equation~\eqref{eq:co_oracle} is a shortest path problem, the underlying graph may change from one input to another.
If we wanted to remain generic, we should therefore write~$\calV(x) \subset \bbR^{d(x)}$.
To keep notations simple, we omit the dependency in~$x$ whenever it is clear from the context.
Note that we could also study more general CO oracles given by
\begin{equation*}
    \argmax_{v \in \calV} \theta^\top g(v) + h(v)
\end{equation*}
where~$g$ is any function from an arbitrary finite set~$\calV$ to~$\bbR^d$, and~$h$ is any function from~$\calV$ to~$\bbR$.
As long as the objective is linear in~$\theta$, the theory we present generalizes seamlessly.
However, for ease of exposition, we keep~$g(v) = v$ and~$h(v) = 0$.
In this case, Equation~\eqref{eq:co_oracle} is equivalent to
\begin{equation*}
    \argmax_{v \in \conv(\calV)} \theta^\top v.
\end{equation*}
Indeed, when the objective is linear in~$v$, it makes no difference to optimize over the convex hull~$\conv(\calV)$ instead of optimizing over~$\calV$.

\subsubsection{From an optimization problem to an oracle}

It is important to note that the formulation~$\argmax_{v \in \calV} \theta^\top v$ is very generic.
Any linear program (LP) or integer linear program (ILP) can be written this way, as long as its feasible set is bounded.
Indeed, the optimum of an LP is always reached at a vertex of the polytope, of which there are finitely many.
Meanwhile, the optimum of an ILP is always reached at an integral point, or more precisely, at a vertex of the convex hull of the integral points.
As a result, Equation~\eqref{eq:co_oracle} encompasses a variety of well-known CO problems related to graphs (paths, flows, spanning trees, coloring), resource management (knapsack, bin packing), scheduling, \textit{etc.}
See \textcite{korteCombinatorialOptimizationTheory2006} for an overview of CO and its applications.

For every one of these problems, dedicated algorithms have been developed over the years, which sometimes exploit the domain structure better than a generic ILP solver (such as Gurobi or SCIP).
Thus, we have no interest in restricting the procedure used to compute an optimal solution: we want to pick the best algorithm for each application.
That is why the methods discussed in this paper only need to access the CO oracle~$f$ as a black box function, without making assumptions about its implementation.

\subsubsection{From an oracle to a probability distribution} \label{sec:oracle_probadist}

When using CO oracles within ML pipelines, the first challenge we face is the \emph{lack of useful derivatives}.
Training often relies on stochastic gradient descent (SGD), so we need to be able to backpropagate loss gradients onto the weights of the ML layers.
Unfortunately, since the feasible set~$\calV$ of Equation~\eqref{eq:co_oracle} is finite, the CO oracle is a piecewise constant mapping and its derivatives are zero almost everywhere.
To recover useful slope information, we seek \emph{approximate derivatives}, which is where the probabilistic approach comes into play.

To describe it, we no longer think about a CO oracle as a function returning a single element~$f(\theta)$ from~$\calV$.
Instead, we use it to define a probability distribution~$p(\cdot|\theta)$ on~$\calV$.
The naive choice would be the Dirac mass~$p(v|\theta) = \delta_{f(\theta)}(v)$, but it shares the lack of differentiability of the oracle itself.
Thus, our goal is to spread out the distribution~$p$ into an approximation~$\widehat{p}$, such that the probability mapping~$\theta \longmapsto \widehat{p}(\cdot | \theta)$ becomes smooth with respect to~$\theta$.
If we can do that, then the expectation mapping
\begin{equation} \label{eq:probabilistic_co_layer}
    \widehat{f}: \theta \longmapsto \bbE_{\widehat{p}(\cdot|\theta)}[V] = \sum_{v \in \calV} v \widehat{p}(v | \theta),
\end{equation}
where it is understood that~$V \sim \widehat{p}(\cdot|\theta)$, will be just as smooth.
This expectation mapping~$\widehat{f}$ is what we take to be our \emph{probabilistic CO layer}: see Section~\ref{sec:probabilistic} for detailed examples.
In what follows, plain letters ($p,f$) always refer to the initial CO oracle, while letters with a hat ($\widehat{p},\widehat{f}$) refer to the probabilistic CO layer that we wrap around it.

\subsubsection{From a probability distribution to a loss function}

The presence of CO oracles in ML pipelines gives rise to a second challenge: the choice of an appropriate loss function to learn the parameters.
As highlighted by \textcite{bengioMachineLearningCombinatorial2021}, this choice heavily depends on the data at our disposal.
They distinguish two main paradigms, which we illustrate using the pipeline of Figure~\ref{fig:warcraft_pipeline}.

If our dataset only contains the map images, then we are in a weakly supervised setting, which they call \emph{learning by experience} (see Section~\ref{sec:experience}).
In that case, the loss function will evaluate the solutions computed by our pipeline using the true cell costs.
On the other hand, if our dataset happens to contain precomputed targets such as the true shortest paths, then we are in a fully supervised setting, which they call \emph{learning by imitation} (see Section~\ref{sec:imitation}).
In that case, the loss function will compare the paths computed by our pipeline with the optimal ones, hoping to minimize the discrepancy.
For both of these cases, the probabilistic perspective plays an important role in ensuring smoothness of the loss.

\subsubsection{Complete pipeline}

The typical pipeline we will focus on is a special case of Equation~\eqref{eq:hybrid_pipeline_0}, which we now describe in more detail:
\begin{equation} \label{eq:hybrid_pipeline}
    \xrightarrow[]{\text{Input~$x$}}
    \ovalbox{\begin{Bcenter}ML layer~$\varphi_w$ \end{Bcenter}}
    \xrightarrow[]{\text{Objective~$\theta = \varphi_w(x)$}}
    \doublebox{\begin{Bcenter} CO oracle~$f$ \end{Bcenter}}
    \xrightarrow[]{\text{Solution~$y = f(\theta)$}}
    \ovalbox{Loss function~$\calL$}
\end{equation}
Our pipeline starts with an ML layer~$\varphi_w$, where~$w$ stands for the vector of weights.
Its role is to encode the input~$x$ into an objective direction~$\theta = \varphi_w(x)$, which is why we often refer to it as the \emph{encoder}.
Then, the CO oracle defined in Equation~\eqref{eq:co_oracle} returns an optimal solution~$y = f(\theta)$.
Finally, the loss function~$\calL$ is used during training to evaluate the quality of the solution.
If our dataset contains~$N$ input samples~$x^{(1)}, ..., x^{(N)}$, training occurs by applying SGD to the following loss minimization problem:
\begin{equation} \label{eq:loss_minimization}
    \min_w \frac{1}{N} \sum_{i=1}^N \calL\Big(\overbrace{f\big(\underbrace{\varphi_w(x^{(i)})}_{\theta^{(i)}}\big)}^{y^{(i)}}, \dots\Big).
\end{equation}
The dots~$\dots$ correspond to additional arguments that may be used by the loss.
For instance, the loss may depend on the input~$x^{(i)}$ itself, or require targets ~$\bar{t}^{(i)}$ for comparison (see Section~\ref{sec:imitation}).

\begin{remark}
    Our work focuses on individual layers and loss functions.
    Although we present various concrete examples, we do not give generic advice on how to build the whole pipeline for a specific application.
    If the use case corresponds to a \enquote{predict, then optimize} setting such as the one from Figure~\ref{fig:warcraft_pipeline}, then \textcite{elmachtoubSmartPredictThen2022} give a few useful pointers.
    If the goal is to approximate hard optimization problems with easier ones, the reader can refer to \textcite{parmentierLearningStructuredApproximations2021} for a general methodology.
\end{remark}

\subsection{Contributions}

Our foremost contribution is the open-source package \texttt{InferOpt.jl}, which is written in the Julia programming language \parencite{bezansonJuliaFreshApproach2017}.
Given a CO oracle provided as a callable object, our package wraps it into a probabilistic CO layer that is compatible with Julia's AD and ML ecosystem.
This is achieved thanks to the \texttt{ChainRules.jl}\footnote{\url{https://github.com/JuliaDiff/ChainRules.jl}} interface \parencite{framescatherinewhiteJuliaDiffChainRulesJl2022}.
Moreover, \texttt{InferOpt.jl} defines several structured loss functions, both for learning by experience and for learning by imitation.

\medskip

On top of that, we present theoretical insights that fill some gaps in previous works.
In addition to the framework of probabilistic CO layers, we propose:
\begin{itemize}
    \item A new perturbation technique designed for CO oracles that only accept objective vectors with a certain sign (such as Dijkstra's algorithm, which fails on graphs with negative edge costs): see Section~\ref{sec:multiplicative}.
    \item A way to differentiate through a large subclass of probabilistic CO layers (those that rely on convex regularization) by combining the Frank-Wolfe algorithm with implicit differentiation: see Section~\ref{sec:frank_wolfe}.
    \item A probabilistic regularization of the regret for learning by experience: see Section~\ref{sec:experience}.
    \item A generic decomposition framework for imitation losses, which subsumes most of the literature so far and suggests ways to build new loss functions: see Section~\ref{sec:loss_decomp}.
\end{itemize}

Finally, we describe numerical experiments on our motivating example of Warcraft shortest paths, as well as three combinatorial optimization problems from operations research: the stochastic vehicle scheduling problem, the single-machine scheduling problem, and the two-stage stochastic minimum weight spanning tree problem.
Here are a few highlights:
\begin{itemize}
    \item We benchmark and show the strengths of the different learning methods proposed in the literature on hybrid ML-CO pipelines.
    \item We use the pipeline of Figure~\ref{fig:warcraft_pipeline} for learning by experience on the Warcraft shortest paths problem, even though the CNN encoder has tens of thousands of parameters.
          To the best of our knowledge, previous attempts to learn such pipelines by experience were restricted to ML layers with fewer than 100 parameters.
    \item We tackle the stochastic vehicle scheduling problem efficiently by approximating it with its deterministic counterpart.
    \item We obtain a fast heuristic with state-of-the-art performance for the single machine scheduling problem, which has been a focus of the by the scheduling community for several decades.
    \item We show that our library enables the use of graph neural networks (GNNs) instead of generalized linear models (GLMs) for instance encoding, bringing additional performance on the two-stage spanning tree problem.
\end{itemize}

\subsection{Notations}

We write~$\onevector$ for the vector with all components equal to~$1$, and~$\rme_i$ for the basis vector corresponding to dimension~$i$.
The notation~$\oneindicator{\{E\}}$ corresponds to the indicator function of the set (or event)~$E$.
The operator~$\odot$ denotes the Hadamard (componentwise) product between vectors of the same size.
We use~$\Delta^d$ to refer to the unit simplex of dimension~$d$, and~$\bbE_p$ to denote an expectation with respect to the distribution~$p$.
If~$\mathcal{S}$ is a set, we write
\begin{equation} \label{eq:conv}
    \conv(\mathcal{S}) = \left\{\sum_i p_i s_i: s_i \in \mathcal{S}, p_i \geq 0, \sum_i p_i = 1\right\} = \{\bbE_p[S]: p \in \Delta^\mathcal{S}\}
\end{equation}
for its convex hull and~$\proj_{\mathcal{S}}$ for the orthogonal projection onto~$\mathcal{S}$.
If~$h$ is a real-valued function, we denote by~$\nabla_a h(x)$ the gradient of~$h$ with respect to parameter~$a$ at point~$x$ and by~$\partial_a h(x)$ its convex subdifferential (set of subgradients).
The notation~$\dom(h)$ stands for the domain of~$h$, \textit{i.e.} the set on which it takes finite values.
If~$h$ is a vector-valued function, we denote by~$J_a h(x)$ its Jacobian matrix.

\subsection{Outline}

In Section~\ref{sec:related_works}, we review the literature on differentiable optimization layers, before focusing on the Warcraft example.
Section~\ref{sec:probabilistic} introduces the family of probabilistic CO layers by splitting it into perturbed and regularized approaches.
Then, Section~\ref{sec:experience} gives tools for learning by experience, while Section~\ref{sec:imitation} discusses loss functions for learning by imitation.
Practical applications of our package are presented in Section~\ref{sec:applications}, before we conclude in Section~\ref{sec:conclusion}.
Proofs for our main theoretical results can be found in Appendix~\ref{sec:proofs}.

\section{Related work} \label{sec:related_works}

\subsection{Optimization layers in ML}

A significant part of modern ML relies on AD: see \textcite{baydinAutomaticDifferentiationMachine2018} for an overview and \textcite{griewankEvaluatingDerivativesPrinciples2008} for an in-depth treatment.
In particular, AD forms the basis of the backpropagation algorithm used to train neural networks.

\subsubsection{The notion of implicit layer}

Standard neural architectures draw from a small collection of \emph{explicit} layers \parencite{goodfellowDeepLearning2016}.
Whatever their connection structure (dense, convolutional, recurrent, \textit{etc.}) and regardless of their activation function, these layers all correspond to input-output mappings that can be expressed using an analytic formula.
This same formula is then used by AD to compute gradients.

On the other hand, the layers defined by \texttt{InferOpt.jl} are of the \emph{implicit} kind, which means they can contain arbitrarily complex iterative procedures.
While we focus here on optimization algorithms, those are not the only kind of implicit layers: fixed point iterations and differential equation solvers are also widely used, depending on the application at hand.
See the tutorial by \textcite{kolterDeepImplicitLayers2020} for more thorough explanations.

Due to the high computational cost of unrolling iterative procedures, efficient AD of implicit layers often relies on the implicit function theorem.
As long as we can specify a set of conditions satisfied by the input-output pair, this theorem equates differentiation with solving a linear system of equations.
See the Python package \texttt{jaxopt}\footnote{\url{https://github.com/google/jaxopt}} for an example implementation, and its companion paper for theoretical details \parencite{blondelEfficientModularImplicit2022}.
The recent Python package \texttt{theseus}\footnote{\url{https://github.com/facebookresearch/theseus}} showcases an application of this technique to robotics and vision \parencite{pinedaTheseusLibraryDifferentiable2022}.

\subsubsection{Convex optimization layers} \label{sec:convex_layers}

Among the early works on optimization layers for deep learning, the seminal OptNet paper by \textcite{amosOptNetDifferentiableOptimization2017} stands out.
It describes a way to differentiate through quadratic programs (QPs) by using the Karush-Kuhn-Tucker (KKT) optimality conditions and plugging them into the implicit function theorem.

More sophisticated tools exist for disciplined conic programs, such as the Python package \texttt{cvxpylayers}\footnote{\url{https://github.com/cvxgrp/cvxpylayers}} \parencite{agrawalDifferentiableConvexOptimization2019}.
The recent Julia package \texttt{DiffOpt.jl}\footnote{\url{https://github.com/jump-dev/DiffOpt.jl}} \parencite{sharmaFlexibleDifferentiableOptimization2022} extends these ideas beyond the conic case to general convex programs.
Note that both libraries only accept optimization problems formulated in a domain-specific modeling language, as opposed to arbitrary oracles.

Strong convexity makes differentiation easier because the solutions evolve smoothly as a function of the constraints and objective parameters.
In particular, this means the methods listed above return exact derivatives and do not rely on approximations.
Regrettably, this nice behavior falls apart as soon as we enter the combinatorial world.

\subsubsection{Linear optimization layers} \label{sec:linear_layers}

Let us consider an LP whose feasible set is a bounded polyhedron, also called \emph{polytope}.
It is well-known that for most objective directions, the optimal solution will be unique and located at a vertex of the polytope.
Even though LPs look like continuous optimization problems, this property shows that they are fundamentally combinatorial.
Indeed, a small change in the objective direction can cause the optimal solution to suddenly jump to another vertex, which results in a discontinuous mapping from objectives to solutions.
In fact, this mapping is piecewise constant, which means no useful differential information can come from it: its Jacobian is undefined at the jump points and zero everywhere else.

Therefore, when differentiating LPs with respect to their objective parameters, we need to resort to approximations.
\textcite{vlastelicaDifferentiationBlackboxCombinatorial2020} use interpolation to turn a piecewise constant mapping into a piecewise linear and continuous one.
However, the dominant approximation paradigm in the literature is regularization, as formalized by \textcite{blondelLearningFenchelYoungLosses2020}.

For instance, \textcite{wilderMeldingDataDecisionsPipeline2019} add a quadratic penalty to the linear objective, which allows them to reuse the QP computations of \textcite{amosOptNetDifferentiableOptimization2017}.
\textcite{mandiInteriorPointSolving2020} propose a log-barrier penalty, which lets them draw a connection with interior-point methods.
\textcite{berthetLearningDifferentiablePerturbed2020} suggest perturbing the optimization problem by adding stochastic noise to the objective direction, which is a form of implicit regularization.

When LP layers are located at the end of a pipeline, a clever choice of loss function can also simplify differentiation.
This is illustrated by the structured support vector machine (S-SVM) loss \parencite{nowozinStructuredLearningPrediction2010}, smart \enquote{predict, then optimize} (SPO+) loss \parencite{elmachtoubSmartPredictThen2022} and Fenchel-Young (FY) loss \parencite{blondelLearningFenchelYoungLosses2020}.

\subsubsection{Integer and combinatorial optimization layers}

In theory, the methods from the previous section still work in the presence of integer variables, that is, for ILPs.
To apply them, we only need to consider the polytope defined by the convex hull of integral solutions.
Alas, in the general case, there is no concise way to describe this convex hull.
This is why many authors decide to differentiate through the continuous relaxation of the ILP instead \parencite{mandiSmartPredictandOptimizeHard2020}: it is an outer approximation of the integral polytope, but it can be sufficient for learning purposes.
Some suggest taking advantage of techniques specific to integer programming, such as integrality cuts \parencite{ferberMIPaaLMixedInteger2020} or a generalization of the notion of active constraint \parencite{paulusCombOptNetFitRight2021}.

There has also been significant progress on finding gradient approximations for combinatorial problems such as ranking \parencite{blondelFastDifferentiableSorting2020} and shortest paths \parencite{parmentierLearningApproximateIndustrial2021}.
Yet these techniques are problem-specific, and therefore hard to generalize, which is why we leave them aside.

Instead, we want to allow implicit manipulation of the integral polytope itself, without making assumptions on its structure.
To achieve that, we can only afford to invoke CO oracles as black boxes \parencite{vlastelicaDifferentiationBlackboxCombinatorial2020, berthetLearningDifferentiablePerturbed2020}.
Compatibility with arbitrary algorithms is one of the fundamental tenets of \texttt{InferOpt.jl}.
The recent Python package \texttt{PyEPO}\footnote{\url{https://github.com/khalil-research/PyEPO}} \parencite{tangPyEPOPyTorchbasedEndtoEnd2022} shares our generic perspective, but it only implements a subset of \texttt{InferOpt.jl} and does not address the notion of probabilistic CO layer.

This probabilistic aspect is central to our proposal, and it is mostly inspired by the works of \textcite{blondelLearningFenchelYoungLosses2020} and \textcite{berthetLearningDifferentiablePerturbed2020}.
Another related approach is the Implicit Maximum Likelihood Estimation of \textcite{niepertImplicitMLEBackpropagating2021}, whose probabilistic formulation is slightly different and based on exponential families.

Finally, note that CO layers can be computationally heavy due to the frequent need to recompute optimal solutions at each epoch.
\textcite{mulambaContrastiveLossesSolution2021} address this issue by proposing a noise contrastive approach with solution caching, while \textcite{mandiDecisionFocusedLearningLens2022} build upon this method with a focus on learning to rank.

\subsection{Similarities and differences with reinforcement learning}

As we will see in Section~\ref{sec:experience}, learning by experience can be reminiscent of reinforcement learning (RL), which also relies on a reward or cost signal given by the environment \parencite{suttonReinforcementLearningIntroduction2018}.
Furthermore, the encoder layer of Equation~\eqref{eq:hybrid_pipeline} is similar to a parametric approximation of the value function, which forms the basis of deep RL approaches.
This prompts us to discuss a few differences between RL and the framework we study.

The standard mathematical formulation of RL is based on Markov decision processes (MDPs), where a reward and state transition are associated with each action.
Usually, the available actions are elementary decisions: pull one lever of a multi-armed bandit, cross one edge on a graph, select one move in a board game, etc.
The resulting value or policy update is local: it is specific to both the current state and the action taken.
The more reward information we gather, the more efficient learning becomes.

In our framework, the basic step is a call to the optimizer.
But combinatorial algorithms can go beyond simple actions: they often output a structured and high-dimensional solution, which aggregates many elementary decisions.
This in turn triggers a global update in our knowledge of the system, whereby the final reward is redistributed between all elementary decisions.
In a way, \emph{backpropagation through the optimizer enables efficient credit assignment}, even for sparse reward signals.

Another difference is related to the Bellman fixed point equation.
In an RL setting, the Bellman equation is used explicitly to derive parameter updates.
In our setting, the Bellman equation is used implicitly within optimizers such as Dijkstra's algorithm.

To conclude, while standard RL decomposes a policy into elementary decisions, the pipelines we study here are able to look directly for complex multistep solutions.
Note that a similar concept of \emph{option} exists in hierarchical RL \parencite{bartoRecentAdvancesHierarchical2003}: comparing both perspectives in detail would no doubt be fruitful, and we leave it for future work.

\subsection{Our guiding example: shortest paths on Warcraft maps} \label{sec:warcraft}

As a way to clarify the concepts introduced in this paper, we illustrate them on the problem of Warcraft shortest paths \parencite{vlastelicaDifferentiationBlackboxCombinatorial2020, berthetLearningDifferentiablePerturbed2020}, which was already introduced on Figure~\ref{fig:warcraft_pipeline}.
The associated dataset\footnote{\url{https://edmond.mpdl.mpg.de/dataset.xhtml?persistentId=doi:10.17617/3.YJCQ5S}}, assembled by \textcite{vlastelicaDifferentiationBlackboxCombinatorial2020}, contains randomly-generated maps similar to those from the Warcraft II video game.
Each of these maps is a red-green-blue (RGB) image of size~$ks \times ks$ containing~$k \times k$ square cells of side length~$s$.
Every cell has its own terrain type (grass, forest, water, earth, \textit{etc.}) which incurs a specific cost when game characters cross it.

The goal is to find the shortest path from the top left corner of the map to the bottom right corner.
At prediction time, cell costs are unknown, which means they must be approximated from the image alone.
Solving the problem of Warcraft shortest paths thus requires adapting the pipeline of Equation~\eqref{eq:hybrid_pipeline} as follows:
\begin{equation} \label{eq:pipeline_warcraft}
    \xrightarrow[x \in \bbR^{ks \times ks \times 3}]{\text{Map image}}
    \ovalbox{\begin{Bcenter} Conv. Neural \\ Network~$\varphi_w$ \end{Bcenter}}
    \xrightarrow[\theta \in \bbR^{k \times k}]{\text{Negative cell costs}}
    \doublebox{\begin{Bcenter} Dijkstra's \\ algorithm \\~$\argmax_{v \in \calP_k} \theta^\top v$ \end{Bcenter}}
    \xrightarrow[y \in \calP_k]{\text{Shortest path}}
    \ovalbox{\begin{Bcenter} Loss \\ function~$\calL$ \end{Bcenter}}
\end{equation}
where we have defined the set of feasible paths
\begin{equation*}
    \calP_k = \{v \in \{0, 1\}^{k \times k}: ~ \text{$v$ represents a path from~$(1, 1)$ to~$(k, k)$}\}.
\end{equation*}
Training proceeds based on the map images and (possibly) the true shortest paths or cell costs provided in the dataset.

As suggested by \textcite{vlastelicaDifferentiationBlackboxCombinatorial2020}, we design the CNN based on the first few layers of a ResNet18 \parencite{heDeepResidualLearning2016}.
We also append a negative softplus activation, in order to make sure that all outputs are negative.
The sign constraint is there to ensure that Dijkstra's algorithm will terminate, but it is not obvious why we require negative costs instead of positive ones.
The reason behind this sign switch is that Dijkstra's algorithm is a minimization oracle, whereas by convention \texttt{InferOpt.jl} works with maximization oracles.

\medskip

We now show how to implement this pipeline in Julia.
Throughout the paper, in addition to \texttt{InferOpt.jl}, the following packages are used: \texttt{Flux.jl}\footnote{\url{https://github.com/FluxML/Flux.jl}} \parencite{innesFashionableModellingFlux2018,innesFluxElegantMachine2018}, \texttt{Graphs.jl}\footnote{\url{https://github.com/JuliaGraphs/Graphs.jl}} \parencite{fairbanksJuliaGraphsGraphsJl2021}, \texttt{GridGraphs.jl}\footnote{\url{https://github.com/gdalle/GridGraphs.jl}}, \texttt{Metalhead.jl}\footnote{\url{https://github.com/FluxML/Metalhead.jl}} and \texttt{Zygote.jl}\footnote{\url{https://github.com/FluxML/Zygote.jl}} \parencite{innesDonUnrollAdjoint2019}, along with the base libraries \texttt{LinearAlgebra}  and \texttt{Statistics}.
Code sample~\ref{code:warcraft_encoder} creates the CNN encoder layer.
Meanwhile, Code sample~\ref{code:warcraft_maximizer} shows how to define the Dijkstra oracle and the true cost function.

\begin{figure}
    \begin{minipage}{0.48\textwidth}
\begin{lstlisting}[
            caption={CNN encoder for Warcraft},
            label={code:warcraft_encoder},
        ]
using Flux, Metalhead, Statistics

resnet18 = ResNet(
    18; pretrain=false, nclasses=1
)

warcraft_encoder = Chain(
    resnet18.layers[1][1:4],
    AdaptiveMaxPool((12, 12)),
    x -> mean(x; dims=3),
    x -> dropdims(x; dims=(3, 4)),
    x -> -softplus.(x)
)
\end{lstlisting}
    \end{minipage}
    \hfill
    \begin{minipage}{0.48\textwidth}
\begin{lstlisting}[
            caption={Dijkstra optimizer for Warcraft},
            label={code:warcraft_maximizer},
        ]
using Graphs, GridGraphs, LinearAlgebra
using GridGraphs: QUEEN_DIRECTIONS

function warcraft_maximizer(theta)
    g = GridGraph(
        -theta;
        directions=QUEEN_DIRECTIONS
    )
    path = grid_dijkstra(g, 1, nv(g))
    y = path_to_matrix(g, path)
    return y
end

function warcraft_cost(y; theta_ref)
    return dot(y, theta_ref)
end
\end{lstlisting}
    \end{minipage}
\end{figure}

Finally, Code sample~\ref{code:warcraft_pipeline} demonstrates the full prediction and optimization pipeline.
It assumes that we have already parsed the Warcraft dataset into three vectors:
\begin{itemize}
    \item \inline{images}, whose elements are three-dimensional arrays representing map images;
    \item \inline{cells}, whose elements are floating-point matrices representing true cell costs;
    \item \inline{paths}, whose elements are binary matrices representing true shortest paths.
\end{itemize}

\begin{figure}
    \centering
    \begin{minipage}{0.8\linewidth}
\begin{lstlisting}[
            caption={Full pipeline for Warcraft shortest paths},
            label={code:warcraft_pipeline}
        ]
x, theta_ref, y_ref = images[1], cells[1], paths[1]
theta = warcraft_encoder(x)
y = warcraft_maximizer(theta)
c = warcraft_cost(y; theta_ref=theta_ref)
\end{lstlisting}
    \end{minipage}
\end{figure}

These functions do not rely on \texttt{InferOpt.jl}, but they will be used throughout the paper inside the differentiable wrappers provided by the package.
Note that some snippets shown here have been shortened for clarity.
Please refer to the documentation of \texttt{InferOpt.jl} and satellite packages for actual runnable examples.

\section{Probabilistic CO layers} \label{sec:probabilistic}

In this section, we focus on the CO oracle~$f$ defined in Equation~\eqref{eq:co_oracle}, which is piecewise constant.
By adopting a probabilistic point of view, we construct several smooth approximations~$\widehat{f}$, which can be computed and differentiated based solely on calls to~$f$.

\subsection{The expectation of a differentiable probability distribution} \label{sec:proba_dist}

As announced in Section~\ref{sec:oracle_probadist}, a probabilistic CO layer works in two steps.
First, it constructs a probability distribution~$\widehat{p}(\cdot | \theta) \in \Delta^\calV$.
Second, it returns the expectation~$\widehat{f}(\theta) = \bbE_{\widehat{p}(\cdot | \theta)}[V] \in \conv(\calV)$.
Figure~\ref{fig:polytope} illustrates this behavior on a two-dimensional polytope, with a maximization problem defined by the vector~$\theta$ (black arrow).
While the CO oracle outputs a single optimal vertex (red square), the probabilistic CO layer defines a distribution on all the vertices (light blue circles).
Its output (dark blue hexagon) is a convex combination of the vertices with nonzero weights, which belongs to the convex hull of~$\calV$ (gray surface).
\begin{figure}
    \centering
    \includegraphics[width=0.8\linewidth,trim={80 55 0 5},clip]{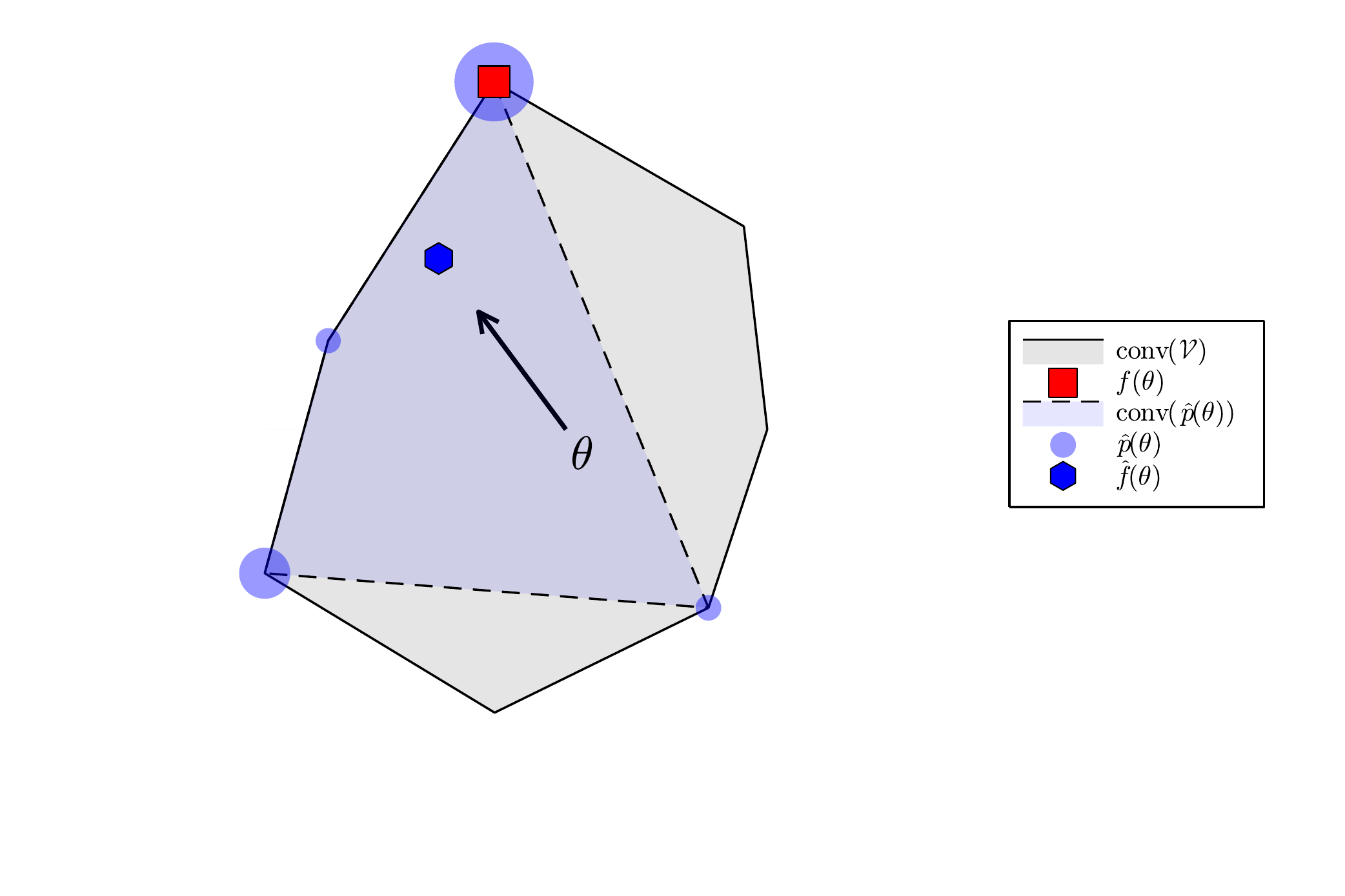}
    \caption{Effect of a probabilistic CO layer}
    \label{fig:polytope}
\end{figure}

For such a layer to be useful in our setting, we impose several conditions.
First, all computations must only require calls to the CO oracle~$f$.
Second, the expectation~$\bbE_{\widehat{p}(\cdot | \theta)}[V]$ must be tractable (whether it is with an explicit formula, Monte-Carlo sampling, variational inference, \textit{etc.}).
Third, the mapping~$\theta \longmapsto \widehat{p}(\cdot | \theta)$ must be differentiable.
If the last condition is satisfied, then the Jacobian of~$\widehat{f}$ is easily deduced from Equation~\eqref{eq:probabilistic_co_layer}:
\begin{equation} \label{eq:probabilistic_co_layer_jacobian}
    J_\theta \widehat{f}(\theta) = J_\theta \bbE_{\widehat{p}(\cdot | \theta)}[V] = \sum_{v \in \calV} v \nabla_\theta \widehat{p}(v | \theta)^\top
\end{equation}
Let us give an example where analytic formulas exist.
If~$\calV = \{\rme_1, ..., \rme_d\}$ is the set of basis vectors, then its convex hull~$\conv(\calV) = \Delta^d$ is the unit simplex of dimension~$d$.
Given an objective direction~$\theta$, solving~$\argmax_{v \in \calV} \theta^\top v$ yields the basis vector~$f(\theta) = \rme_i$ where~$i$ is the index maximizing~$\theta_i$.

We want a probability distribution that evolves smoothly with~$\theta$, so we need to spread out the naive Dirac mass~$\delta_{f(\theta)}(\cdot)$ by putting weight on several vertices instead of just one.
Let us assign to each vertex a probability that depends on its level of optimality in the optimization problem~$\argmax_{v \in \calV} \theta^\top v$, that is, on its inner product with~$\theta$.
The Boltzmann distribution is a natural candidate, leading to~$\widehat{p}(\rme_i | \theta) \propto e^{\theta^\top \rme_i} = e^{\theta_i}$.
Computing the expectation reveals a well-known operation:
\begin{equation*}
    \widehat{f}(\theta) = \bbE_{\widehat{p}(\cdot | \theta)}[V] = \sum_{i=1}^{d} \frac{e^{\theta_i}}{\sum_{j=1}^{d} e^{\theta_j}} \rme_i = \softmax(\theta)
\end{equation*}
Unlike the \enquote{hardmax} function~$f$, the softmax function~$\widehat{f}$ is differentiable, which justifies its frequent use as an activation function in classification tasks.

While the Boltzmann distribution can be used for specific sets~$\calV$, in the general case, it has an intractable normalizing constant.
This means we would need Markov chain Monte-Carlo (MCMC) methods to compute derivatives, which undermines the simplicity we are looking for.
Fortunately, Sections~\ref{sec:additive} and~\ref{sec:multiplicative} present other probability distributions which can be easily approximated through sampling.

\subsection{Regularization as another way to define a distribution} \label{sec:regularization}

Although this probabilistic point of view was recently put forward by \textcite{berthetLearningDifferentiablePerturbed2020}, the most popular paradigm in the literature remains regularization \parencite{blondelLearningFenchelYoungLosses2020}.
Instead of using the CO oracle~\eqref{eq:co_oracle}, regularization solves a different problem:
\begin{equation} \label{eq:regularized_co_layer}
    \widehat{f}_\Omega: \theta \longmapsto \argmax_{\mu \in \dom(\Omega)} \theta^\top \mu - \Omega(\mu)
\end{equation}
where~$\Omega: \bbR^d \to \bbR$ is a smooth and convex function that penalizes the output~$\mu$.
Usually,~$\Omega$ is chosen to enforce~$\Omega(\mu) = +\infty$ whenever~$\mu \notin \conv(\calV)$, which means~$\dom(\Omega) \subseteq \conv(\calV)$.
Remember that since~$\calV$ is finite,~$\conv(\calV)$ is a polytope whose vertices form a subset of~$\calV$.

The change of notation from~$v$ (vertex) to~$\mu$ (moment) stresses the fact that~$v$ is an element of~$\calV$, while~$\mu$ belongs to~$\dom(\Omega) \subseteq \conv(\calV)$.
By Equation~\eqref{eq:conv}, any feasible~$\mu$ is the expectation of some distribution over~$\calV$, hence our choice of letter.
This means we can write~$\widehat{f}_\Omega(\theta)$ as a convex combination of the elements of~$\calV$, whose weights are then interpreted as probabilities~$\widehat{p}_\Omega(\cdot|\theta)$.
In other words, the two perspectives are not opposed: regularization is just another way to define a probability distribution.

\medskip

For instance, if we go back to the concrete example from Section~\ref{sec:proba_dist} and select~$\Omega(\mu) = \sum_i \mu_i \log \mu_i$ (the negative Shannon entropy), we find once again that~$\widehat{f}_\Omega(\theta) = \softmax(\theta)$.
But the case of the unit simplex is very peculiar, because for any~$\mu \in \conv(\calV)$, the convex decomposition of~$\mu$ onto the vertices~$\calV$ is unique.
In other words, the correspondence between regularizations~$\Omega$ and probability mappings~$\widehat{p}_\Omega$ is one-to-one.

This does not hold for arbitrary polytopes.
As a result, we need to be more specific in how we choose the convex decomposition.
In particular, we need the weights to be differentiable, in order to compute the Jacobian with Equation~\eqref{eq:probabilistic_co_layer_jacobian}.
Section~\ref{sec:frank_wolfe} describes one possible approach, which relies on the Frank-Wolfe algorithm and implicit differentiation.

Conversely, the probability distributions given in Sections~\ref{sec:additive} and~\ref{sec:multiplicative} also give rise to an implicit regularization, which can be expressed using Fenchel conjugates.
This point of view is especially useful when we want to combine these layers with Fenchel-Young losses (Section~\ref{sec:fyl}).
In essence, we claim that \emph{probabilistic CO layers and regularization are two sides of the same coin}.

\subsection{Collection of probabilistic CO layers}

Our package implements various flavors of probabilistic CO layers, which are summed up in Table~\ref{tab:layers}.
Code sample~\ref{code:layers_operations} displays the operations they all support.

\begin{remark}
    We also implement an \inline{Interpolation} layer, corresponding to the piecewise linear interpolation of \textcite{vlastelicaDifferentiationBlackboxCombinatorial2020}.
    However, to the best of our knowledge it cannot be cast as a probabilistic CO layer, so it does not support as many operations, and we only mention it for benchmarking purposes.
\end{remark}

We now present each row of Table~\ref{tab:layers} in more detail.
The main goal of Sections~\ref{sec:additive},~\ref{sec:multiplicative} and~\ref{sec:frank_wolfe} is to explain how~$\widehat{p}(\cdot | \theta)$ and~$\widehat{f}(\theta)$ are computed, as well as their derivatives.
We also draw connections with the regularization paradigm, which will ease the introduction of Fenchel-Young losses in Section~\ref{sec:fyl}.
A hasty reader can safely skip to Section~\ref{sec:experience}.

\begin{table}[ht]
    \centering
    \begin{adjustbox}{width=\linewidth}
        \begin{tabular}{@{}cccc@{}}
            \toprule
            \textbf{Layer}                   & \textbf{Notations}                                                & \textbf{Probability}~$\widehat{p}(\cdot|\theta)$                              & \textbf{Regularization}     \\ \midrule
            \texttt{PerturbedAdditive}       & ~$\widehat{p}_\varepsilon^+$,~$\widehat{f}_\varepsilon^+$         & Explicit:~$f(\theta + \varepsilon Z)$                                         & Implicit: Fenchel conjugate \\
            \texttt{PerturbedMultiplicative} & ~$\widehat{p}_\varepsilon^\odot$,~$\widehat{f}_\varepsilon^\odot$ & Explicit: ~$f(\theta \odot e^{\varepsilon Z - \varepsilon^2 \onevector / 2})$ & Implicit: Fenchel conjugate \\
            \texttt{RegularizedGeneric}      & ~$\widehat{p}_\Omega^\FW$,~$\widehat{f}_\Omega^\FW$               & Implicit: Frank-Wolfe weights                                                 & Explicit: function~$\Omega$ \\ \bottomrule
        \end{tabular}
    \end{adjustbox}
    \caption{Probabilistic CO layers and their defining features}
    \label{tab:layers}
\end{table}

\begin{figure}
    \centering
    \begin{minipage}{0.8\linewidth}
\begin{lstlisting}[
            caption={Supported operations for a probabilistic CO layer},
            label={code:layers_operations}
        ]
using InferOpt, Zygote

p = compute_probability_distribution(layer, theta)
rand(p)
compute_expectation(p)

y = layer(theta)  # equal to the expectation of p
Zygote.jacobian(layer, theta)
\end{lstlisting}
    \end{minipage}
\end{figure}

\subsubsection{Additive perturbation} \label{sec:additive}

A natural way to define a distribution on~$\calV$ is to solve~\eqref{eq:co_oracle} with a stochastic perturbation of the objective direction~$\theta$.
\textcite{berthetLearningDifferentiablePerturbed2020} suggest the following additive perturbation mechanism:
\begin{equation} \label{eq:perturbed_lp}
    \widehat{f}_\varepsilon^+(\theta) = \bbE \left[\argmax_{v \in \calV} (\theta + \varepsilon Z)^\top v \right] = \bbE \left[f(\theta + \varepsilon Z)\right]
\end{equation}
where~$\varepsilon > 0$ controls the amplitude of the perturbation,~$Z \sim \calN(0, I)$ is a standard Gaussian vector and the expectation is taken with respect to~$Z$ unless otherwise specified.
Choosing~$\varepsilon$ is a trade-off between smoothness (large~$\varepsilon$) and accuracy of the approximation (small~$\varepsilon$).
The associated probability distribution on~$\calV$ can be described explicitly:
\begin{equation}
    \widehat{f}_\varepsilon^+(\theta) = \sum_{v \in \calV} v \widehat{p}_\varepsilon^+(v | \theta) \qquad \text{with} \qquad \widehat{p}_\varepsilon^+(v|\theta) = \bbP \left(f(\theta + \varepsilon Z) = v\right).
\end{equation}
Meanwhile, Proposition~\ref{prop:additive_perturbation} allows us to compute differentials.
Although the expectations cannot be expressed in closed form, they can be estimated using~$M$ Monte-Carlo samples~$Z_1,...,Z_M \sim \mathcal{N}(0, I)$.
Increasing~$M$ yields smoother approximations but makes the complexity grow linearly.

\begin{proposition}[Differentiating through an additive perturbation \parencite{berthetLearningDifferentiablePerturbed2020}] \label{prop:additive_perturbation}
    We have:
    \begin{align*}
        \nabla_\theta \widehat{p}_\varepsilon^+(v|\theta)
         & = \frac{1}{\varepsilon} \bbE \left[\oneindicator{\{f(\theta + \varepsilon Z) = v\}} Z\right] \\
        J_\theta \widehat{f}_\varepsilon^+(\theta)
         & = \frac{1}{\varepsilon} \bbE \left[f(\theta + \varepsilon Z) Z^\top\right]
    \end{align*}
\end{proposition}

\begin{proof}
    See Appendix~\ref{sec:proof_additive_perturbation}.
\end{proof}

In order to recover the regularization associated with~$p_\varepsilon^+$, we leverage convex conjugation.
Let~$F_\varepsilon^+$ be the function defined by
\begin{equation*}
    F_\varepsilon^+(\theta) = \bbE \left[\max_{v \in \calV} (\theta + \varepsilon Z)^\top v \right]
\end{equation*}
and let~$\Omega_\varepsilon^+ = (F_\varepsilon^+)^*$ denote its Fenchel conjugate.

\begin{proposition}[Regularization associated with an additive perturbation \parencite{berthetLearningDifferentiablePerturbed2020}] \label{prop:additive_perturbation_omega}
    The function~$\Omega_\varepsilon^+$ is convex, it satisfies~$\dom(\Omega_\varepsilon^+) \subset \conv(\calV)$ and
    \begin{equation*}
        \widehat{f}_\varepsilon^+(\theta) = \argmax_{\mu \in \conv(\calV)} \theta^\top \mu - \Omega_\varepsilon^+(\mu) = \widehat{f}_{\Omega_\varepsilon^+}(\theta).
    \end{equation*}
\end{proposition}

\begin{proof}
    See Appendix~\ref{sec:proof_additive_perturbation_omega}.
\end{proof}

Code sample~\ref{code:perturbed} shows how this translates into \texttt{InferOpt.jl} syntax.

\subsubsection{Multiplicative perturbation} \label{sec:multiplicative}

Since the Gaussian distribution puts mass on all of~$\bbR^d$, it can happen that some components of~$\theta + \varepsilon Z$ switch their sign with respect to~$\theta$.
This may cause problems whenever the CO oracle for~$f$ has sign-dependent behavior.
For instance, Dijkstra's algorithm for shortest paths requires all the edges of a graph to have a positive cost.
In those cases, we need a sign-preserving kind of perturbation.
Changing the distribution of~$Z$ to make it positive almost surely is not the right answer because it would bias the pipeline, leading to~$\bbE[\theta + \varepsilon Z] > \theta$ for the componentwise order.
So instead of being additive, the perturbation becomes multiplicative:
\begin{equation} \label{eq:perturbed_lp_sign}
    \widehat{f}_\varepsilon^\odot(\theta) = \bbE \left[\argmax_{v \in \calV} \left( \theta \odot e^{\varepsilon Z - \varepsilon^2 \onevector / 2}\right)^\top v \right] = \bbE \left[f\left(\theta \odot e^{\varepsilon Z - \varepsilon^2 \onevector / 2}\right)\right]
\end{equation}
Here,~$\odot$ denotes the Hadamard product, and the exponential is taken componentwise.
Since~$\bbE[e^{\varepsilon Z}] = e^{\varepsilon^2 \onevector / 2} \neq 1$, we add a correction term in the exponent to remove any bias:~$\bbE [\theta \odot e^{\varepsilon Z - \varepsilon^2 \onevector / 2}] = \theta$.
As before, the associated probability distribution is easy to describe:
\begin{equation}
    \widehat{f}_\varepsilon^\odot(\theta) = \sum_{v \in \calV} v \widehat{p}_\varepsilon^\odot(v | \theta) \qquad \text{with} \qquad \widehat{p}_\varepsilon^\odot(v | \theta) = \bbP \left(f\left(\theta \odot e^{\varepsilon Z - \varepsilon^2 \onevector / 2}\right) = v\right).
\end{equation}
And Proposition~\ref{prop:multiplicative_perturbation} provides differentiation formulas that are very similar to the additive case.

\begin{proposition}[Differentiating through a multiplicative perturbation] \label{prop:multiplicative_perturbation}
    We have:
    \begin{align*}
        \nabla_\theta \widehat{p}_\varepsilon^\odot(v|\theta)
         & = \frac{1}{\varepsilon \theta} \odot \bbE \left[\oneindicator{\left\{f\left(\theta \odot e^{\varepsilon Z - \varepsilon^2 \onevector / 2}\right) = v\right\}} Z\right] \\
        J_\theta \widehat{f}_\varepsilon^\odot(\theta)
         & = \frac{1}{\varepsilon \theta} \odot \bbE \left[f\left(\theta \odot e^{\varepsilon Z - \varepsilon^2 \onevector / 2}\right) Z^\top\right]
    \end{align*}
\end{proposition}

\begin{proof}
    See Appendix~\ref{sec:proof_multiplicative_perturbation}.
\end{proof}

As far as regularization is concerned, we need a slight tweak compared to the additive case.
Let~$F_\varepsilon^\odot$ be the function defined by
\begin{equation*}
    F_\varepsilon^\odot(\theta) = \bbE \left[\max_{v \in \calV} \left(\theta \odot e^{\varepsilon Z - \varepsilon^2 \onevector / 2}\right)^\top v \right]
\end{equation*}
and let~$\Omega_\varepsilon^\odot = (F_\varepsilon^\odot)^*$ denote its Fenchel conjugate.
We define
\begin{equation*}
    \widehat{f}_\varepsilon^{\odot \mathrm{scaled}}(\theta) = \bbE \left[ e^{\varepsilon Z - \varepsilon^2 \onevector / 2} \odot f \left(\theta \odot e^{\varepsilon Z - \varepsilon^2 \onevector / 2}\right) \right]
\end{equation*}

\begin{proposition}[Regularization associated with a multiplicative perturbation] \label{prop:multiplicative_perturbation_omega}
    The function~$\Omega_\varepsilon^\odot$ is convex and satisfies
    \begin{equation*}
        \widehat{f}_\varepsilon^{\odot \mathrm{scaled}}(\theta) = \argmax_{\mu \in \dom(\Omega_\varepsilon^\odot)} \theta^\top \mu - \Omega_\varepsilon^\odot(\mu) = \widehat{f}_{\Omega_\varepsilon^\odot}(\theta).
    \end{equation*}
\end{proposition}

Unlike in the additive case, it is not~$\widehat{f}_\varepsilon^\odot$ itself that can be viewed as the product of regularization with~$\Omega_\varepsilon^\odot$, but~$\widehat{f}_\varepsilon^{\odot \mathrm{scaled}}$.
Furthermore, this time we have~$\dom(\Omega_\varepsilon^\odot) \not\subseteq \conv(\calV)$.

\begin{proof}
    See Appendix~\ref{sec:proof_multiplicative_perturbation_omega}.
\end{proof}

Code sample~\ref{code:perturbed} shows how this translates into \texttt{InferOpt.jl} syntax.

\subsubsection{Generic regularization} \label{sec:frank_wolfe}

We now switch our focus to the case of an explicit regularization~$\Omega$.
Provided the regularization is convex and smooth, approximate computation of~$\widehat{f}_\Omega(\theta)$ is made possible by the Frank-Wolfe algorithm \parencite{frankAlgorithmQuadraticProgramming1956}.
This algorithm is interesting for two reasons.
First, it only requires access to the CO oracle~$f$ and the gradient of~$\Omega$.
Second, its output is expressed as a convex combination of only a few polytope vertices \parencite{jaggiRevisitingFrankWolfeProjectionFree2013}.
In other words, the Frank-Wolfe algorithm does not just return a single point~$\widehat{f}_\Omega(\theta) \in \conv(\calV)$: it also defines a sparse probability distribution~$\widehat{p}_\Omega^{\FW}(\cdot|\theta)$ over the vertices~$\calV$ such that
\begin{equation*}
    \widehat{f}_\Omega(\theta) = \sum_{v \in \calV} v \widehat{p}_\Omega^{\FW}(v|\theta).
\end{equation*}
This distribution is called sparse because most of the weights are actually zero.
Note that~$\widehat{p}_\Omega^{\FW}(\cdot|\theta)$ is not uniquely specified by the regularization~$\Omega$, but instead depends on the precise implementation of the Frank-Wolfe algorithm (initialization, step size, convergence criterion, \textit{etc.}).
In particular, the number of atoms in the distribution is upper-bounded by the number of Frank-Wolfe iterations.

\medskip

As pointed out by \textcite[Appendix C]{blondelEfficientModularImplicit2022}, there exists a function~$g(p, \theta)$ defined on~$\Delta^\calV \times \bbR^d$ such that~$\widehat{p}_\Omega^{\FW}(\cdot|\theta)$ is a fixed point of its projected gradient operator~$p \longmapsto \proj_{\Delta^\calV}(p - \nabla_p g(p, \theta))$.
Since the orthogonal projection onto the simplex~$\Delta^\calV$ is itself differentiable \parencite{martinsSoftmaxSparsemaxSparse2016}, we can apply the implicit function theorem to this fixed point equation.
Doing so yields gradients~$\nabla_\theta \widehat{p}_\Omega(v|\theta)$ that we use to compute a Jacobian for~$\widehat{f}_\Omega(\theta)$.
Again, by sparsity, this sum only has a few non-zero terms, which makes it tractable:
\begin{equation}
    J_\theta \widehat{f}_\Omega(\theta) = \sum_{v \in \calV} v \nabla_\theta \widehat{p}_\Omega^{\FW}(v|\theta)^\top.
\end{equation}
Among all the possible functions~$\Omega$, the quadratic penalty~$\Omega(\mu) = \tfrac12 \lVert \mu \rVert^2$ is particularly interesting.
It gives rise to the SparseMAP method \parencite{niculaeSparseMAPDifferentiableSparse2018}, whose name comes from the sparsity of the Euclidean projection onto a polytope:
\begin{equation*}
    \widehat{f}_\Omega(\theta) = \argmax_{\mu \in \conv(\calV)} \left\{\theta^\top \mu - \frac12 \lVert \mu \rVert^2 \right\} = \argmin_{\mu \in \conv(\calV)} \lVert \mu - \theta \rVert^2.
\end{equation*}
This is the one we used for the example of Code sample~\ref{code:regularized_generic}.
Our implementation relies on the recent package \texttt{FrankWolfe.jl}\footnote{\url{https://github.com/ZIB-IOL/FrankWolfe.jl}} \parencite{besanconFrankWolfeJlHighPerformance2022}.

\begin{remark}
    \textcite{blondelLearningFenchelYoungLosses2020} also suggest distribution regularization, whereby~$\Omega(\mu)$ is defined through a generalized entropy~$H(p)$ on~$\Delta^{\calV}$:
    \begin{equation*}
        \Omega(\mu) = -\max_{p \in \Delta^\calV} H(p) \quad \text{s.t.} \quad \bbE_p[V] = \mu.
    \end{equation*}
    Distribution regularization can only be computed explicitly for certain entropies~$H$ (Shannon entropy, Gini index) and certain polytopes~$\conv(\calV)$ (unit simplex, permutahedron, spanning trees, \textit{etc.}).
    In each case, a custom combinatorial algorithm is required.
    Since we aim for a generic approach, we only consider mean regularization, which is defined directly on the expectation~$\mu$.
\end{remark}

\begin{figure}
    \centering
    \begin{minipage}{0.50\linewidth}
\begin{lstlisting}[
            caption={Probabilistic CO layers defined by perturbation},
            label={code:perturbed}
        ]
using InferOpt

perturbed_add = PerturbedAdditive(
    warcraft_maximizer;
    epsilon=0.5, nb_samples=10
)

perturbed_mult = PerturbedMultiplicative(
    warcraft_maximizer;
    epsilon=0.5, nb_samples=10
)
\end{lstlisting}
    \end{minipage}
    \hfill
    \begin{minipage}{0.45\linewidth}
\begin{lstlisting}[
            caption={Probabilistic CO layer defined by regularization},
            label={code:regularized_generic}
        ]
using InferOpt

regularized = RegularizedGeneric(
    warcraft_maximizer;
    omega=y -> 0.5 * sum(y .^ 2),
    omega_grad=y -> y
)
\end{lstlisting}
    \end{minipage}
\end{figure}

\subsection{The case of inexact CO oracles}

In our discussion so far, an implicit assumption was that the CO oracle~$f$ returns an exact solution to Equation~\eqref{eq:co_oracle}.
For most polynomial problems (as well as some NP-hard problems which are tractable in practice), this is perfectly reasonable.
But in some cases, exact solutions are too expensive to compute.
Then, our CO oracle may only be able to return an inexact solution, for instance because branch \& bound has to be interrupted before the whole tree can be explored.
What kind of impact does this have on the precision of the computed Jacobian?

Let us denote by~$f$ a hypothetical exact oracle, and by~$g$ an inexact oracle.

\begin{proposition}[Jacobian precision for inexact oracles -- perturbed case]
    \label{prop:inexact_oracles}
    Suppose we use~$g$ instead of~$f$ with additive (resp. multiplicative) perturbation.
    Then the error on the Jacobian of the probabilistic CO layer satisfies:
    \begin{align*}
        \left\lVert J_\theta \widehat{g}_\varepsilon^+(\theta) - J_\theta \widehat{f}_\varepsilon^+(\theta) \right\lVert^2
         & \leq \frac{\sqrt{d}}{\varepsilon} \lVert g - f \rVert_{\infty}                                \\
        \left\lVert J_\theta \widehat{g}_\varepsilon^\odot(\theta) - J_\theta \widehat{f}_\varepsilon^\odot(\theta) \right\lVert^2
         & \leq \frac{\sqrt{d}}{\varepsilon \min_i \lvert \theta_i \rvert} \lVert g - f \rVert_{\infty}.
    \end{align*}
\end{proposition}

While requiring the inexact oracle~$g$ to be uniformly close to~$f$ is quite restrictive, this result does provide heuristic justification for the use of inexact oracles in practice.

\begin{proof}
    See Appendix~\ref{sec:proof_inexact_oracles}.
\end{proof}

\section{Learning by experience} \label{sec:experience}

Now that we have seen several ways to construct probabilistic CO layers, we turn to the definition of an appropriate loss function.
Let us start with learning by experience, which takes place when we only have access to input samples without target outputs.
In that case, Equation~\eqref{eq:loss_minimization} simplifies as
\begin{equation}\label{eq:pb_learning_by_experience}
    \min_w \frac{1}{N} \sum_{i=1}^N \calL\Big(f\big(\varphi_w(x^{(i)})\big)\Big).
\end{equation}
As we will see below, the \emph{regret}, which is the natural choice of loss, does not yield interesting gradients.
That is why we propose a family of \emph{smooth regret surrogates} derived from our probabilistic CO layers, and explain how to differentiate them.
While similar losses have been hinted at in previous works, to the best of our knowledge, our general point of view is new.

To make notations lighter, we restrict ourselves to a single input~$x$.
Furthermore, we write losses as functions of~$\theta$ instead of~$w$.
Indeed, our losses do not just rely on~$y = f(\theta)$: they use~$f$ as an ingredient internally.
In practice, we leave it to AD to exploit the relation~$\theta = \varphi_w(x)$ in order to compute gradients with respect to~$w$.

\subsection{Minimizing a smooth regret surrogate}

When we learn by experience, the problem statement usually includes a cost function~$c: \calV \to \bbR$, and we want our pipeline to generate solutions that are as cheap as possible.
Internally, this cost function may use parameters that are unknown to us at prediction time: typically, it may assess the quality of our solution using the true objective direction~$\bar{\theta}$.
It may be useful to think about~$c$ as the feedback provided by an outside evaluator, rather than a function we implement ourselves.

The natural loss to minimize is the cost incurred by our prediction pipeline, also called regret:
\begin{equation} \label{eq:regret}
    \calR(\theta) = c(f(\theta)).
\end{equation}
This function relies on the CO oracle~$f$, which is piecewise constant.
Our spontaneous impulse would be to replace the CO oracle~$f$ with a probabilistic CO layer~$\widehat{f}$, thus minimizing~$c(\widehat{f}(\theta))$.
Unfortunately, the cost function~$c$ is not necessarily smooth either.
To make matters worse,~$c$ may only be defined on vertices~$v \in \calV$, and not on general convex combinations~$\mu \in \conv(\calV)$.

The solution we propose relies on the \emph{pushforward measure} (also called image measure) of~$\widehat{p}(\cdot | \theta)$ with respect to the function~$c$.
Recall that a probabilistic CO layer is defined by~$\widehat{f}(\theta) = \bbE_{\widehat{p}(\cdot | \theta)}[V]$.
To compose it with an arbitrary cost, instead of applying~$c$ outside the expectation, we apply it inside the expectation.
In other words, we first push the measure~$\widehat{p}(\cdot | \theta)$ forward through the function~$c$, before taking the expectation.
This gives rise to the notion of \emph{expected regret}:
\begin{equation} \label{eq:expected_regret}
    \calR_{\widehat{p}}(\theta) = \bbE_{\widehat{p}(\cdot|\theta)}[c(V)]
\end{equation}
By integration, this loss is just as smooth as the probability mapping~$\theta \longmapsto \widehat{p}(\cdot|\theta)$, which means we can compute its gradient easily.
We therefore suggest using the expected regret~$\calR_{\widehat{p}}$ that stems from the probabilistic CO layers~$\widehat p_\varepsilon^+$,~$\widehat{p}_\varepsilon^\odot$, and~$\widehat{p}_{\Omega}^{\mathrm{FW}}$ defined in Section~\ref{sec:probabilistic}.

Note that if~$c$ is linear and defined on all of~$\conv(\calV)$, then~$\bbE_{\widehat{p}(\cdot|\theta)}[c(V)] = c\left(\bbE_{\widehat{p}(\cdot|\theta)}[V]\right)$ and the two quantities coincide.
Furthermore, if~$c$ is convex, then~$\bbE_{\widehat{p}(\cdot|\theta)}[c(V)] \geq c\left(\bbE_{\widehat{p}(\cdot|\theta)}[V]\right)$ by Jensen's inequality, which means the expected regret is an upper bound.

\medskip

Code sample~\ref{code:regrets} demonstrates how to define an expected regret from probabilistic CO layers, while Code sample~\ref{code:regret_operations} shows that we can compute and differentiate it automatically.
Finally, Code sample~\ref{code:experience_learning} displays a complete program for learning by experience.
The rest of this section explains how to compute derivatives of the expected regret~$\calR_{\widehat{p}}$ and can be skipped without danger.

\begin{figure}
    \centering
    \begin{minipage}{0.48\textwidth}
\begin{lstlisting}[
            caption={Expected regrets associated with probabilistic CO layers},
            label={code:regrets}
        ]
using InferOpt

regret_pert = Pushforward(
    perturbed_add, warcraft_cost
)
regret_reg = Pushforward(
    regularized, warcraft_cost
)
\end{lstlisting}
    \end{minipage}
    \hfill
    \begin{minipage}{0.48\textwidth}
\begin{lstlisting}[
            caption={Supported operations for an expected regret},
            label={code:regret_operations}
        ]
using Zygote

R = regret(theta)
Zygote.gradient(regret, theta)
\end{lstlisting}
    \end{minipage}
\end{figure}

\begin{center}
    \begin{minipage}{0.80\textwidth}
\begin{lstlisting}[
            caption={Learning with an expected regret},
            label={code:experience_learning}
        ]
using Flux, InferOpt

gradient_optimizer = Adam()
parameters = Flux.params(warcraft_encoder)
data = images

function pipeline_loss(x)
    theta = warcraft_encoder(x)
    return regret(theta)
end

for epoch in 1:1000
    train!(pipeline_loss, parameters, data, gradient_optimizer)
end
\end{lstlisting}
    \end{minipage}
\end{center}

\begin{remark}
    Since the learning problem is non-convex, we may also try to minimize the (non-smooth) regret~$\calR$ using global optimization algorithms such as DIRECT \parencite{jonesLipschitzianOptimizationLipschitz1993}.
    Perhaps surprisingly, this has been shown to yield good results when~$\varphi_w$ is a generalized linear model and the dimension of the weights~$w$ is not too large, i.e., no greater that~$100$ \parencite{parmentierLearningStructuredApproximations2021}.
    When the ML layer~$\varphi_w$ is a large neural network, we cannot use this approach anymore.
\end{remark}

\subsection{Derivatives of the regret for learning by experience}

When~$\widehat{p}$ comes from a random perturbation, we can formulate the gradient of the expected regret as an expectation too, and approximate it with Monte-Carlo samples.

\begin{proposition}[Gradient of the expected regret, perturbation setting] \label{prop:perturbation_regret_gradient}
    We have:
    \begin{align*}
        \nabla_\theta \calR_{\widehat{p}_\varepsilon^+}(\theta)
         & = \frac{1}{\varepsilon} \bbE \left[(c \circ f)(\theta + \varepsilon Z) Z\right]                                                                 \\
        \nabla_\theta \calR_{\widehat{p}_\varepsilon^\odot}(\theta)
         & = \frac{1}{\varepsilon \theta} \odot \bbE \left[(c \circ f)\left(\theta \odot e^{\varepsilon Z - \varepsilon^2 \onevector / 2}\right) Z\right].
    \end{align*}
\end{proposition}

\begin{proof}
    For the additive perturbation, it is a consequence of Proposition~\ref{prop:additive_perturbation}.
    For the multiplicative perturbation, it is a consequence of Proposition~\ref{prop:multiplicative_perturbation}.
\end{proof}

These gradients obey a simple logic: the more the perturbation~$Z$ increases the cost of a solution, the more positive weight it gets, and vice versa.
To see it, we remember that~$\bbE[Z] = 0$ for a standard Gaussian, and rewrite the regret gradients as follows:
\begin{align*}
    \nabla_\theta \calR_{\widehat{p}_\varepsilon^+}(\theta)
     & = \frac{1}{\varepsilon} \bbE \left[(c \circ f)(\theta + \varepsilon Z)Z - (c \circ f)(\theta)Z\right]                                                                   \\
    \nabla_\theta \calR_{\widehat{p}_\varepsilon^\odot}(\theta)
     & = \frac{1}{\varepsilon \theta} \odot \bbE \left[(c \circ f)\left(\theta \odot e^{\varepsilon Z - \varepsilon^2 \onevector / 2}\right) Z - (c \circ f)(\theta) Z\right].
\end{align*}
On the other hand, when~$\widehat{p}$ is derived from an explicit regularization~$\Omega$, the expected regret is amenable to implicit differentiation of the Frank-Wolfe algorithm.
Once more, the sparsity property makes exact computation tractable by reducing the number of terms in the sum:
\begin{equation*}
    \nabla_\theta \calR_{\widehat{p}_\Omega^{\FW}}(\theta) = \sum_{v \in \calV} c(v) \nabla_\theta \widehat{p}_\Omega^{\FW}(v|\theta)
\end{equation*}

\begin{remark}
    Although the previous discussion focuses on a scalar-valued cost, it actually applies to any pushforward function~$c$, even with vector values.
    The formulas for the generic Jacobian are given below:
    \begin{align*}
        J_\theta \bbE_{\widehat{p}_\varepsilon^+(\cdot | \theta)}[c(V)]
         & = \frac{1}{\varepsilon} \bbE \left[(c \circ f)(\theta + \varepsilon Z) Z^\top\right]                                                                \\
        J_\theta \bbE_{\widehat{p}_\varepsilon^\odot(\cdot | \theta)}[c(V)]
         & = \frac{1}{\varepsilon \theta} \odot \bbE \left[(c \circ f)\left(\theta \odot e^{\varepsilon Z - \varepsilon^2 \onevector / 2}\right) Z^\top\right] \\
        J_\theta \bbE_{\widehat{p}_\Omega^{\FW}(\cdot | \theta)}[c(V)]
         & = \sum_{v \in \calV} c(v) \nabla_\theta \widehat{p}_\Omega^{\FW}(v|\theta)^\top
    \end{align*}
    At the moment, \texttt{InferOpt.jl} only handles the case where~$c$ is a fully-defined function without free parameters.
    In the near future, we will add support for the case where~$c$ is itself an ML layer with learnable weights.
\end{remark}

\section{Learning by imitation} \label{sec:imitation}

We now move on to learning by imitation, where additional information is used to guide the training procedure.
For each input sample~$x^{(i)}$, we assume we have access to a \emph{target}~$\bar{t}^{(i)}$.
In that case, Equation~\eqref{eq:loss_minimization} simplifies as
\begin{equation}\label{eq:pb_learning_by_imitation}
    \min_w \frac{1}{N} \sum_{i=1}^N \calL\Big(f\big(\varphi_w(x^{(i)})\big), \bar{t}^{(i)}\Big),
\end{equation}
and we can see that the loss takes the target as an additional argument.
In this section, we introduce imitation losses that are well-suited to hybrid ML-CO pipelines, and explain how to compute their gradients.
As in Section~\ref{sec:experience}, we only consider a single input~$x$, and we write losses as~$\calL(\theta, \bar{t})$.

\subsection{A loss that takes the optimization layer into account} \label{sec:loss_decomp}

There are two main kinds of target.
The first one is a good quality solution~$\bar{t} = \bar{y}$.
The second one is the true objective direction~$\bar{\theta}$, from which we can also deduce~$\bar{y} = f(\bar{\theta})$, so that~$\bar{t} = (\bar{\theta}, \bar{y})$.
When learning by imitation, it is tempting to focus only on reproducing the targets, but this would be misguided.
To explain why, we revisit the pipeline of Equation~\eqref{eq:hybrid_pipeline}.

Remember that we may have access to the true objective direction~$\bar{\theta}$ during training, but at prediction time, the CO oracle~$f$ is applied to the encoder output~$\theta = \varphi_w(x)$ instead.
Minimizing a naive square loss like~$\lVert \varphi_w(x) - \bar{\theta}\rVert^2$ completely neglects the asymmetric impacts of the prediction errors on~$\theta$: for example, overestimating or underestimating~$\theta$ may have very different consequences on the quality of the downstream solution.
That is why, according to \textcite{elmachtoubSmartPredictThen2022}, we need a loss function that takes the optimization step into account.
The same holds true when we have access to a precomputed solution~$\bar{y}$.
\textcite{berthetLearningDifferentiablePerturbed2020} present experiments showing that the naive square loss~$\lVert \widehat{f}(\varphi_w(x)) - \bar{y}\rVert^2$ performs poorly compared with more refined approaches.
Our own numerical findings (Section~\ref{sec:applications}) support their conclusion.

\medskip

To sum up, we want a loss that does not neglect the optimization step.
Let~$y$ temporarily denote the output of our pipeline.
When surveying the literature, we realized that most flavors of imitation learning use losses that combine the same components:
\begin{equation} \label{eq:loss_aux}
    \calL_{\ell, \Omega}^{\text{aux}}(\theta, \bar{t}, y)
    = \underbrace{\ell(y, \bar{t})}_{\text{base loss}} + \underbrace{\theta^\top(y - \bar{y})}_{\substack{\text{gap between} \\ \text{$y$ and~$\bar{y}$ for the} \\ \text{CO problem~\eqref{eq:co_oracle}}}} - \underbrace{\left(\Omega(y) - \Omega(\bar{y})\right)}_{\text{regularization term}}
\end{equation}
Here is another way to write it:
\begin{equation*}
    \calL_{\ell, \Omega}^{\text{aux}}(\theta, \bar{t}, y)
    = \underbrace{\ell(y, \bar{t})}_{\text{base loss}}
    + \underbrace{\left(\theta^\top y - \Omega(y)\right) - \left(\theta^\top \bar{y} - \Omega(\bar{y})\right)}_{\substack{\text{gap between~$y$ and~$\bar{y}$} \\ \text{for the regularized CO problem~\eqref{eq:regularized_co_layer}}}}
\end{equation*}
The base loss~$\ell(y, \bar{t})$ is similar in spirit to the cost function~$c(y)$ from Section~\ref{sec:experience}.
But it is the gap term that truly makes it possible for the optimization problem to play a role in the loss.
Indeed, minimizing the gap encourages the (regularized) CO problem to output a solution~$y$ that is close to the target~$\bar{y}$.

Putting these components together yields a linear function of~$\theta$, and we can remove the dependency in~$y$ by maximizing over~$y$:
\begin{equation} \label{eq:loss_decomposition}
    \calL_{\ell, \Omega}^{\text{gen}}(\theta, \bar{t}) =  \max_{y \in \dom(\Omega)} \calL_{\ell, \Omega}^{\text{aux}}(\theta, \bar{t}, y) = \max_{y \in \dom(\Omega)} \left[\ell(y, \bar{t}) + \theta^\top(y - \bar{y}) - \left(\Omega(y) - \Omega(\bar{y})\right) \right].
\end{equation}
The following result justifies why this is an interesting loss.

\begin{proposition}[Properties of the generic loss for learning by imitation] \label{prop:properties_loss_imitation}
    The function~$\calL_{\ell, \Omega}^{\text{gen}}(\theta, \bar{t})$ is convex with respect to~$\theta$, and a subgradient is given by
    \begin{equation} \label{eq:subgradient_loss_imitation}
        \Big(\argmax_{y \in \dom(\Omega)} \calL_{\ell, \Omega}^{\text{aux}}(\theta, \bar{t}, y)\Big) - \bar{y} \quad \in \quad \partial_\theta \calL_{\ell, \Omega}^{\text{gen}}(\theta, \bar{t}).
    \end{equation}
\end{proposition}

\begin{proof}
    As a pointwise maximum of affine functions,~$\theta \longmapsto \calL_{\ell, \Omega}^{\text{gen}}(\theta, \bar{t})$ is convex.
    Its subgradient is obtained using Danskin's theorem \parencite{danskinTheoryMaxMinIts1967}.
\end{proof}

The idea is that solving~$\argmax_{y \in \dom(\Omega)} \calL_{\ell, \Omega}^{\text{aux}}(\theta, \bar{t}, y)$ should not be much harder than the regularized CO problem~\eqref{eq:regularized_co_layer}.
Therefore, using such a loss function dispenses us from differentiating through the probabilistic CO layer: most of the time, we only need to compute the layer output in order to obtain a loss subgradient for free.

\subsection{Collection of losses for learning by imitation}

Several prominent loss functions from the literature are special cases of our decomposition~\eqref{eq:loss_decomposition}: we gather them in Table~\ref{tab:decomposition}.
Code sample~\ref{code:losses} clarifies their construction, while Code sample~\ref{code:loss_operations} displays supported operations.
Finally, the entire program necessary for learning by imitation is shown on Code sample~\ref{code:imitation_learning}.

In Sections~\ref{sec:ssvm},~\ref{sec:spo},~\ref{sec:fyl} and~\ref{sec:loss_generic}, we go over these special cases to explain how to compute each loss and its subgradient using Equation~\eqref{eq:subgradient_loss_imitation}.
They can be skipped without danger.

\begin{remark}
    While the S-SVM and SPO+ losses do not fall within the framework of probabilistic CO layers (due to the absence of regularization), we still include them for benchmarking purposes.
\end{remark}

\begin{table}
    \centering
    \begin{tabular}{@{}cccccc@{}}
        \toprule
        \textbf{Method} & \textbf{Notation}            & \textbf{Target}            & \textbf{Base loss}                & \textbf{Regul.} & \textbf{Loss formula}                                                                                           \\ \midrule
        S-SVM           & ~$\calL_\ell^{\text{S-SVM}}$ & ~$\bar{y}$                 & ~$\ell(y,\bar{y})$                & No              & ~$\displaystyle \max_y \ell(y, \bar{y}) + \theta^\top (y - \bar{y})$                                            \\
        SPO+            & ~$\calL^{\text{SPO+}}$       & ~$(\bar{\theta}, \bar{y})$ & ~$\bar{\theta}^\top(\bar{y} - y)$ & No              & ~$\displaystyle \max_y \bar{\theta}^\top (\bar{y} - y) + 2 \theta^\top (y - \bar{y})$                           \\
        FY              & ~$\calL_\Omega^{\text{FY}}$  & ~$\bar{y}$                 & ~$0$                              & Yes             & ~$\displaystyle \max_y \theta^\top (y - \bar{y}) - \left(\Omega(y) - \Omega(\bar{y})\right)$                    \\
        Generic         & ~$\calL_{\ell, \Omega}^{\text{gen}}$      & ~$\bar{t}$                 & ~$\ell(y, \bar{t})$               & Yes             & ~$\displaystyle \max_y \ell(y, \bar{t}) + \theta^\top (y - \bar{y}) - \left(\Omega(y) - \Omega(\bar{y})\right)$ \\ \bottomrule
    \end{tabular}
    \caption{A common decomposition for loss functions in imitation learning}
    \label{tab:decomposition}
\end{table}

\begin{figure}
    \centering
    \begin{minipage}{0.51\linewidth}
\begin{lstlisting}[
            caption={Example imitation losses},
            label={code:losses}
        ]
using InferOpt

fyl_pert = FenchelYoungLoss(perturbed_add)
fyl_reg = FenchelYoungLoss(regularized)
spol = SPOPlusLoss(warcraft_maximizer)
\end{lstlisting}
    \end{minipage}
    \hfill
    \begin{minipage}{0.45\linewidth}
\begin{lstlisting}[
            caption={Supported operations for an imitation loss},
            label={code:loss_operations}
        ]
using Zygote

L = loss(theta, y_ref)
Zygote.gradient(loss, theta, y_ref)
\end{lstlisting}
    \end{minipage}
\end{figure}

\begin{center}
    \begin{minipage}{0.8\textwidth}
\begin{lstlisting}[
            caption={Learning with an imitation loss},
            label={code:imitation_learning}
        ]
using Flux, InferOpt

gradient_optimizer = ADAM()
parameters = Flux.params(warcraft_encoder)
data = zip(images, paths)

function pipeline_loss(x, y)
    theta = warcraft_encoder(x)
    return loss(theta, y)
end

for epoch in 1:1000
    train!(pipeline_loss, parameters, data, gradient_optimizer)
end
\end{lstlisting}
    \end{minipage}
\end{center}

\subsubsection{Structured support vector machines} \label{sec:ssvm}

The structured support vector machine (S-SVM) was among the first methods introduced for learning in structured spaces \parencite[Chapter 6]{nowozinStructuredLearningPrediction2010}.
Given a target solution~$\bar{y}$ and an underlying distance function~$\ell(y, \bar{y})$ on~$\calV$, the S-SVM loss is computed as follows:
\begin{equation}
    \calL_\ell^{\text{S-SVM}}(\theta, \bar{y}) = \max_{y \in \calV} \{ \ell(y, \bar{y}) + \theta^\top(y - \bar{y}) \}.
\end{equation}
The subgradient formula~\eqref{eq:subgradient_loss_imitation} becomes
\begin{equation*}
    \argmax_{y \in \calV} \{ \ell(y, \bar{y}) + \theta^\top(y - \bar{y}) \} - \bar{y} \quad \in \quad \partial_\theta \calL_\ell^{\text{S-SVM}}.
\end{equation*}
Note that due to the presence of~$\ell$, computing a subgradient requires an auxiliary solver that is different from the linear oracle~$f$.
This is why we do not illustrate the S-SVM with a code sample.
In \texttt{InferOpt.jl}, we only implement this auxiliary solver for the unit simplex, in the case where~$\ell$ is the Hamming distance.
However, we also provide a generic layer where the user can plug in the relevant auxiliary solver.

\subsubsection{Smart \enquote{predict, then optimize}} \label{sec:spo}

The smart \enquote{predict, then optimize} (SPO) paradigm is applicable when the true objective direction~$\bar{\theta}$ is known (remember that in this case, we have~$\bar{y} = f(\bar{\theta})$).
\textcite{elmachtoubSmartPredictThen2022} define the SPO+ loss function as follows:
\begin{align} \label{eq:spoplus_loss}
    \calL^{\text{SPO+}}(\theta, \bar{\theta})
     & = (2\theta - \bar{\theta})^\top f(2 \theta - \bar{\theta}) + (\bar{\theta} - 2\theta)^\top \bar{y}           \\
     & = \max_{y \in \calV} \left\{ \bar{\theta}^\top (\bar{y} - y) + 2\theta^\top(y - \bar{y}) \right\}. \nonumber
\end{align}
It can be seen as a special case of S-SVM.
But this time, computing the loss and its subgradient only requires calling~$f$ twice:
\begin{equation*}
    2f(2 \theta - \bar{\theta}) - 2\bar{y} \quad \in \quad \partial_\theta \calL^{\text{SPO+}}(\theta, \bar{\theta}).
\end{equation*}

\subsubsection{Fenchel-Young losses} \label{sec:fyl}

The framework of Fenchel-Young losses is built on the theory of convex conjugates, in particular the Fenchel-Young inequality \parencite{blondelLearningFenchelYoungLosses2020}.
Starting from a target solution~$\bar{y}$ and a regularization~$\Omega$, a loss is constructed as follows:
\begin{align} \label{eq:fenchel_young_loss}
    \calL_{\Omega}^{\text{FY}}(\theta, \bar{y})
     & = \Omega^*(\theta) + \Omega(\bar{y}) - \theta^\top \bar{y}                                                                         \\
     & = \max_{y \in \conv(\calV)} \left(\theta^\top y - \Omega(y)\right) - \left(\theta^\top \bar{y} - \Omega(\bar{y}) \right) \nonumber
\end{align}
This time, the loss and subgradient require access to~$\widehat{f}_\Omega$:
\begin{equation*}
    \widehat{f}_{\Omega}(\theta) - \bar{y} \quad \in \quad \partial_\theta \calL_{\Omega}^{\text{FY}}(\theta, \bar{y}).
\end{equation*}
As can be inferred from the expression above, there are deep connections between Fenchel-Young losses and the regularization paradigm of Section~\ref{sec:regularization}.
In particular, it is also possible to use implicit regularization by perturbation \parencite{berthetLearningDifferentiablePerturbed2020}.
The fact that we cannot compute~$\Omega_\varepsilon^+(y)$ or~$\Omega_\varepsilon^\odot(y)$ is not a real obstacle: since those terms do not depend on~$\theta$, we can just drop them from the loss during training.
We end up with the following estimators for the loss and its subgradient:
\begin{align*}
    \calL_{\Omega_\varepsilon^+}^{\text{FY}}(\theta, \bar{y})
     & = F_\varepsilon^+(\theta) - \theta^\top \bar{y}     & \widehat{f}_\varepsilon^+(\theta) - \bar{y}                       & \in \partial_\theta \calL_{\Omega_\varepsilon^+}^{\text{FY}}(\theta, \bar{y})     \\
    \calL_{\Omega_\varepsilon^\odot}^{\text{FY}}(\theta, \bar{y})
     & = F_\varepsilon^\odot(\theta) - \theta^\top \bar{y} & \widehat{f}_\varepsilon^{\odot \mathrm{scaled}}(\theta) - \bar{y} & \in \partial_\theta \calL_{\Omega_\varepsilon^\odot}^{\text{FY}}(\theta, \bar{y})
\end{align*}

\subsubsection{Generic imitation loss} \label{sec:loss_generic}

Of course, it is tempting to fill in the blanks of Table~\ref{tab:decomposition} by combining every single term of the loss decomposition~\eqref{eq:loss_decomposition}.
To the best of our knowledge, this has not been done before in the literature, but there is no theoretical obstacle.

If we use this generic loss together with regularization, then it is interesting to remark that~$\calL_{\ell, \Omega}^{\text{gen}}(\theta, \bar{t})$ acts as a convex upper bound on the base loss~$\ell(\widehat{f}_\Omega(\theta), \bar{t})$.
Indeed, since~$\bar{y}$ is a worse solution than~$\widehat{f}_\Omega(\theta)$ for~\eqref{eq:regularized_co_layer}, we have
\begin{align*}
    \ell(\widehat{f}_\Omega(\theta), \bar{t})
     & \leq \ell(\widehat{f}_\Omega(\theta), \bar{t}) + \left[ \theta^\top \widehat{f}_\Omega(\theta) - \Omega(\widehat{f}_\Omega(\theta))\right] - \left[\theta^\top \bar{y} - \Omega(\bar{y}) \right] \\
     & \leq \max_{y \in \conv(\calV)} \left( \ell(y, \bar{t}) + \theta^\top (y - \bar{y}) - \left(\Omega(y) - \Omega(\bar{y})\right) \right) = \calL_{\ell, \Omega}^{\text{gen}}(\theta, \bar{t}).
\end{align*}
Therefore, our generic loss can be seen as a crossover between the Fenchel-Young loss and a problem-specific base loss.
It is not yet implemented in \texttt{InferOpt.jl}, and we leave its thorough testing for future work.

\section{Applications} \label{sec:applications}

In this section, we define specific ML-CO pipelines, in both learning settings (by experience and by imitation), to address four problems: our shortest path problem on Warcraft maps, a stochastic vehicle scheduling problem, a single-machine scheduling problem, and a two-stage stochastic minimum weight spanning tree problem.
The first one stems from the ML community.
The other three are classics from the field of operations research.
They illustrate the idea of approximating hard optimization problems with simpler ones using ML-CO pipelines.

% Each subfile must contain a subsection title

\subsection{Shortest paths on Warcraft maps}\label{subsec:warcraft_appli}

First we come back to our guiding example of Section~\ref{sec:warcraft}.
Our aim is to illustrate the various learning settings introduced in this paper, and to evaluate their relative performance.
We do so with two kinds of shortest path (SP) oracles.
The first one uses Dijkstra's algorithm.
The second one uses the Ford-Bellman algorithm with a bounded number of iterations.

\subsubsection{Experimental setting}\label{subsubsec:experimental_setting_warcraft}

In every experiment presented here, we only consider a sub-dataset, made up of $1\%$~of the original Warcraft dataset from \textcite{vlastelicaDifferentiationBlackboxCombinatorial2020}.
It contains $200$~samples, or maps, which we split into $80$~training samples, $100$~validation samples (for hyperparameter tuning) and $20$~test samples (for performance evaluation).
For each learning setting, the train, test and validation sets remain the same, and we individually tune a subset of the hyperparameters stated in Table~\ref{tab:hyperparameters_warcraft_overview}.
Our motivation for reducing the dataset is to show that we can still obtain convincing results with a limited amount of computation.

We use the \texttt{Metalhead.jl} package to build a truncated ResNet18 CNN, \texttt{Flux.jl} to train our pipelines with the Adam optimizer \parencite{kingmaAdamMethodStochastic2015}, and \texttt{GridGraphs.jl} to compute shortest paths.
Our code is available in the \texttt{WarcraftShortestPaths.jl}\footnote{\url{https://github.com/LouisBouvier/WarcraftShortestPaths.jl}} repository.
The experiments are conducted on a MacBook Pro with 2,3 GHz Intel Core i9, 8 cores and 16 Go 2667 MHz DDR4 RAM.

\begin{table}[ht]
    \centering
    \begin{tabular}{cc}
        \toprule
        \textbf{Hyperparameter} & \textbf{Description}                              \\ \midrule
        $\texttt{epsilon}$      & Scale of the noise for perturbation. \\
        $\texttt{nb\_samples}$  & Number of noise samples $M$ for perturbation.     \\
        $\texttt{batch\_size}$  & Size of the batches to compute gradients.         \\
        $\texttt{lr\_start}$    & Starting learning rate.                           \\
        \bottomrule
    \end{tabular}
    \caption{Hyperparameters for learning Warcraft shortest paths.}
    \label{tab:hyperparameters_warcraft_overview}
\end{table}

\medskip

To obtain a precise learning setting, we need to define:
\begin{enumerate}
    \item The combinatorial problem we need to solve, along with an appropriate oracle.
    \item The probabilistic CO layer used to wrap said oracle.
    \item The data we have at our disposal to train our pipeline.
    \item The loss function we want to minimize.
\end{enumerate}
No matter the setting, the data always contains a list of RGB map images.
When we learn by experience, we also have access to a black box cost function, which evaluates paths based on the true cell costs (see the beginning of Section~\ref{sec:experience}).
On the other hand, when we learn by imitation, we add targets to the maps (as defined in Section \ref{sec:imitation}).
The target in our case always includes the optimal path, with or without the true cell costs.

All those ingredients are detailed in Table~\ref{tab:learning_setting_warcraft} for each learning setting we consider.
The names given in the first column are reused in the legends of Figure~\ref{fig:gaps_warcraft}.
Most of the probabilistic CO layers considered in this paper do not prevent the objective vector $\theta$ from changing its sign, and the same goes for the losses
As a result, we need oracles able to accommodate negative cell costs.
That is why we use the Ford-Bellman algorithm, while limiting the number of iterations to the number of nodes in the grid graph (to ensure termination even with negative cycles).
Our multiplicative perturbation is the only approach that preserves non-negative costs.
It enables us to apply Dijkstra's algorithm, which has a smaller time and space complexity.

\begin{table}[ht]
    \setlength{\tabcolsep} {0.3cm}
    \scalebox{0.68} {
        \centering
        \begin{tabular}{lllll}
            \toprule
            \textbf{Seeting name}                                & \makecell{\textbf{CO problem}             \\ \textbf{(CO oracle)}}        & \textbf{Probabilistic CO layer}    & \makecell{\textbf{Exp./Imit.} \\ \textbf{Target}} & \textbf{Loss}      \\
            \midrule
            Cost perturbed multiplicative noise          & \makecell{SP with non-negative costs      \\ (Dijkstra)}  & \makecell{Multiplicative\\ perturbation}  & \makecell{Experience\\ No target}          & Perturbed cost     \\
            \midrule
            Cost perturbed additive noise                & \makecell{SP on an extended acyclic graph \\ (Ford-Bellman)}  & \makecell{Additive \\ perturbation}  & \makecell{Experience\\ No target}          & Perturbed cost     \\
            \midrule
            Cost regularized half square norm            & \makecell{SP on an extended acyclic graph \\ (Ford-Bellman)} & Half square norm           & \makecell{Experience\\ No target}         & Regularized cost   \\
            \midrule
            SPO+                                         & \makecell{SP on an extended acyclic graph \\ (Ford-Bellman)} & No regularization          & \makecell{Imitation\\ Cost and path}         & SPO+ loss          \\
            \midrule
            MSE perturbed multiplicative noise           & \makecell{SP with non-negative costs      \\ (Dijkstra)} & \makecell{Multiplicative\\ perturbation} & \makecell{Imitation\\ Path}         & Mean squared error \\
            \midrule
            MSE regularized half square norm             & \makecell{SP on an extended acyclic graph \\ (Ford-Bellman)} & Half square norm           & \makecell{Imitation\\ Path}         & Mean squared error \\
            \midrule
            Fenchel-Young perturbed multiplicative noise & \makecell{SP with non-negative costs      \\ (Dijkstra)} & \makecell{Multiplicative\\ perturbation} & \makecell{Imitation\\ Path}         & Fenchel-Young      \\
            \midrule
            Fenchel-Young perturbed additive noise       & \makecell{SP on an extended acyclic graph \\ (Ford-Bellman)} & \makecell{Additive\\ perturbation} & \makecell{Imitation\\ Path}         & Fenchel-Young      \\
            \midrule
            Fenchel-Young regularized half square norm   & \makecell{SP on an extended acyclic graph \\ (Ford-Bellman)} & Half square norm           & \makecell{Imitation\\ Path}         & Fenchel-Young      \\
            \bottomrule
        \end{tabular}
    }
    \caption{Learning settings for Warcraft shortest paths.}
    \label{tab:learning_setting_warcraft}
\end{table}

\subsubsection{Results}

In Figure~\ref{fig:gaps_warcraft}, we show the average train (Figure~\ref{fig:gaps_warcraft_train}) and test (Figure~\ref{fig:gaps_warcraft_test}) optimality gaps, computed using the true cell costs.
We compare all the settings detailed in Table~\ref{tab:learning_setting_warcraft}, with the exception of MSE base loss + additive noise (we could not get satisfactory results using only 80 training samples).
To quantify training effort, instead of counting epochs (\textit{i.e.} passes through the dataset), we use the number of optimizer calls, because these calls are the truly time-consuming part.
This aims at comparing learning settings which involve different amounts of computation per gradient step.
For instance, using SPO+, we need 2 optimizer calls to compute the loss gradient for one sample.
On the other hand, we need $M$ optimizer calls if we choose Fenchel-Young perturbed additive noise.

\medskip

When learning by imitation, SPO+ reaches almost zero average gap both on train and test sets after very few optimizer calls, even though we only kept $1\%$ of the initial dataset.
This impressive result can be understood since, with SPO+, we have access to the true cell costs during training, and we leverage the problem structure within the loss.

Assuming we only have access to target paths, we obtain better results with Fenchel-Young losses than with MSE losses.
The train and test average gaps are lower than $5\%$ with the former, and we observe good generalization performance.
This may be explained by the use of the optimization problem in the Fenchel-Young loss definition.
On the contrary, in the MSE setting, although we have access to target paths, we only seek to imitate them without truly accounting for solution cost.

\medskip

Perhaps surprisingly, we also manage to learn by experience with our small sub-dataset.
Indeed, using the techniques introduced in Section~\ref{sec:experience}, we reach $7\%$ average test gaps in the cost perturbed multiplicative noise setting, which is better than learning by imitation with an MSE loss.
To the best of our knowledge, it is the first time that learning by experience (as defined in Section~\ref{sec:experience}) is combined with CNNs.

\begin{figure}[ht]
    \begin{subfigure}{\textwidth}
        \makebox[\textwidth][c]{
            \includegraphics[width=1.3\linewidth,trim={0 40 0 0},clip]{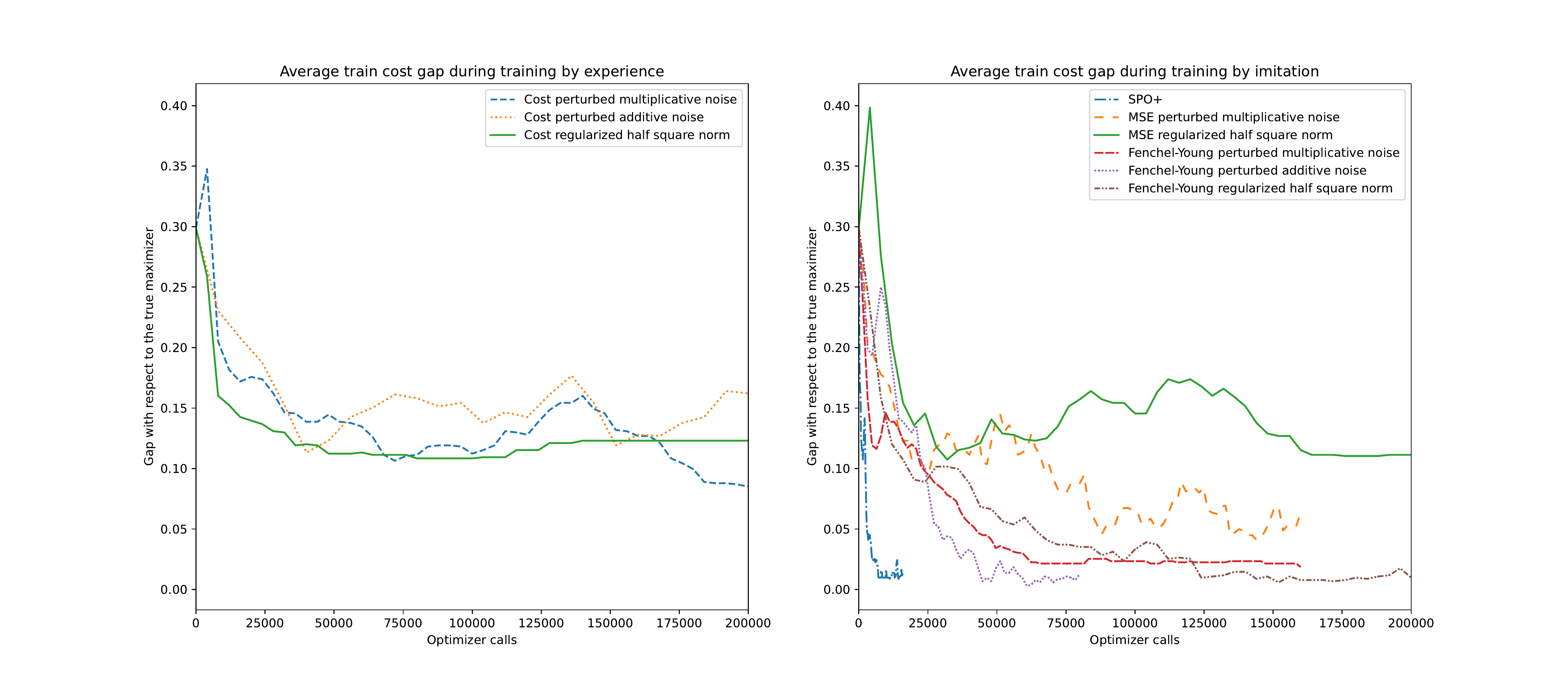}
        }
        \caption{Average train gap}
        \label{fig:gaps_warcraft_train}
    \end{subfigure}
    \begin{subfigure}{\textwidth}
        \makebox[\textwidth][c]{
            \includegraphics[width=1.3\linewidth,trim={0 40 0 0},clip]{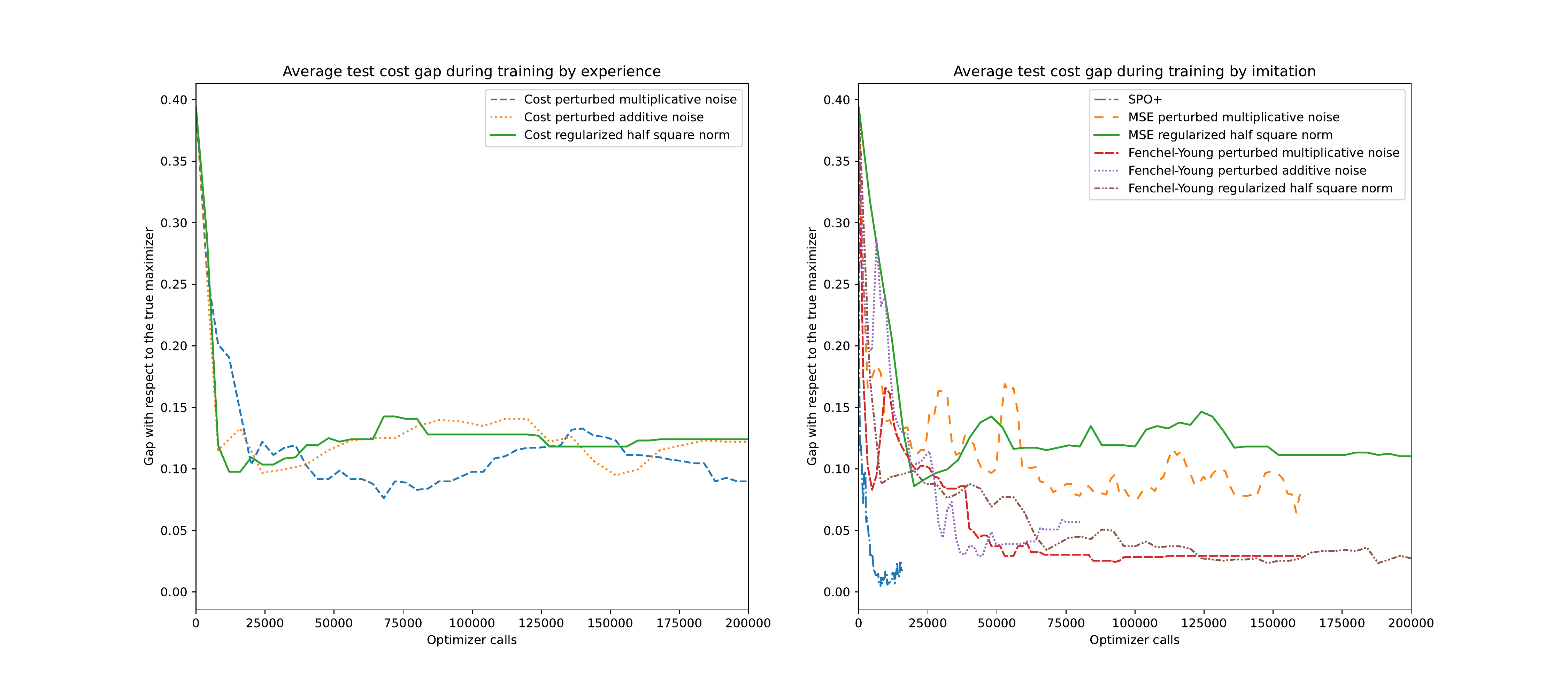}
        }
        \caption{Average test gap}
        \label{fig:gaps_warcraft_test}
    \end{subfigure}
    \caption{Train and test optimality gaps along training in the Warcraft application}
    \label{fig:gaps_warcraft}
\end{figure}

\subsection{Approximating hard optimization problems}

Solving hard optimization problems is one of the main applications of hybrid CO-ML pipelines.
The principle is to build a pipeline which approximates a hard combinatorial problem, for which we do not have an efficient algorithm.
The approximation relies on a similar but easier problem, for which an efficient algorithm exists.
This \enquote{easy} algorithm is then used as a CO layer.
In order to obtain the best possible approximation (in terms of closeness to the original problem), this CO layer is complemented with encoding and decoding layers, whose weights can be learned.

In the remaining of this section, we use \texttt{InferOpt.jl} in this framework, and apply it to three different problems.
In Section~\ref{subsec:StochasticVSP}, we focus on the \emph{stochastic vehicle scheduling problem} with similar experiments as in \textcite{parmentierLearningStructuredApproximations2021} for the learning by imitation, and \textcite{parmentierLearningApproximateIndustrial2021} for the learning by experience.
Then, in Section~\ref{subsec:scheduling}, we study the \emph{single machine scheduling problem} and compare our results to those in \textcite{parmentierLearningSolveSingle2021}.
Finally, in Section~\ref{subsec:spanning_tree}, we look at the \emph{two-stage stochastic minimum weight spanning tree problem}, and demonstrate the use of a GNN in combination with \texttt{InferOpt.jl}.
Source code for these three applications can be found in their respective satellite packages.

\subsection{Stochastic vehicle scheduling problem}\label{subsec:StochasticVSP}

We use \texttt{InferOpt.jl} to solve the \emph{stochastic vehicle scheduling problem} (StoVSP), by learning a transformation that approximates its instances as instances of the easier to solve \emph{vehicle scheduling problem} (VSP). Source code can be found in the satellite package \texttt{StochasticVehicleScheduling.jl}\footnote{\url{https://github.com/BatyLeo/StochasticVehicleScheduling.jl}}.

\subsubsection{Problem formulation}

\paragraph{Vehicle scheduling problem}

The (deterministic) VSP consists in assigning vehicles to cover a set of scheduled tasks, while minimizing the total cost. Let~$\bar V$ be the set of tasks. Each task~$v\in \bar V$ has a scheduled beginning time~$t_v^b$, and a scheduled end time~$t_v^e$, such that~$t_v^e > t_v^b$. We denote~$t^{tr}_{(u, v)}$ the travel time from task~$u$ to task~$v$. A task~$v$ can be scheduled after another task~$u$ only if we can reach it in time, before it starts:
\begin{equation}\label{eq:task_chaining}
    t_v^b \geq t_u^e + t^{tr}_{(u, v)}
\end{equation}

An instance of VSP can be modeled with a directed graph~$D = (V, A)$, with~$V = \bar V\cup\{o, d\}$, and~$o$,~$d$ origin and destination dummy nodes. For all task~$v\in V$,~$(o, v)$ and~$(v, d)$ are arcs in~$A$. Additionally, there is an arc between two tasks~$u$ and~$v$ only if (\ref{eq:task_chaining}) is satisfied. The resulting graph~$D$ is acyclic.

A solution of the VSP problem is a set of~$o-d$ paths partitioning~$D$, such that all tasks are covered exactly once. The objective is to minimize the sum of path edge costs. This can be formulated as an LP (see Appendix~\ref{sec:appendix_deterministicVSP}), and can be solved either using a flow algorithm, or using a general purpose LP solver.

\paragraph{Stochastic Vehicle Scheduling}

In the stochastic VSP, we consider the same setting as the deterministic version, to which we add the following. Once the scheduling decision is set, we observe random delays, which propagate along vehicle paths. The objective is to minimize the sum of vehicle costs and expected total delay costs.

We consider a finite set of scenarios~$s\in S$. For each task~$v\in \bar V$, we denote~$\gamma_v^s$ the intrinsic delay of~$v$ in scenario~$s$, and~$d_v^s$ its total delay. We also denote~$\delta_{u, v}^s$ the slack between tasks~$u$ and~$v$. These quantities follow the delay propagation equation when~$u$ and~$v$ are consecutively operated by the same vehicle:
\begin{equation}
d_v^s = \gamma_v^s + \max(d_u^s - \delta_{u, v}^s, 0)
\end{equation}

This leads to a much more difficult problem to solve. In Appendix~\ref{sub:solveStoVSP}, we provide a compact MILP formulation, which enables to easily solve optimally instances with up to 25 tasks using commercial MILP solvers.

\subsubsection{Datasets}

To generate our instance datasets, we use a generator similar to the one used by \textcite{parmentierLearningApproximateIndustrial2021}. More details are given in Appendix~\ref{sub:VSPdatagen}. We consider 3 training/validation datasets used to train our models. Each dataset contains 100 instances, and is divided into 50 training instances and 50 validation instances. The first dataset contains only instances with 25 tasks and 10 scenarios, the second one only instances with 50 tasks and 50 scenarios, and the last one only instances with 100 tasks and 50 scenarios. The 25 tasks dataset contains small instances for which we can easily compute optimal solutions to use as target solutions during the learning. For the two other datasets, we instead used a local search heuristic to compute \enquote{good} (non necessarily optimal) solutions of each instance. For testing purposes, we consider several additional test datasets we only use for evaluating the final performances and generalization abilities of our learned models. These test datasets contain larger instances with up to~$1000$ tasks. Datasets content is summarized in Table~\ref{tab:vsp-datasets}.

\begin{table}[H]
    \centering
    \begin{tabular}{ccccccccc}
        \toprule
        \multirow{2}{*}{\textbf{Dataset type}} & \multicolumn{8}{c}{\textbf{Number of tasks}}                                                                         \\ \cline{2-9} 
                                                & \textbf{25} & \textbf{50} & \textbf{100} & \textbf{200} & \textbf{300} & \textbf{500} & \textbf{750} & \textbf{1000} \\ \midrule
        \textbf{Training}                      & 50          & 50          & 50           & 0            & 0            & 0            & 0            & 0             \\ \midrule
        \textbf{Validation}                    & 50          & 50          & 50           & 0            & 0            & 0            & 0            & 0             \\ \midrule
        \textbf{Test}                          & 50          & 50          & 50           & 50           & 50           & 50           & 50           & 50            \\ \bottomrule
    \end{tabular}
    \caption{Summary of the number of instances in each used dataset}
    \label{tab:vsp-datasets}
\end{table}

\subsubsection{InferOpt pipeline}

We use the deterministic VSP as a CO layer in the following pipeline to solve pipeline to solve the stochastic VSP. The goal is to train a model able to generalize well enough to instances it has not seen before, especially bigger instances, for which we cannot compute a good solution in reasonable time.
For this, we consider the pipeline presented in Figure~\ref{fig:vsp-pipeline}.

\begin{figure}[H]
    \centering
    \includegraphics[width=\textwidth]{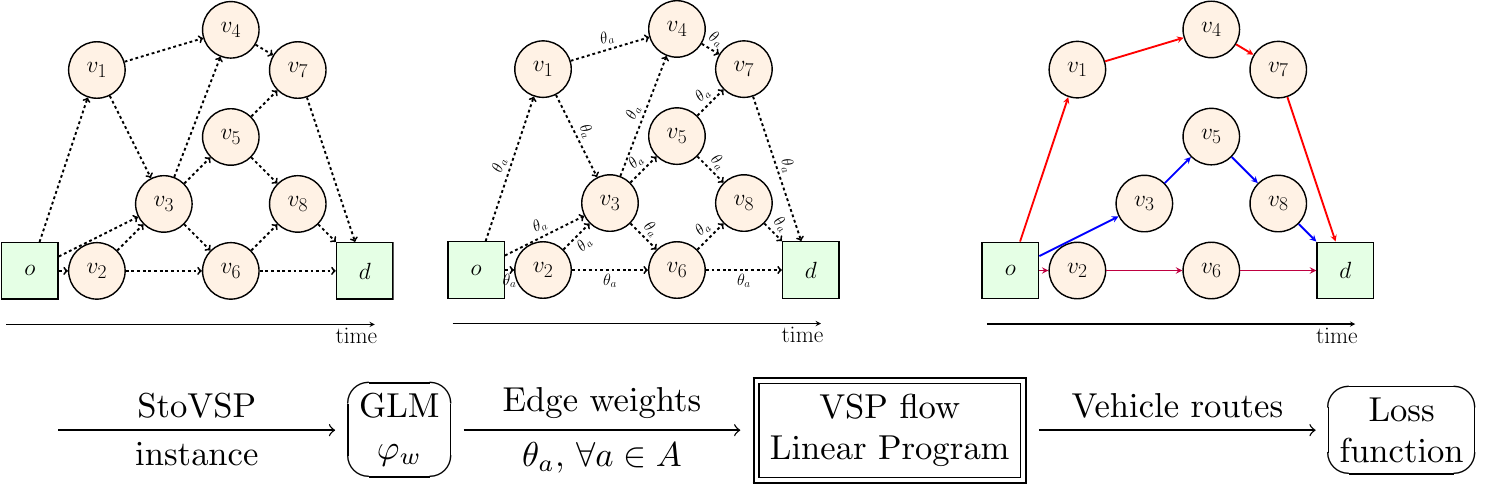}
    \caption{Pipeline for the stochastic VSP}
    \label{fig:vsp-pipeline}
\end{figure}

An instance~$x$ of StoVSP is encoded by computing a vector~$\phi(x, a)$ of 20 features on each arc~$a\in A$. These features contain edge length, slack quantiles and cumulative distribution information. For more details, refer to Appendix~\ref{appendix:vsp-features}.
Instance~$x$ is then given as input to a GLM~$\varphi_w$ with learnable weights~$w$, such that~$\varphi_w(x) = \theta$, with~$\theta_a = w^\top \phi(x, a)$. Then, we use the predicted~$\theta$ as the objective arc weights of the VSP flow linear program to compute the solution paths:

$$
\begin{aligned}
\max_{y} & \,\sum_{a\in A}\theta_a y_a\\
\text{s.t.} & \sum_{a\in\delta^-(v)}y_a = \sum_{a\in\delta^+(v)}y_a, &\forall v\in \bar V & \quad(\text{flow polytope})\\
& \sum_{a\in\delta^-(v)}y_a = 1, &\forall v\in \bar V & \quad \text{(task covering)}\\
& y_a \geq 0,\, & \forall a\in A
\end{aligned}
$$
Training this pipeline means finding parameter vector~$w$ that minimizes the chosen loss.

\subsubsection{Experimental setting}

For each of the three datasets considered, we train two models: we learn the first one by imitation thanks to a Fenchel-Young loss, and the second one by experience thanks to the expected regret.
In both cases, we use a \texttt{PerturbedAdditive} regularization with parameters~$\varepsilon$ and \texttt{nb\_samples}.
We use the \texttt{Flux.jl} Julia library and the Adam optimizer for training. Hyperparameters used for each model/dataset pair can be found in Table~\ref{tab:hyperparameters-vsp}.

\begin{table}[H]
    \centering
    \begin{tabular}{c ccc ccc }
        \toprule
        \multirow{2}{*}{\textbf{Hyperparameters}} & \multicolumn{3}{c}{\textbf{Learn by imitation models}}                                            & \multicolumn{3}{c}{\textbf{Learn by experience models}}                                           \\ \cline{2-7}
        & \multicolumn{1}{c}{25 tasks} & \multicolumn{1}{c}{50 tasks} & 100 tasks & 25 tasks & 50 tasks & 100 tasks \\ \toprule
       ~$\varepsilon$                             & 0.1      & 0.1      & 0.1       & 50       & 100      & 300     \\ \midrule
        \texttt{nb\_samples}                      & 20       & 20       & 20        & 20       & 20       & 20      \\ \midrule
        training epochs                       & 50       & 50       & 50        & 200      & 200      & 200       \\ \midrule
        batch size                      & 1        & 1        & 1         & 1        & 1        & 1         \\
        \bottomrule
    \end{tabular}
    \caption{Hyperparameter for Stochastic Vehicle Scheduling experiments}
    \label{tab:hyperparameters-vsp}
\end{table}

Datasets features are normalized during training by dividing each feature by its standard deviation in the train dataset. At the end of training, the obtained~$w$ is renormalized accordingly. We optimize hyperparameters using metric values on validation datasets. The experiments are conducted on a 2.6 GHz Intel Core i9-11950H, 16 cores, with 64 Go RAM. We faced one practical difficulty during the training: the calibration of hyperparameter~$\varepsilon$ when learning by experience. Indeed, training performance is quite sensible to its value, which lead to a lot of hyperparameter tuning on the validation set. Detailed training plots can be found in Appendix~\ref{appendix:vsp-training-plots}.

\subsubsection{Results}

Finally, we can look at the performance of the learned pipelines on all the different datasets, and compare them with each others. In Figure~\ref{fig:vsp-prediction}, we observe the resulting solution of three different algorithms on the same input instance. Interestingly, the model learned by imitation outputs a different solution than the model learned by experience.

\begin{figure}[H]
    \centering
    \begin{subfigure}[b]{0.32\textwidth}
        \centering
        \includegraphics[width=\textwidth]{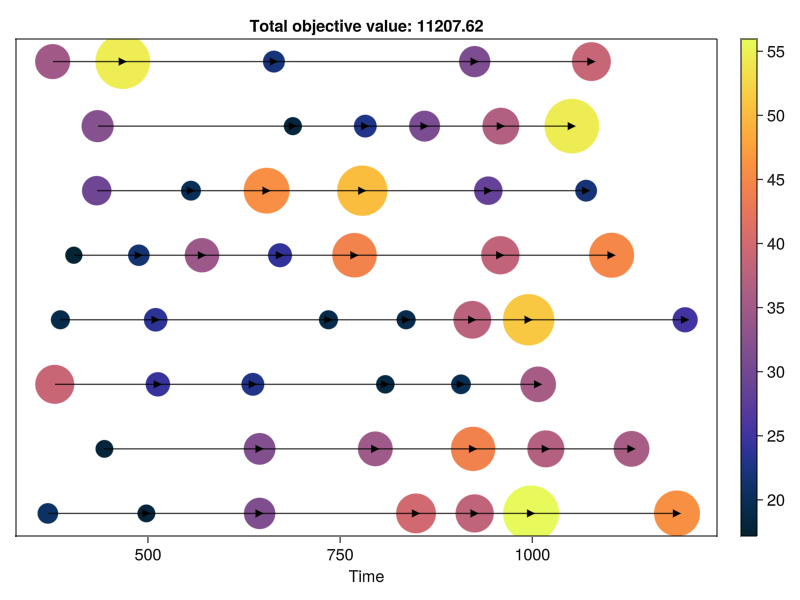}
        \caption{Local search heuristic}
    \end{subfigure}
    \begin{subfigure}[b]{0.32\textwidth}
        \centering
        \includegraphics[width=\textwidth]{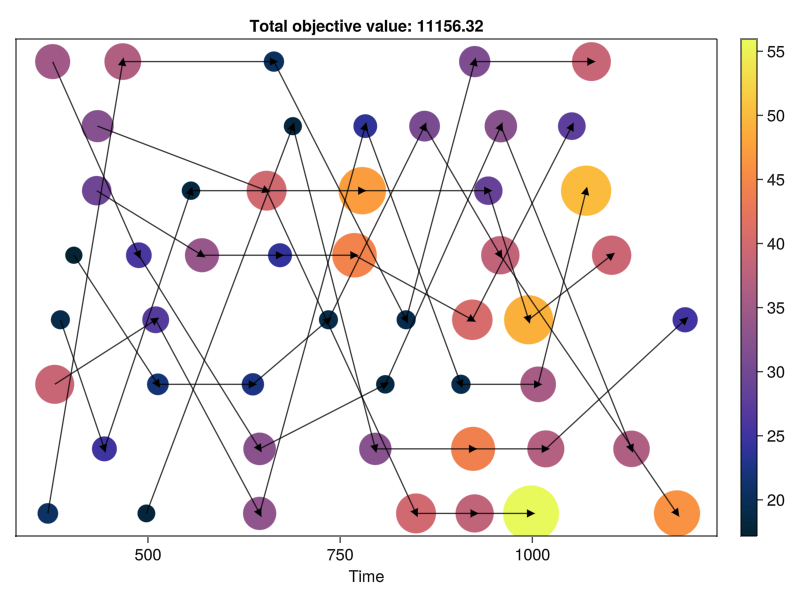}
        \caption{Model learned by imitation}
    \end{subfigure}
    \begin{subfigure}[b]{0.32\textwidth}
        \centering
        \includegraphics[width=\textwidth]{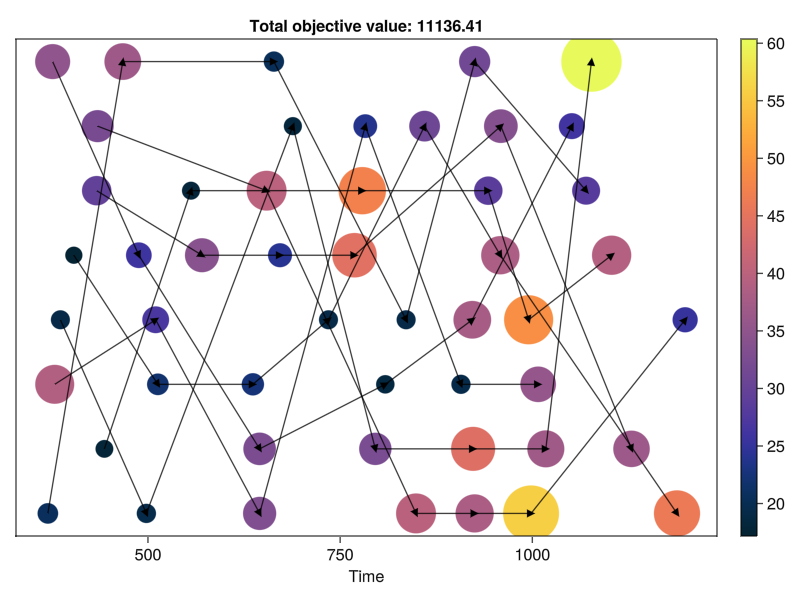}
        \caption{Model learned by experience}
    \end{subfigure}
    \caption{Comparison of predictions of three different models on the same 50-task instance. Each circle represents a task, with color value and size corresponding to its total delay (intrinsic + propagated) in the predicted solution. Tasks are ordered respect to start time along the~$x$ axis. Each arrow path represents a vehicle path in the predicted solution.}
    \label{fig:vsp-prediction}
\end{figure}

In Table~\ref{tab:vsp-time}, we compare the average execution time between the local search heuristic and one of our learned pipelines. We see that, once the model is learned offline, online prediction of a solution through the pipeline is very fast.

\begin{table}[H]
    \centering
    \begin{tabular}{cccccc}
\toprule
\multicolumn{6}{c}{\textbf{Test dataset}} \\ \midrule
\multicolumn{2}{c}{25 tasks} & \multicolumn{2}{c}{50 tasks} & \multicolumn{2}{c}{100 tasks}\\
\midrule
heuristic & learned pipeline & heuristic & learned pipeline & heuristic & learned pipeline\\
\midrule
$0.12s$ & $0.004s$ & $1.47s$ & $0.01s$ & $5.28s$ & $0.043s$\\
\bottomrule
\end{tabular}

    \caption{Learning by imitation: comparison of average execution time on an instance between the local search heuristic (10000 iterations) a learned pipeline}
    \label{tab:vsp-time}
\end{table}

\paragraph{Learning by imitation}

We train three models in the learning by imitation setting, one for each training dataset. Final results on test data are summarized in Tables~\ref{tab:vsp-imitation-gap} and~\ref{tab:vsp-imitation-cost}.
First, we observe the average and maximum cost gaps on test datasets with instances of the same size as in training datasets, between predicted solutions and labeled solutions. We recall that labeled solutions are optimal for 25 tasks, and heuristic (obtained from local search) for 50 and 100 tasks.
That is why gaps should stay positive when predicting on the 25 tasks test dataset, but can be negative when predicting on the two other datasets.
We can see that all three models obtain very good gap values, with less than~$1\%$ on 25 tasks instances on average, and average negative gaps on~$50$ and~$100$ tasks instances, which means better solutions on average than the local search heuristic used as labels for the learning.

\begin{table}[H]
    \centering
    \begin{tabular}{ccccccc}
\toprule
\multirow{3}{*}{\textbf{Train dataset}} & \multicolumn{6}{c}{\textbf{Test dataset}} \\ \cline{2-7}
 & \multicolumn{2}{c}{25 tasks} & \multicolumn{2}{c}{50 tasks} & \multicolumn{2}{c}{100 tasks}\\
\cline{2-7}
 & mean & max & mean & max & mean & max\\
\midrule
25 tasks & $0.68\%$ & $9.46\%$ & $-0.41\%$ & $4.26\%$ & $-1.02\%$ & $2.4\%$\\
\midrule
50 tasks & $0.49\%$ & $3.01\%$ & $-0.46\%$ & $2.34\%$ & $-1.6\%$ & $0.62\%$\\
\midrule
100 tasks & $0.62\%$ & $3.36\%$ & $-0.14\%$ & $9.9\%$ & $-1.2\%$ & $0.11\%$\\
\bottomrule
\end{tabular}

    \caption{Learning by imitation: cost gap}
    \label{tab:vsp-imitation-gap}
\end{table}
Then, we measure the performance of our models on test datasets with larger instances than training ones.
We do not necessarily have label solutions available for these big instances, but can evaluate the \emph{average cost per task}.
This metric is very useful to assess how well a model generalizes. Indeed, since all instances share the same map size regardless of the number of tasks, the average cost per task should decrease when the number of tasks increases, because optimal paths contain more tasks. It is the case for the model trained on 25 tasks instances, but not for the other two, which deteriorate on instances larger than 300 tasks.
This is probably due to the fact that 50 and 100 tasks datasets are not labeled with optimal solutions, but only with solutions from the local search heuristic. Indeed, the 100-task model performs even worse than the 50-task model, because the local search heuristic performance decreases when the number of tasks increases.

\begin{table}[H]
    \centering
    \begin{tabular}{ccccccccc}
\toprule
\multirow{2}{*}{\textbf{Train dataset}} & \multicolumn{8}{c}{\textbf{Test dataset} (number of tasks in each instance)} \\ \cline{2-9}
 & 25 & 50 & 100 & 200 & 300 & 500 & 750 & 1000\\
\midrule
25 tasks & 274.72 & 225.29 & 207.14 & 194.46 & 186.68 & 182.56 & 178.57 & 177.3\\
\midrule
50 tasks & 274.27 & 225.23 & 205.97 & 195.78 & 193.12 & 194.48 & 196.99 & 199.38\\
\midrule
100 tasks & 274.61 & 225.87 & 206.8 & 197.97 & 195.53 & 207.02 & 219.34 & 227.14\\
\bottomrule
\end{tabular}

    \caption{Learning by imitation: average cost per task}
    \label{tab:vsp-imitation-cost}
\end{table}

\paragraph{Learning by experience} In the learning by experience setting, we also train three models on the same three training datasets, and summarize results in  Tables~\ref{tab:vsp-experience-gap} and~\ref{tab:vsp-experience-cost}. First, we observe that average cost gaps and maximum cost gaps are lower on all models, respect to the corresponding models learned by imitation.

\begin{table}[H]
    \centering
    \begin{tabular}{ccccccc}
\toprule
\multirow{3}{*}{\textbf{Train dataset}} & \multicolumn{6}{c}{\textbf{Test dataset}} \\ \cline{2-7}
 & \multicolumn{2}{c}{25 tasks} & \multicolumn{2}{c}{50 tasks} & \multicolumn{2}{c}{100 tasks}\\
\cline{2-7}
 & mean & max & mean & max & mean & max\\
\midrule
25 tasks & $0.45\%$ & $4.2\%$ & $-0.77\%$ & $0.63\%$ & $-2.11\%$ & $-0.14\%$\\
\midrule
50 tasks & $0.43\%$ & $3.04\%$ & $-0.78\%$ & $0.74\%$ & $-2.06\%$ & $-0.22\%$\\
\midrule
100 tasks & $0.43\%$ & $3.28\%$ & $-0.83\%$ & $0.97\%$ & $-2.06\%$ & $-0.29\%$\\
\bottomrule
\end{tabular}

    \caption{Learning by experience: cost gap}
    \label{tab:vsp-experience-gap}
\end{table}

Finally, unlike the learning by imitation setting, we observe decreasing average cost per task on all three models for instances up to 1000 tasks. This means models learned by experience generalize much better, because they are not biased trying to imitate non-optimal solutions.

\begin{table}[H]
    \centering
    \begin{tabular}{ccccccccc}
\toprule
\multirow{2}{*}{\textbf{Train dataset}} & \multicolumn{8}{c}{\textbf{Test dataset} (number of tasks in each instance)} \\ \cline{2-9}
 & 25 & 50 & 100 & 200 & 300 & 500 & 750 & 1000\\
\midrule
25 tasks & 274.19 & 224.55 & 204.9 & 191.86 & 184.71 & 181.29 & 178.0 & 177.02\\
\midrule
50 tasks & 274.12 & 224.51 & 205.0 & 191.85 & 184.3 & 180.48 & 176.96 & 176.0\\
\midrule
100 tasks & 274.13 & 224.41 & 205.0 & 191.85 & 184.63 & 181.08 & 177.81 & 176.74\\
\bottomrule
\end{tabular}

    \caption{Learning by experience: average cost per task}
    \label{tab:vsp-experience-cost}
\end{table}

In the stochastic VSP, learning by experience seems more performant than learning by imitation, the only downside being the practical difficulty of the training. Indeed, results are more sensible to hyperparameters and the model needs to be trained for more epochs because the loss convergence is slower.

\subsection{Single-machine scheduling}\label{subsec:scheduling}

We solve a single-machine scheduling problem with release times and sum of completion times. Given an instance of this NP-hard problem, we use an ML predictor to define the input of a ranking problem, which is an LP on the permutahedron and can be solved instantaneously. Its optimal solution (permutation) can then be decoded into an approximate solution of the hard problem. This follows the recent work of \textcite{parmentierLearningSolveSingle2021}, re-coded in Julia in the satellite package \texttt{SingleMachineScheduling.jl}\footnote{\url{https://github.com/axelparmentier/SingleMachineScheduling.jl}}.  

\subsubsection{Problem formulation}

We need to schedule~$n$ jobs on a single machine. Each job~$j \in [n]$ has an associated processing time~$p_j$, and release date~$r_j$. Before~$r_j$, the job~$j$ cannot start. The machine can process only one job at a time. Preemption is not allowed, which means whenever a job has started, it must be completed, and cannot be paused to be finished later. Our aim is to find a schedule (permutation)~$s = (j_1, ..., j_n)$ of~$[n]$, such that the sum of completion times~$\sum_j C_j(s)$ is minimal. In this sum,~$C_j(s)$ is the completion time of job~$j$ in the schedule~$s$. More precisely, given the schedule~$s$, we have~$C_{j_1}(s) = r_{j_1} + p_{j_1}$, and for~$k>1$, we have~$C_{j_k}(s) = \max \big(r_{j_k}, C_{j_{k-1}}(s)\big) + p_{j_k}$.

\subsubsection{Existing algorithms for the hard problem}\label{subsubsec:scheduling_algo}
Two kinds of algorithms are considered, both to benchmark our learning pipelines, and to provide solutions for the learning dataset. The first kind includes several heuristics: the RDI/APRTF heuristic \parencite{chandIterativeHeuristicSingle1996}, referred to as \texttt{RDIA}, the RBS heuristic with beam width~$w = 2$ \parencite{dellacroceHybridHeuristicApproach2014}, referred to as \texttt{RBS}, and the matheuristic \parencite{dellacroceHybridHeuristicApproach2014}, referred to as \texttt{MATH}. The second kind is an exact algorithm, used to solve instances with up to~$n = 110$ jobs, the branch-and-memorize algorithm \parencite{shangBranchMemorizeExact2021}.

\subsubsection{Datasets}
We consider two types of datasets generated randomly: 
\begin{itemize}
    \item A train dataset made of~$420$ instances with~${n \in \{50, 70, 90, 110\}}$ jobs, and corresponding solutions derived by means of the branch-and-memorize algorithm \parencite{shangBranchMemorizeExact2021}. Each value of~$n$ is equally represented.
    \item A test dataset made of~$1820$ instances with~$n \in \{50+k \times 100, k \in [25]\}$ jobs, and the associated best solutions derived by the algorithms noted above. Each value of~$n$ is equally represented.
\end{itemize}

\subsubsection{InferOpt pipeline}
In the pipeline \eqref{eq:scheduling_learning_pipeline}, an instance of the hard problem~$x$ is passed through an ML layer involving the GLM~$\varphi_w$, leading to the parameter~$\theta = \varphi_w(x)$. We then sort the components of~$\theta$ using a ranking algorithm. It leads to a permutation~$s$. This permutation is then passed through a decoder. The output of the decoder~$y$ is the candidate solution to the hard problem.

\begin{equation}
    \xrightarrow[]{\begin{Bcenter}Hard problem\\ instance~$x$ \end{Bcenter}}
    \ovalbox{\begin{Bcenter}ML layer~$\varphi_w$ \end{Bcenter}}
    \xrightarrow[]{\begin{Bcenter}Parameter \\~$\theta = \varphi_w(x)$\end{Bcenter}}
    \doublebox{\begin{Bcenter} Ranking $f$ \end{Bcenter}}
    \xrightarrow[]{\begin{Bcenter}Permutation \\~$s = f(\theta)$\end{Bcenter}}
    \ovalbox{Decoder $\psi$}
    \xrightarrow[]{\begin{Bcenter}Solution \\~$y = \psi(s)$\end{Bcenter}}
    \label{eq:scheduling_learning_pipeline}
\end{equation}

\textcite{parmentierLearningSolveSingle2021} consider two possibilities for the decoder~$\psi$. First, we can use the permutation~$s = f(\theta)$ as the candidate solution of the hard problem. We call this algorithm \emph{pure machine learning heuristic}, referred to as \texttt{PMLH}. It corresponds to~$\psi = \operatorname{id}$. Second, in the \emph{improved machine learning heuristic}, referred to as \texttt{IMLH}, we combine a local search and the RDI procedure in the decoder~$\psi$. We consider \texttt{PMLH} in the following.

\subsubsection{Experimental setting}
\paragraph{A fast ML-CO pipeline.}
We focus on a learning by imitation setting, with~$\psi = \operatorname{id}$. Since our train dataset is made of small instances (with up to~$110$ jobs), and our test dataset has large instances (with up to~$2550$ jobs), we want to see if our model generalizes well. Since ranking is quasi-instantaneous, and the algorithms defined in Section~\ref{subsubsec:scheduling_algo} require much more computational effort, small average gaps on the test set would lead to an ML-CO algorithm that saves a lot of computation time on large-scale instances.
\paragraph{Learning process.}
In the learning phase, we use the permutation~$s = f(\theta)$ as candidate solution to be compared with the associated target solution in the dataset. We use three types of probabilistic CO layers and associated Fenchel-Young loss functions: additive perturbation, multiplicative perturbation, and half square norm regularization. Our hyperparameters are detailed in Table~\ref{tab:hyperparameters_scheduling}. The value of~$\varepsilon$ is tuned with a grid search on a validation dataset for each probabilistic CO layer, selecting the one leading to the smallest average gap. We use a batch size of~$1$ sample and learn through~$300$ epochs in each case.
Experiments are done on a computing cluster with processor 32 x Intel Xeon E5-2667, 3.30GHz (hyperthreading), and 192 Go of RAM.

\begin{table}
    \centering
    \begin{tabular}{ccc}
        \toprule
        \textbf{Probabilistic CO layer} &~$\varepsilon$ & \texttt{nb\_samples}          \\ \midrule
        Additive perturbation   &~$0.01$ &~$100$\\
        Multiplicative perturbation  & ~$0.2$ &~$100$  \\
        Half square norm  &  ~$0.2$ &       \\
        \bottomrule
    \end{tabular}
    \caption{Hyperparameters for scheduling experiments.}
    \label{tab:hyperparameters_scheduling}
\end{table}

\subsubsection{Results}

In Table~\ref{tab:scheduling_gaps}, we compare the average train and test gaps given by each of the three probabilistic CO layers detailed in Table~\ref{tab:hyperparameters_scheduling}, as well as the GLM-based pipeline from \textcite{parmentierLearningSolveSingle2021}. In each case, we derive solutions in the following way: given an instance~$x$, we follow the pipeline \eqref{eq:scheduling_learning_pipeline}, where the weights of the GLM~$w$ are fixed thanks to the learning phase based on the specific probabilistic CO layer (or by \textcite{parmentierLearningSolveSingle2021}), and where the decoder is~$\psi = \operatorname{id}$. We recall that on the train set, gaps are computed with respect to the optimal solutions. On the test set, we consider the best solutions found by the algorithms detailed in Section~\ref{subsubsec:scheduling_algo}. Overall, we manage to obtain low test gaps. We even get better solutions than the heuristics with our ML-CO pipelines on the test set.

\begin{table}[H]
    \centering
    \begin{tabular}{lccc}
    \toprule
                 & \textbf{Train set gap (\%)}   & \textbf{Test set gap (\%)}   \\
    \toprule
    \textbf{Additive perturbation (GLM)} & 1.68             & -0.26     \\
    \midrule
    \textbf{Multiplicative perturbation (GLM)} & 3.02             & 0.34    \\
    \midrule
    \textbf{Half square norm regularization (GLM)} & 1.52             & 0.59        \\
    \midrule
    \textbf{\textcite{parmentierLearningSolveSingle2021} model (GLM)} & 2.32             & 1.93        \\
    \bottomrule
    \end{tabular}
    \caption{Train and test average gaps in \% for pipelines with decoder~$\psi = \operatorname{id}$.}
    \label{tab:scheduling_gaps}
\end{table}

In Figure~\ref{fig:gaps_scheduling}, we show the histograms of train and test gaps given by the pipeline trained with the additive perturbation, which gives the best average results in Table~\ref{tab:scheduling_gaps}. It emphasizes a good generalization performance, even trained on small instances.
Since ranking is very fast, it highlights the practical interest of this kind of ML-CO algorithm. For further studies in a learning by experience setting, as well as additional details on the decoder part, we refer to \textcite{parmentierLearningStructuredApproximations2021} and \textcite{parmentierLearningSolveSingle2021}.

\begin{figure}[H]
    \includegraphics[width=0.49\linewidth,clip]{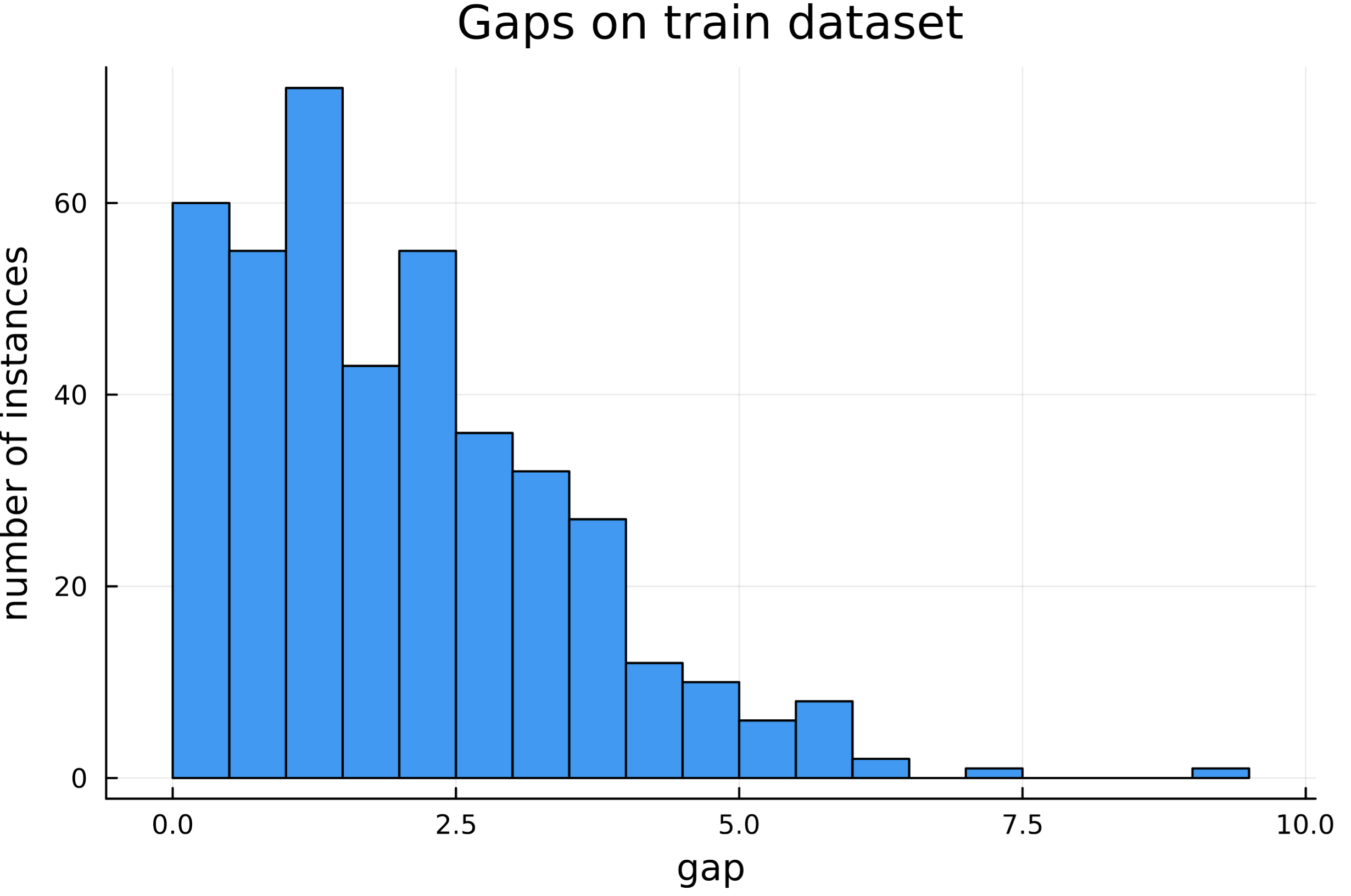}
    \includegraphics[width=0.49\linewidth,clip]{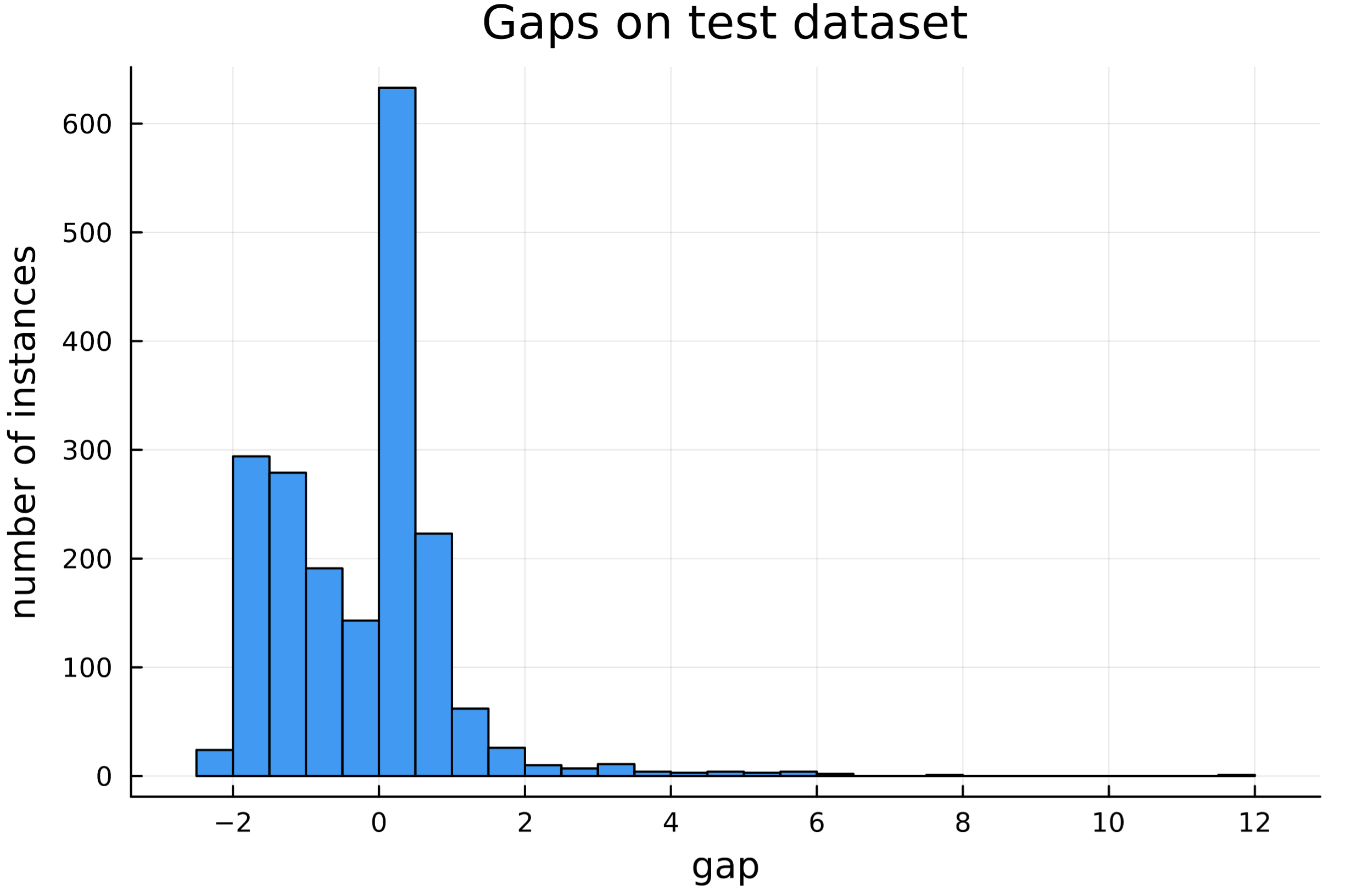}
    \caption{Histograms of train and test gaps in \% given by the pipeline trained with the additive perturbation CO layer and decoder~$\psi = \operatorname{id}$.}
    \label{fig:gaps_scheduling}
\end{figure}

\subsection{Two-stage stochastic minimum weight spanning tree} \label{subsec:spanning_tree}

Let us now focus on the \emph{two-stage minimum weight spanning tree problem}. This last subsection has two main purposes: first, having an application on a two stage stochastic optimization problem where we use the single stage problem as a CO layer, and second, we show how to use \texttt{InferOpt.jl} with a pipeline containing a graph neural network (GNN), and compare it to a basic generalized linear model (GLM). Source code for this application can be found in the corresponding satellite package \texttt{TwoStageSpanningTree.jl}\footnote{\url{https://github.com/axelparmentier/MinimumWeightTwoStageSpanningTree.jl}}.

\subsubsection{Problem formulation}

Let~$G = (V,E)$ be an undirected graph, and~$S$ be a (finite) set of scenarios. The goal is to build a spanning tree on~$G$ over two stages at minimum cost, the second-stage building costs being unknown when first-stage decisions are taken. For each edge~$e\in E$ and scenario and~$s\in S$, we denote by~$c_e$ in~$\mathbb{R}$ the scenario-independent first stage cost of building~$e$, and by~$d_{es}$ the scenario-dependent second stage cost. 
The two-stage spanning tree problem can be formulated as the following ILP:
\begin{subequations}
    \begin{empheq}{align}
        \min\limits_{y, z}\,& \displaystyle\sum_{e \in E} c_e y_e + \frac{1}{|S|}\sum_{e \in E} \sum_{s \in S} d_{es} z_{es},\\
        \mathrm{s.t.}\,&\displaystyle\sum_{e \in E} y_e + z_{es} = |V|-1, & \forall s \in S, \label{subeq:tree1} \\
        &\displaystyle\sum_{e \in E(Y)} y_e + z_{es} \leq |Y| - 1, \qquad &\forall Y,\, \emptyset \subsetneq Y \subsetneq V,\, \forall s \in S, \label{subeq:tree2} \\
        & y_e \in \{0,1\}, & \forall e \in E, \\
        & z_{es} \in \{0,1\}, & \forall e\in E,\, \forall s\in S,
    \end{empheq}
    \label{eq:two-stage-spanning-tree}
\end{subequations}
where~$y_e$ is a binary variable indicating if~$e$ is selected in the first-stage solution, and~$z_{es}$ is a binary variable indicating if~$e$ is selected in the second-stage solution for scenario~$s$. The two constraints~\eqref{subeq:tree1} and~\eqref{subeq:tree2} jointly ensures that~$y_e + z_{es}$ belongs to the spanning tree polytope.

Formulation~\eqref{eq:two-stage-spanning-tree} contains an exponential number of constraints.
Several mathematical programming approaches enable to solve it.
Instances with up to 50 vertices can be solved to optimality using a cut-generation approach or a Benders decomposition.
A Lagrangian relaxation and a Lagrangian heuristic can provide solutions within a 2\% optimality gap for instances with up to 10000 vertices in approximately one hour.
A detailed description of this heuristic can be found in Appendix~\ref{sub:appendix_lagrangian_heuristic}.

\subsubsection{Datasets}

We build three datasets, one for training, one for validation, and one for testing. Each dataset contains 600 grid graphs instances whose widths range from 10 to 60 (that is, with~$100$ to~$3600$ vertices), and from~$5$ to~$20$ scenarios. For all instances, first stage costs are drawn uniformly between~$0$ and~$20$. For the second stage costs, depending on the instances, they are drawn uniformly between~$0$ and~$n$, with~$n$ varying between~$10$ and~$30$. Each instance is labeled with a solution given by the Lagrangian heuristic.

\begin{figure}[H]
    \centering
    \includegraphics{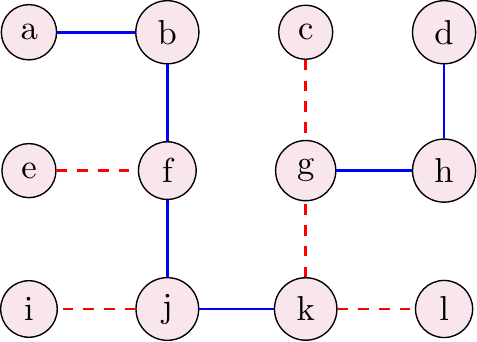}
    \caption{Example of a two stage spanning tree on a~$3\times 4$ grid graph. Blue edges correspond to first-stage forest, and red dashed edges correspond to second-stage forest.}
    \label{fig:two-stage-spanning-tree}
\end{figure}

\subsubsection{InferOpt pipeline}

The learning pipeline is presented in Figure~\ref{fig:two-stage-spanning-tree-pipeline}.

\begin{figure}[H]
    \centering
    \includegraphics{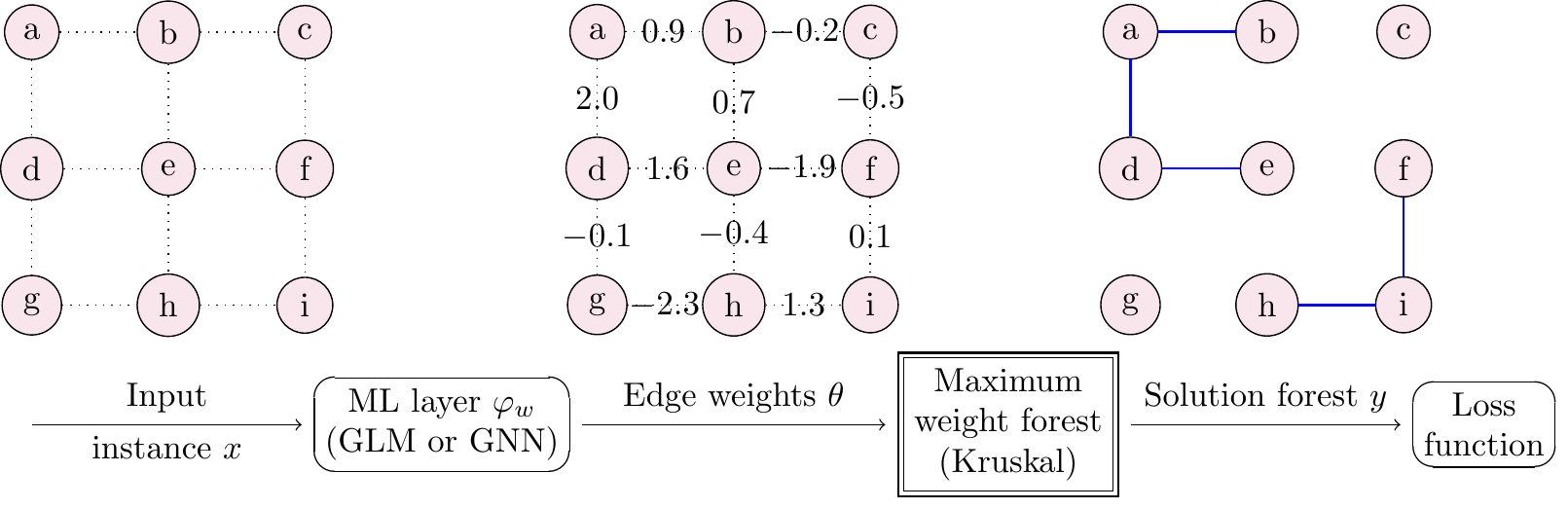}
    \caption{Two-stage minimum spanning tree pipeline}
    \label{fig:two-stage-spanning-tree-pipeline}
\end{figure}

\paragraph{Input}
Each edge~$e$ of an instance~$x$ is encoded by a feature vector~$\phi(x, e)$. In the following experiments, we test two different set of features.
The first one is called \textit{Basic features}, and only contains information about first-stage costs and quantiles of second-stage costs.
The second feature set is called \textit{Advanced features}, and contains multiple other statistics.
See Appendix~\ref{subsec:two-stage-spanning-tree-features} for a detailed description of the features.

\paragraph{ML predictor}
The feature matrix is given as input to an ML layer~$\varphi_w$ with learnable parameters~$w$, which predicts edge weights~$\theta_e$.
We tested and compared two different predictors: first a GLM such that~$\theta_e = w^\top \phi(x, e)$, and then a GNN, in particular the \texttt{GraphConv} architecture (\cite{morrisWeisfeilerLemanGo2019}). 1

\paragraph{Combinatorial layer}
We use the predicted edge weights~$\theta$ as the objective of a maximum weight forest problem layer. A maximum weight forest can be efficiently found using Kruskal's algorithm.

\paragraph{Decoder}
As a post-processing, the output forest of the training pipeline is interpreted as a first stage solution of instance~$x$.
For each scenario, we can complete it into a minimum cost spanning tree containing this forest, using Kruskal's algorithm.
%It can be decoded into a two-stage spanning tree solution of~$x$ by applying Kruskal's algorithm on the graph initialized with high weights on already selected first stage edges (in order to make sure to select them), and minimum of second stage weights for each edge not included in the first stage solution.

\subsubsection{Experimental setting}

We train each model by imitation using a \texttt{FenchelYoungLoss} with \texttt{PerturbedAdditive} regularization. Hypeparameters used for each model can be found in table~\ref{tab:hyperparameters_spanning-tree}.

\begin{table}[h]
    \centering
    \begin{tabular}{cccccc}
    \toprule
        & $\varepsilon$ & \texttt{nb\_samples} & learning rate & batch size & training epochs \\
    \midrule
    GLM & 1             & 20         & 0.01 & 1 & 20  \\
    \midrule
    GNN & 1             & 20         & 0.001 & 1 & 20 \\
    \bottomrule
    \end{tabular}
    \caption{Hypeparameters for two-stage spanning tree experiments}
    \label{tab:hyperparameters_spanning-tree}
\end{table}

We use the \texttt{Flux.jl} Julia library and the Adam optimizer for training.
We construct our GNN with the \texttt{GraphNeuralNetworks.jl}\footnote{\url{https://github.com/CarloLucibello/GraphNeuralNetworks.jl}} Julia library \parencite{lucibelloGraphNeuralNetworksJlGeometric2021}. The GNN architecture contains 3 \texttt{GraphConv} layers with hidden size 50.

The GLM and GNN models are both trained on each of the two feature sets. Features are normalized by dividing each feature by its train dataset standard deviation. We optimize hyperparameters and perform model selection by using the average optimality gap values on the validation dataset. The experiments are conducted on a 2.6 GHz Intel Core i9-11950H, 16 cores, with 64 Go RAM. Epochs with best gaps on the validation dataset are selected for each experiment.

\subsubsection{Results}

The CO layer takes the same CPU time as one iteration of the Lagrangian relaxation, and the feature computation takes about the same time as~$2$ iterations. Therefore, resulting pipelines are roughly 15000 times faster than the Lagrangian heuristic.

\medskip
We evaluate the performance of the trained pipelines on the test dataset. For this we compute the average, minimum, and maximum gaps respect to the lower bound and upper bounds respectively given by the Lagrangian relaxation and the Lagrangian heuristic. Results are gathered in Table~\ref{tab:two-stage-spanning-tree-basic} below.

\begin{table}[H]
    \centering
    \begin{tabular}{lcccccc}
    \toprule
                 & \multicolumn{3}{c}{\textbf{Optimality gap (\%)}}   & \multicolumn{3}{c}{\textbf{Lagrangian heuristic gap (\%)}}   \\
                 & \textbf{Average} & \textbf{Min} & \textbf{Max} & \textbf{Average} & \textbf{Min} & \textbf{Max} \\
    \toprule
    \textbf{GLM on basic features} & 4.77             & 1.08         & 12.33        & 2.91             & 0.71         & 8.81         \\
    \midrule
    \textbf{GNN on basic features} & 3.06             & 0.09         & 12.09        & 1.23             & -0.24        & 7.07\\
    \midrule
    \textbf{GLM on advanced features} & 2.21             & 0.09         & 7.66         & 0.41             & -0.73        & 3.12         \\
    \midrule
    \textbf{GNN on advanced features} & 2.11             & 0.09         & 7.91         & 0.31             & -0.73        & 2.86\\
    \bottomrule
    \end{tabular}
    \caption{Test gaps in \% for models trained with the basic features}
    \label{tab:two-stage-spanning-tree-basic}
\end{table}

\noindent On both feature sets we observe low gaps for both GLM and GNN, with the GNN performing better. We also see that the GLM trained on advanced features performs better that the GNN on basic features, which means that feature engineering is more effective than replacing the linear model with a more complex neural network. Advanced features combined with the GNN performs best. As can be seen in section~\ref{subsec:StochasticVSP} on the stochastic  VSP, we could maybe further improve performance if we learn by experience instead of learning by imitation. Indeed, the labels computed by the Lagrangian heuristic are not necessarily optimal and probably introduce bias in the learning.

\section{Conclusion} \label{sec:conclusion}

In this paper, we introduce \texttt{InferOpt.jl}, a Julia package that seamlessly turns CO algorithms into layers of ML pipelines.
The package is compatible with most approximate differentiation techniques and structured losses from the literature, but also includes novel ones.
Its theoretical foundation is a probabilistic point of view on CO layers, which unifies existing and new methods.
Our probabilistic framework notably provides natural \enquote{learning by experience} counterparts to the more common \enquote{learning by imitation} approaches.
\texttt{InferOpt.jl} also unlocks the use of large neural networks in hybrid pipelines, even when learned by experience.
We demonstrate it with a hybrid ML-CO pipeline containing a CNN with tens of thousands of parameters, which to the best of our knowledge, has never been achieved.

Using \texttt{InferOpt.jl}, we benchmark pipeline architectures and learning algorithms.
In general, we observe that hybrid pipelines can yield fast algorithms with state-of-the-art performance on well-studied applications such as the single machine scheduling problem.
Depending on the learning setting, we may require more or less information in the training set.
Unsurprisingly, the more information we have in the training set, the better our predicted solutions become. This is notably clear on the Warcraft shortest path problem.
However, on very hard combinatorial tasks where no algorithm scales on larges instances, learning by experience can outperform learning by imitation, as can be seen on the stochastic vehicle scheduling problem.

When we learn by imitating precomputed solutions, Fenchel-Young losses seem to outperform alternatives.
If we use a CO algorithm that is sensitive to the sign of the input, the multiplicative perturbation seems the best approach.
Concerning the pipeline architecture, the most important aspect seems to be the choice of the combinatorial optimization layer and of the features.
Using GNNs instead of GLMs for encoding can be beneficial, as we remarked on the two-stage spanning tree problem.

\medskip

We also highlight some perspectives. In this paper, we learn the objective function of CO problems based on the notion of probabilistic CO layers. We do not consider the case when the desired flexibility concerns the constraints, as studied by \textcite{paulusCombOptNetFitRight2021} for instance, even though this perspective would be equally natural.
Furthermore, we make the assumption that the CO oracle derives optimal solutions within our layer. We briefly mention the case of inexact CO oracles, but additional research would be necessary in this direction.
Finally, a large part of the ML literature is dedicated to reinforcement learning. The connection between this field and structured learning deserves further exploration.

\subsection*{Acknowledgements}

The authors want to thank Mathieu Besançon for his support with various Julia packages, as well as his expert proofreading.
We also thank Éloïse Berthier, Alexandre Forel, Jérôme Malick, Clément Mantoux, and  Pierre Marion for their time and advice.

\addcontentsline{toc}{section}{References}
\printbibliography

\appendix

\section{Proofs} \label{sec:proofs}

In this Section, the dominated convergence theorem is used implicitly to justify any differentiation under the integral sign.
When dealing with polyhedral functions such as~$\theta \longmapsto \max_{v \in \calV} \theta^\top v$, we often write~$\nabla_\theta$ for simplicity even though they are only subdifferentiable, because the set of non-differentiability has measure zero.
We denote the standard Gaussian density by~$\nu$.

\subsection{Additive perturbation}

These proofs are already given by \textcite{berthetLearningDifferentiablePerturbed2020}, but we include them for comparison purposes.

\subsubsection{Derivatives} \label{sec:proof_additive_perturbation}

Proof of Proposition~\ref{prop:additive_perturbation}.

\begin{proof}

    The following change of variable is a diffeomorphism:
    \begin{equation*}
        u = \theta + \varepsilon z \quad \iff \quad \frac{u - \theta}{\varepsilon} = z.
    \end{equation*}
    We apply it to the definition of~$\widehat{p}_\varepsilon^+(v|\theta)$.
    \begin{align*}
        \widehat{p}_\varepsilon^+(v | \theta)
         & = \int_{\bbR^d} \oneindicator{\{f(\theta + \varepsilon z) = v\}} ~ \nu(z) ~ \rmd z                                       \\
         & = \int_{\bbR^d} \oneindicator{\{f(u) = v\}} ~ \nu\left(\frac{u-\theta}{\varepsilon}\right) \frac{\rmd u}{\varepsilon^d}.
    \end{align*}
    We now differentiate with respect to~$\theta$ before applying the reverse change of variable:
    \begin{align*}
        \nabla_\theta \widehat{p}_\varepsilon^+(v | \theta)
         & = \int_{\bbR^d} \oneindicator{\{f(u) = v\}} ~ \left( \frac{-1}{\varepsilon} \nabla \nu\left(\frac{u-\theta}{\varepsilon}\right) \right) \frac{\rmd u}{\varepsilon^d} \\
         & = \frac{-1}{\varepsilon} \int_{\bbR^d} \oneindicator{\{f(\theta + \varepsilon z) = v\}} ~ \nabla \nu(z) ~ \rmd z                                                     \\
         & = \frac{1}{\varepsilon} \int_{\bbR^d} \oneindicator{\{f(\theta + \varepsilon z) = v\}} ~  z \nu(z) ~ \rmd z.
    \end{align*}
    The last equality holds because the standard Gaussian density satisfies~$\nabla \nu(z) = -z \nu(z)$.
    From there, we deduce the Jacobian of~$\widehat{f}_\varepsilon^+ (\theta)$:
    \begin{align*}
        J_\theta \widehat{f}_\varepsilon^+(\theta)
         & = \sum_{v \in \calV} v \nabla_\theta \widehat{p}_\varepsilon^+(\theta, y)^\top                                                                                                              \\
         & = \frac{1}{\varepsilon} \int_{\bbR^d} \underbrace{\Big(\sum_{v \in \calV} v ~ \oneindicator{\{f(\theta + \varepsilon z) = v\}} \Big)}_{f(\theta + \varepsilon z)} ~  z^\top \nu(z) ~ \rmd z
    \end{align*}
    We arrive at the following simple expression, which was already given by \textcite{berthetLearningDifferentiablePerturbed2020}:
    \begin{equation*}
        J_\theta \widehat{f}_\varepsilon^+(\theta) = \frac{1}{\varepsilon} \bbE \left[f(\theta + \varepsilon Z) Z^\top\right]
    \end{equation*}

\end{proof}

\subsubsection{Regularization} \label{sec:proof_additive_perturbation_omega}

Proof of Proposition~\ref{prop:additive_perturbation_omega}.

\begin{proof}
    Because~$\Omega_\varepsilon^+ = (F_\varepsilon^+)^*$ is a Fenchel conjugate, it is automatically convex.
    Furthermore,
    \begin{align*}
        \Omega_\varepsilon^+(\mu)
         & = \sup_{\theta \in \bbR^d} \left\{\theta^\top \mu - F_\varepsilon^+(\theta)\right\}                                               \\
         & = \sup_{\theta \in \bbR^d} \left\{\theta^\top \mu - \bbE \left[\max_{v \in \calV} (\theta + \varepsilon Z)^\top v\right]\right\}.
    \end{align*}
    We consider~$\mu \notin \conv(\calV)$.
    By convex separation, there exists~$\tilde{\theta} \in \bbR^d$ and~$\alpha > 0$ such that~$\tilde{\theta}^\top \mu \geq \alpha + \tilde{\theta}^\top v$ for all~$v \in \calV$.
    This implies that, for all~$t > 0$,
    \begin{align*}
        \Omega_\varepsilon^+(\mu)
         & \geq t \tilde{\theta}^\top \mu - \bbE \left[\max_{v \in \calV} \left(t\tilde{\theta} + \varepsilon Z\right)^\top v\right]                                 \\
         & \geq t \tilde{\theta}^\top \mu - t \bbE \left[\max_{v \in \calV} \tilde{\theta}^\top v\right] - \varepsilon \bbE \left[\max_{v \in \calV} Z^\top v\right] \\
         & \geq t \alpha - \varepsilon \bbE \left[\max_{v \in \calV} Z^\top v\right] \xrightarrow[t \to +\infty]{} +\infty
    \end{align*}
    We have shown that~$\mu \notin \dom(\Omega_\varepsilon^+)$, and therefore~$\dom(\Omega_\varepsilon^+) \subset \conv(\calV)$.
    We define
    \begin{equation*}
        f_Z(\theta, v) = (\theta + \varepsilon Z)^\top v \qquad \text{so that} \qquad F_\varepsilon^+(\theta) = \bbE \left[\max_{v \in \calV} f_Z(\theta, v)\right].
    \end{equation*}
    Danskin's theorem helps us compute the gradient of~$F_\varepsilon^+$:
    \begin{align*}
        \nabla_\theta F_\varepsilon^+(\theta)
         & = \bbE \left[\nabla_\theta \left(\max_{v \in \calV} f_Z(\theta, v)\right)\right]                                \\
         & = \bbE \left[\nabla_1 f_Z\left(\theta, \argmax_{v \in \calV} f_Z(\theta, v)\right)\right]                       \\
         & = \bbE \left[ \argmax_{v \in \calV} f_Z(\theta, v)\right]                                                       \\
         & = \bbE \left[ \argmax_{v \in \calV} (\theta + \varepsilon Z)^\top v\right] = \widehat{f}_\varepsilon^+(\theta).
    \end{align*}
    As shown by \textcite[Proposition 2.2]{berthetLearningDifferentiablePerturbed2020}, the function~$\Omega_\varepsilon^+$ is a Legendre type function, which means that
    \begin{equation*}
        \nabla_\theta F_\varepsilon^+ = \nabla_\theta (\Omega_\varepsilon^+)^* = (\nabla_\theta \Omega_\varepsilon^+)^{-1}.
    \end{equation*}
    From this, we deduce
    \begin{align*}
        \nabla_\theta F_\varepsilon^+(\theta)
         & = \argmax_{\mu \in \bbR^d} \left\{\theta^\top \mu - \Omega_\varepsilon^+(\mu)\right\}                                                                   \\
         & = \argmax_{\mu \in \dom(\Omega_\varepsilon^+)} \left\{\theta^\top \mu - \Omega_\varepsilon^+(\mu)\right\} = \widehat{f}_{\Omega_\varepsilon^+}(\theta).
    \end{align*}
    Hence, we can conclude:
    \begin{equation*}
        \widehat{f}_\varepsilon^+(\theta) = \nabla_\theta F_\varepsilon^+(\theta) = \argmax_{\mu \in \dom(\Omega_\varepsilon^+)} \{\theta^\top \mu - \Omega_\varepsilon^+(\mu)\} = \widehat{f}_{\Omega_\varepsilon^+}(\theta).
    \end{equation*}
    To recover the formula given in Proposition~\ref{prop:additive_perturbation_omega}, we simply remember that~$\dom(\Omega_\varepsilon^+) \subseteq \conv(\calV)$.
\end{proof}

\subsection{Multiplicative perturbation}

\subsubsection{Derivatives} \label{sec:proof_multiplicative_perturbation}

Proof of Proposition~\ref{prop:multiplicative_perturbation}.

\begin{proof}

    Suppose~$\theta \in \bbR^d$ only has positive components.
    Then the following change of variable is a diffeomorphism:
    \begin{equation*}
        u = \theta \odot e^{\varepsilon z - \varepsilon^2 \onevector / 2} \quad \iff \quad \frac{\log(u) - \log(\theta)}{\varepsilon} + \frac{\varepsilon \onevector}{2} = z
    \end{equation*}
    We apply it to the definition of~$\widehat{p}_\varepsilon^\odot(v|\theta)$.
    \begin{align*}
        \widehat{p}_\varepsilon^\odot(\theta, y)
         & = \int_{\bbR^d} \oneindicator{\left\{f\left(\theta \odot e^{\varepsilon z - \varepsilon^2 \onevector / 2}\right) = v\right\}} ~ \nu(z) ~ \rmd z                                              \\
         & = \int_{(0,+\infty)^d} \oneindicator{\{f(u) = v\}} ~ \nu\left(\frac{\log(u) - \log(\theta)}{\varepsilon} + \frac{\varepsilon \onevector}{2}\right) \frac{\rmd u}{\varepsilon^d \prod_i u_i}.
    \end{align*}
    We now differentiate with respect to~$\theta$ before applying the reverse change of variable:
    \begin{align*}
        \nabla_\theta \widehat{p}_\varepsilon^\odot(\theta, y)
         & = \int_{(0,+\infty)^d} \oneindicator{\{f(u) = v\}} ~ \left( \frac{-1}{\varepsilon \theta} \odot \nabla \nu\left(\frac{\log(u) - \log(\theta)}{\varepsilon} + \frac{\varepsilon \onevector}{2}\right) \right) \frac{\rmd u}{\varepsilon^d \prod_i u_i} \\
         & = \frac{-1}{\varepsilon \theta} \odot \int_{\bbR^d} \oneindicator{\left\{f\left(\theta \odot e^{\varepsilon z - \varepsilon^2 \onevector / 2}\right) = v\right\}} ~ \nabla \nu(z) ~ \rmd z                                                            \\
         & = \frac{1}{\varepsilon \theta} \odot \int_{\bbR^d} \oneindicator{\left\{f\left(\theta \odot e^{\varepsilon z - \varepsilon^2 \onevector / 2}\right) = v\right\}} ~ z \nu(z) ~ \rmd z.
    \end{align*}
    From there, we deduce the Jacobian of~$\widehat{f}_\varepsilon^\odot(\theta)$:
    \begin{align*}
        J_\theta \widehat{f}_\varepsilon^\odot(\theta)
         & = \sum_{v \in \calV} v \nabla_\theta \widehat{p}_\varepsilon^\odot(\theta, y)^\top                                                                                                                                                                                                                                      \\
         & = \frac{1}{\varepsilon \theta} \odot \int_{\bbR^d} \underbrace{\Big(\sum_{v \in \calV} v ~ \oneindicator{\left\{f\left(\theta \odot e^{\varepsilon z - \varepsilon^2 \onevector / 2}\right) = v\right\}} \Big)}_{f\left(\theta \odot e^{\varepsilon z - \varepsilon^2 \onevector / 2}\right)} ~  z^\top \nu(z) ~ \rmd z
    \end{align*}
    We arrive at a simple variant of the previous expression:
    \begin{equation*}
        J_\theta \widehat{f}_\varepsilon^\odot(\theta) = \frac{1}{\varepsilon \theta} \odot \bbE \left[f\left(\theta \odot e^{\varepsilon z - \varepsilon^2 \onevector / 2}\right) Z^\top\right]
    \end{equation*}

\end{proof}

\subsubsection{Regularization} \label{sec:proof_multiplicative_perturbation_omega}

Proof of Proposition~\ref{prop:multiplicative_perturbation_omega}.

\begin{proof}
    Because~$\Omega_\varepsilon^\odot = (F_\varepsilon^\odot)^*$ is a Fenchel conjugate, it is automatically convex.
    Furthermore,
    \begin{align*}
        \Omega_\varepsilon^\odot(\mu)
         & = \sup_{\theta \in \bbR^d} \left\{\theta^\top \mu - F_\varepsilon^\odot(\theta)\right\}                                                                                              \\
         & = \sup_{\theta \in \bbR^d} \left\{\theta^\top \mu - \bbE \left[\max_{v \in \calV} \left(\theta \odot e^{\varepsilon Z - \varepsilon^2 \onevector / 2}\right)^\top v\right]\right\}   \\
         & = \sup_{\theta \in \bbR^d} \left\{\theta^\top \mu - \bbE \left[\max_{v \in \calV} \theta^\top \left(v \odot e^{\varepsilon Z - \varepsilon^2 \onevector / 2} \right)\right]\right\}.
    \end{align*}
    This last expression shows why we do not have~$\dom(\Omega_\varepsilon^\odot) \subset \conv(\calV)$ (unlike in the additive case).
    Indeed, even when~$\mu \notin \conv(\calV)$, the multiplicative scaling of~$v$ might allow it to compensate the inner product~$\theta^\top \mu$ and stop~$\Omega_\varepsilon^\odot(\mu)$ from going to~$+\infty$.
    We define
    \begin{equation*}
        f_Z(\theta, v) = \left(\theta \odot e^{\varepsilon Z - \varepsilon^2 \onevector / 2}\right)^\top v \qquad \text{so that} \qquad F_\varepsilon^\odot(\theta) = \bbE \left[\max_{v \in \calV} f_Z(\theta, v)\right].
    \end{equation*}
    Danskin's theorem helps us compute the gradient of~$F_\varepsilon^\odot$:
    \begin{align*}
        \nabla_\theta F_\varepsilon^\odot(\theta)
         & = \bbE \left[\nabla_\theta \left( \max_{v \in \calV} f_Z(\theta, v) \right) \right]                                                                                                 \\
         & = \bbE \left[\nabla_1 f_Z \left(\theta, \argmax_{v \in \calV} f_Z(\theta, v) \right) \right]                                                                                        \\
         & = \bbE \left[e^{\varepsilon Z - \varepsilon^2 \onevector / 2} \odot \argmax_{v \in \calV} f_Z(\theta, v) \right]                                                                    \\
         & = \bbE \left[e^{\varepsilon Z - \varepsilon^2 \onevector / 2} \odot \argmax_{v \in \calV} \left(\theta \odot e^{\varepsilon Z - \varepsilon^2 \onevector / 2}\right)^\top v \right] \\
         & = \widehat{f}_\varepsilon^{\odot \mathrm{scaled}}(\theta) \neq \widehat{f}_\varepsilon^\odot(\theta)
    \end{align*}
    We could prove in a way similar to \textcite[Proposition 2.2]{berthetLearningDifferentiablePerturbed2020}
    that~$\Omega_\varepsilon^\odot$ is a Legendre type function, which means that
    \begin{equation*}
        \nabla_\theta F_\varepsilon^\odot = \nabla_\theta (\Omega_\varepsilon^\odot)^* = (\nabla_\theta \Omega_\varepsilon^\odot)^{-1}.
    \end{equation*}
    From this, we deduce
    \begin{align*}
        \nabla_\theta F_\varepsilon^\odot(\theta)
         & = \argmax_{\mu \in \bbR^d} \left\{\theta^\top \mu - \Omega_\varepsilon^\odot(\mu)\right\}                                                                       \\
         & = \argmax_{\mu \in \dom(\Omega_\varepsilon^\odot)} \left\{\theta^\top \mu - \Omega_\varepsilon^\odot(\mu)\right\} = \widehat{f}_{\Omega_\varepsilon^+}(\theta).
    \end{align*}
    This time we cannot replace~$\dom(\Omega_\varepsilon^\odot)$ by~$\conv(\calV)$, but we still obtain a similar conclusion:
    \begin{equation*}
        \widehat{f}_\varepsilon^{\odot \mathrm{scaled}}(\theta) = \nabla_\theta F_\varepsilon^\odot(\theta) = \argmax_{\mu \in \dom(\Omega_\varepsilon^\odot)} \{\theta^\top \mu - \Omega_\varepsilon^\odot(\mu)\} = \widehat{f}_{\Omega_\varepsilon^\odot}(\theta).
    \end{equation*}
\end{proof}

\subsection{Inexact oracles} \label{sec:proof_inexact_oracles}

Proof of Proposition~\ref{prop:inexact_oracles}.

\begin{proof}
    We start with additive perturbation.
    By Proposition~\ref{prop:additive_perturbation}, we have:
    \begin{align*}
        J_\theta \widehat{g}_\varepsilon^+(\theta) - J_\theta \widehat{f}_\varepsilon^+(\theta)
         & = \frac{1}{\varepsilon} \bbE \left[g(\theta + \varepsilon Z) Z^\top\right] - \frac{1}{\varepsilon} \bbE \left[f(\theta + \varepsilon Z) Z^\top\right] \\
         & = \frac{1}{\varepsilon} \bbE \left[ \left(g(\theta + \varepsilon Z) - f(\theta + \varepsilon Z)\right) Z^\top\right].
    \end{align*}
    We bound the spectral norm of the error using Jensen's inequality:
    \begin{align*}
        \left\lVert J_\theta \widehat{g}_\varepsilon^+(\theta) - J_\theta \widehat{f}_\varepsilon^+(\theta) \right\lVert^2
         & \leq \frac{1}{\varepsilon^2} \bbE \left\lVert \left(g(\theta + \varepsilon Z) - f(\theta + \varepsilon Z)\right) Z^\top\right\rVert^2 \\
         & = \frac{1}{\varepsilon^2} \bbE \left[ \lVert g(\theta + \varepsilon Z) - f(\theta + \varepsilon Z) \rVert^2 \lVert Z \rVert^2 \right] \\
         & \leq \frac{1}{\varepsilon^2} \lVert g - f \rVert_{\infty}^2 \bbE \left[\lVert Z \rVert^2\right]                                       \\
         & = \frac{d}{\varepsilon^2} \lVert g - f \rVert_{\infty}^2.
    \end{align*}
    We now move on to multiplicative perturbation.
    For multiplicative perturbation, following the same proof starting from Proposition~\ref{prop:multiplicative_perturbation} yields a similar inequality:
    \begin{align*}
        \left\lVert J_\theta \widehat{g}_\varepsilon^\odot(\theta) - J_\theta \widehat{f}_\varepsilon^\odot(\theta) \right\lVert^2
         & \leq \frac{d}{\varepsilon^2 \min_i \lvert \theta_i \rvert^2} \lVert g - f \rVert_{\infty}^2.
    \end{align*}
\end{proof}

\section{More details on applications}

\subsection{Stochastic vehicle scheduling problem}

\subsubsection{Linear program formulation of (VSP):}\label{sec:appendix_deterministicVSP}

The deterministic Vehicle Scheduling Problem (VSP), can be formulated as the following linear program:
\begin{equation}
    \begin{aligned}
        \min & \, \sum_{a\in A} \theta_a y_a &\\
        s.t. & \sum_{a\in \delta^-(v)} y_a = \sum_{a\in \delta^+(v)} y_a, & \forall v \in V\backslash \{o, d\}\\
        & \sum_{a\in \delta^-(v)} y_a = 1, & \forall v \in V\backslash \{o, d\}\\
        & y_a \in \{0, 1\}, &\forall a\in A
    \end{aligned}
    \label{eq:deterministicVSP}
\end{equation}
\noindent Since the constraint matrix is totally unimodular, the constraint polytope gives a perfect formulation, therefore binary constraints can be relaxed, and this problem can be solved using either a flow algorithm or a linear programming solver.

\subsubsection{Two MILP formulations for the StoVSP}\label{sub:solveStoVSP}

\paragraph{MILP formulation}

We denote respectively $c_{\text{vehicle}}$ and $c_{\text{delay}}$ the cost of one vehicle, and the cost of one minute of delay. The StoVSP problem can be formulated as the following MIP with quadratic constraints:
\begin{subequations}    
    \begin{align}
        \min_{d, y} & \,\delaycost\dfrac{1}{|S|}\sum\limits_{s\in S}\sum\limits_{v\in V\backslash\{o,d\}} d_v^s + \vehiclecost \sum\limits_{a\in\delta^+(o)}y_a\\
        \text{s.t.} & \sum_{a\in\delta^-(v)}y_a = \sum_{a\in\delta^+(v)}y_a &\forall v\in V\backslash\{o, d\}\label{eq:delay-constraints}\\
        & \sum_{a\in\delta^-(v)}y_a = 1 &\forall v\in V\backslash\{o, d\}\\
        & d_v^s \geq \gamma_v^s + \sum_{\substack{a\in\delta^-(v) \\ a=(u, v)}} (d_u^s - s_{u, v}^s) y_a & \forall v\in V\backslash\{o, d\}, \forall s\in S\\
        & d_v^s\geq \gamma_v^s & \forall v\in V\backslash\{o, d\}, \forall s\in S\\
        & y_a\in\{0,1\} & \forall a\in A
    \end{align}
\end{subequations}
These quadratic delay constraints \eqref{eq:delay-constraints} can be linearized using McCormick inequalities, and we can then solve the resulting MILP with industrial solvers. This approach does not scale well on instances with large number of tasks and scenarios. 

\paragraph{Column generation formulation}

Another possible formulation for StoVSP is the column generation approach. Let $\mathcal P$ the set of feasible vehicle routes. Let us define the cost of a vehicle route $P\in\mathcal P$ by $c_P^s = \vehiclecost +\delaycost\times \sum_{v\in P} d_v^s$. We introduce decision variables $y_P\in\{0, 1\}$ that equals $1$ if the vehicle route $P$ is chosen. We can then formulate StoVSP as follows:
\begin{equation}
    \begin{aligned}
        \min & \frac{1}{|S|}\sum_{s\in S}\sum_{P\in\mathcal{P}}c_P^s y_P &\\
        \text{s.t.} & \sum_{p\ni v} y_P = 1 & \forall v\in V\backslash\{o, d\} & \quad(\lambda_v\in\mathbb R)\\
        & y_P\in\{0,1\} & \forall p\in \mathcal{P} &
    \end{aligned}
    \label{eq:vsp-column-generation}
\end{equation}
Since there is an exponential number of variables, this formulation can be solved using a column generation algorithm. We denote $\lambda_v$ the dual variables of \eqref{eq:vsp-column-generation}. The associated sub-problem is a constrained shortest path problem :
\begin{equation}
    \min_{P\in\mathcal P} (c_P  - \sum_{v\in P}\lambda_v)
    \label{eq:vsp-subproblem}
\end{equation}
The sub-problem \eqref{eq:vsp-subproblem} can be solved using a stochastic constrained shortest path algorithm (see \cite{parmentierAlgorithmsNonlinearStochastic2019} for in-depth details about these algorithms, and \texttt{ConstrainedShortestPaths.jl}\footnote{\url{https://github.com/BatyLeo/ConstrainedShortestPaths.jl}} for a Julia implementation).

Solving \eqref{eq:vsp-column-generation} using column generation gives good bounds on instances with up to 50 tasks. A Branch-and-price could be used to try to find the optimal solution.

\subsubsection{Data generation}\label{sub:VSPdatagen}

\paragraph{Instance generator}

To generate our dataset, we use a similar approach as by \cite{parmentierLearningApproximateIndustrial2021}. We define a squared map of 50 minutes width, divided in 25 squared ($10\times 10$) districts. Each task $v$ has a (uniformly drawn) start point, a start time $t_v^b$, end point, and end time $t_v^e$.

\begin{figure}[H]
    \centering
    \includegraphics[width=0.5\textwidth]{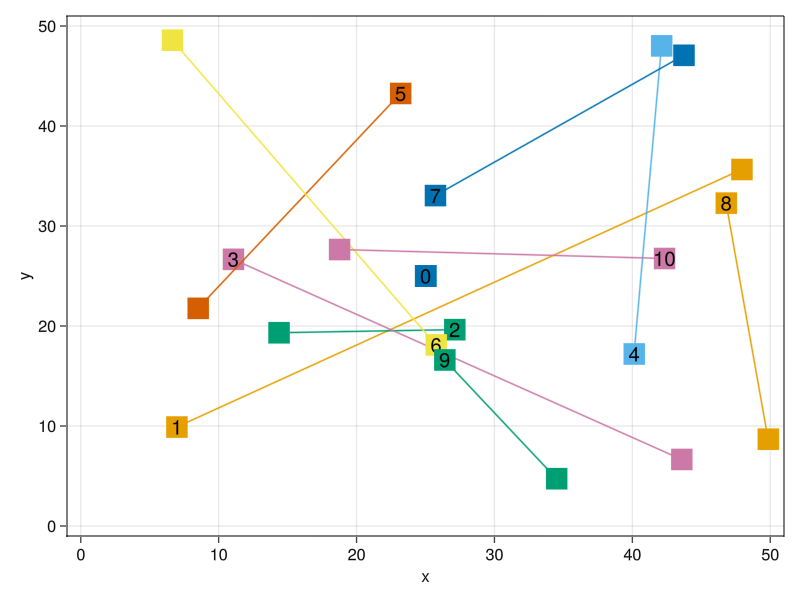}
    \caption{Example of a map with 10 tasks. Task 0 represents the depot, in the middle of the map. Square with task indices are starting points, and tasks without are end points.}
    \label{fig:vsp-instance}
\end{figure}

We draw each scenario $s\in S$ separately and independently. For each district $d$ and hour $h$ of the day, we draw $\beta_{d, h}\sim \log\mathcal{N}(\mu_{d, h}, \sigma_{d, h})$, with $(\mu_{d, h}, \sigma_{d, h})\sim\mathcal{U}([1, 3])\times \mathcal{U}([0.4, 0.6])$. For each district $d$ and hour $h$ of the day, $\zeta_{d,h}^{dis}$ models the congestion in the district at this time, and is computed as follows:

\begin{equation}
    \begin{cases}
        \forall d,\, \zeta^{dis}_{d, 0} = \beta_{d,0}\\
        \forall d,h,\, \zeta^{dis}_{d, h+1} = \frac{1}{2}\zeta^{dis}_{d, h} + \beta_{d,h}
    \end{cases}
\end{equation}
For each hour $h$ of the day, $\zeta^{inter}_h$ models the congestion on roads between districts, and is computed similarly. With $I\sim \log\mathcal{N}(\mu=0.02, \sigma=0.05)$:
\begin{equation}
    \begin{cases}
        \zeta^{inter}_0 = I\\
        \zeta^{inter}_{h+1} = (\zeta^{inter}_{h} + 0.1)I
    \end{cases}
\end{equation}
Let $v$ be a task corresponding to a trip between district $d_1$ and $d_2$. We compute the perturbed start and end times are:
\begin{equation}
    \begin{cases}
        \xi_v^b = t_v^b + \beta_v\\
        \xi_v^e = \xi_v^b + t_v^e - t_v^b + \zeta^{dis}_{d_1,h(\xi_1)} + \zeta^{inter}_{h(\xi_2)} + \zeta^{dis}_{d_2, h(\xi_3)}
    \end{cases}
\end{equation}
with $\xi_1 = \xi_v^b$, $\xi_2 = \xi_1 + \zeta^{dis}_{d_1,h(\xi_1)}$, and $\xi_3 = \xi_2 + t_v^e - t_v^b + \xi^{inter}_{h(\xi_2)}$.
And the perturbed travel times along arcs $a=(u, v)$:
\begin{equation}
    \xi_a^{tr} = \xi_v^b + t_a^{tr} + \zeta^{dis}_{d_1,h(\xi_1)} + \zeta^{inter}_{h(\xi_2)} + \zeta^{dis}_{d_2, h(\xi_3)}
\end{equation}
with $\xi_1 = \xi_u^e$, $\xi_2 = \xi_1 + \zeta^{dis}_{d_1,h(\xi_1)}$, and $\xi_3 = \xi_2 + t_a^{tr} + \xi^{inter}_{h(\xi_2)}$.

\paragraph{Encoding features}\label{appendix:vsp-features}

For an instance, we compute a matrix of 20 features for every arc of the corresponding graph. Let $a=(u, v)\in A$. The first feature is the length of $a$ in minutes, i.e. the deterministic travel time between $u$ and $v$. The second feature is equal to $c_{\text{vehicle}}$ if $a$ is connected to the source (i;.e if $u=o$), else 0. Then, we add the 9 deciles of $\xi_v^b - (\xi_u^e + \xi_a^{tr})$ which represents the slack between $u$
and $v$. Finally, we add the values of the cumulative probability distribution of the slack, evaluated in $-100$, $-50$, $-20$, $-10$, $0$, $50$, $200$, and $500$.
 
\paragraph{Label solutions}

We label each instance with a local search heuristic initialized with the solution of the deterministic problem. Its implementation can be found on the GitHub repository\footnote{\url{https://github.com/BatyLeo/StochasticVehicleScheduling.jl}}.

\subsubsection{Training plots}\label{appendix:vsp-training-plots}

\begin{figure}[H]
    \centering
    \includegraphics[width=0.8\textwidth]{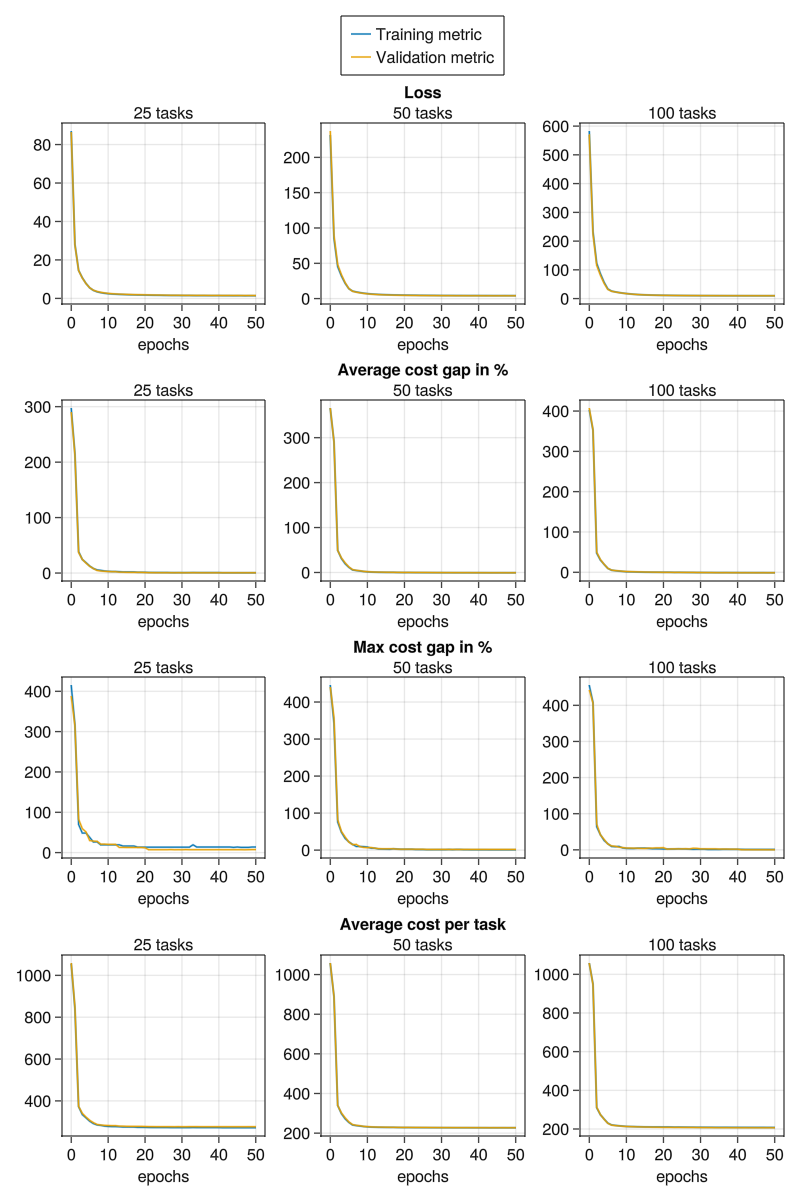}
    \caption{Learning by imitation metric evolution during learning}
\end{figure}

\begin{figure}[H]
    \centering
    \includegraphics[width=0.8\textwidth]{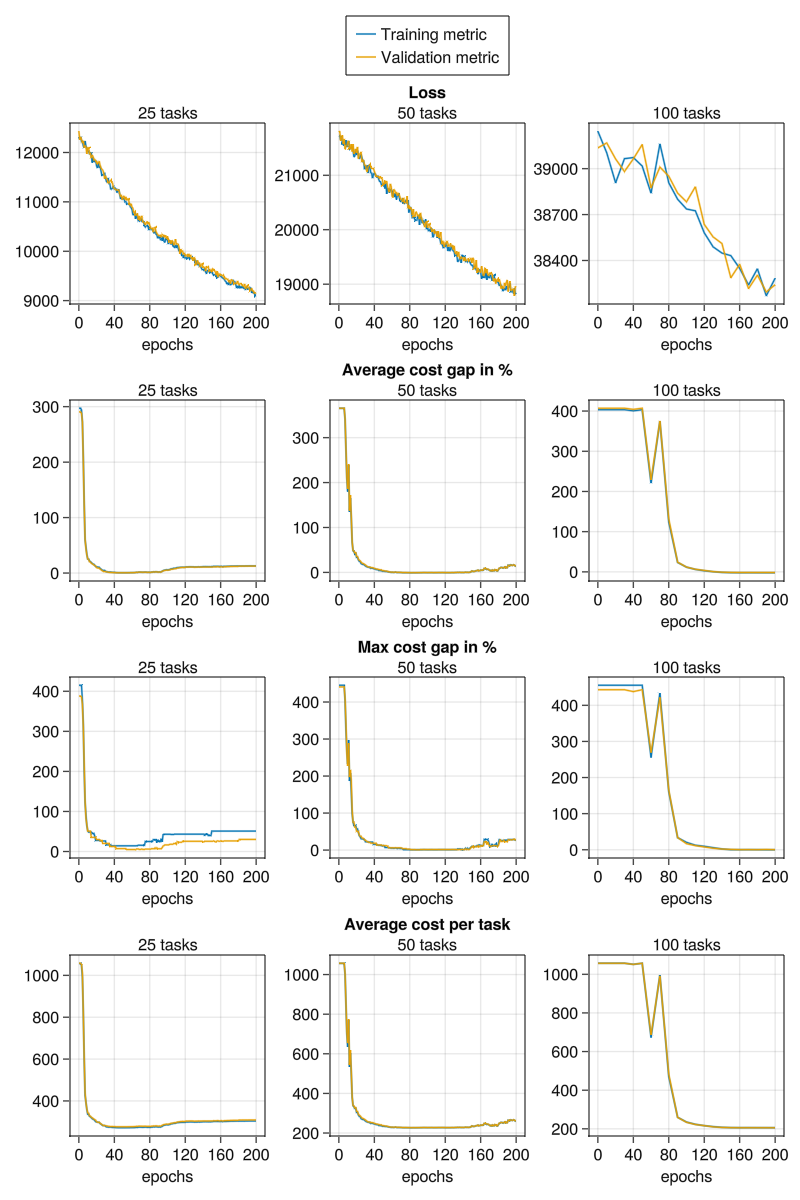}
    \caption{Learning by experience metric evolution during learning}
\end{figure}

\subsection{Two-stage minimum weight spanning tree}\label{sub:two-stage-spanning-tree}

This subsection contains additional material about the two-stage stochastic spanning tree problem.

\subsubsection{Lagrangian heuristic}\label{sub:appendix_lagrangian_heuristic}

We present the heuristic based on Lagrangian relaxation used to label our datasets.

\paragraph{Lagrangian relaxation}

An option to compute a lower bound on the optimal value of the \emph{two-stage minimum weight spanning tree} is to apply a Lagrangian relaxation. For this, we duplicate complicating variables $y$ for each scenario and introduce associated $\theta_{es}$ dual variable for each new constraint:

\begin{equation}    
    \begin{array}{rll}
        \min\,& \displaystyle\sum_{e \in E}c_e y_e + \frac{1}{|S|}\sum_{e \in E}\sum_{s \in S}d_{es}z_{es}\\
        \mathrm{s.t.}\,&\displaystyle\sum_{e \in E} y_e + z_{es} = |V|-1 & \forall s \in S \\
        &\displaystyle\sum_{e \in E(Y)} y_e + z_{es} \leq |Y| - 1, \qquad &\forall Y,\, \emptyset \subsetneq Y \subsetneq E,\, \forall s \in S \\
        & \textcolor{red}{y_{es} = y_e,}\, & \textcolor{red}{\forall e\in E,\, \forall s\in S \text{ ($\theta_{es}$ dual variables)}}\\
        & y, z \in \{0,1\}
    \end{array}
\end{equation}

\noindent Lagrangian dual problem can be formulated as follows:
\begin{equation}
    \max_\theta G(\theta) = \left\{\begin{aligned}
        \min_y\quad & \sum_{e\in E}(c_e + \dfrac{1}{|S|}\sum_{s\in S}\theta_{es})y_e\\
        \mathrm{s.t.}\quad & 0 \leq y\leq M
    \end{aligned}
    \right. + \frac{1}{|S|}\sum_{s\in S}\left\{\begin{aligned}
        \min_{y_s, z_s}\quad & \sum_{e\in E} d_{es}z_{es} - \theta_{es}y_{es}\\
        \mathrm{s.t.}\quad & y_s + z_s \in \text{spanning tree polytope}
    \end{aligned}\right.,
\end{equation}
where $M$ is a large constant. In theory $M$ is not needed, we can take $M=+\infty$, but taking a finite $M$ leads to more informative gradients. $G$ can be maximized by gradient ascent:
\begin{equation}
    (\nabla_\theta G(\theta))_{es} = \frac{1}{|S|}(y_e - y_{es})
\end{equation}

\paragraph{Lagrangian heuristic}

Once we find a good $\theta$ with the Lagrangian relaxation, we can retrieve the corresponding $y_{es}$ solution, and use it as input of the Lagrangian heuristic \ref{alg:lagrangian-heuristic}. The general idea is to rebuild a forest from the edges that are selected in the first stage for most scenarios.

\SetKwComment{Comment}{/* }{ */}
\begin{algorithm}[H]
    \caption{Lagrangian heuristic}\label{alg:lagrangian-heuristic}
    \KwData{$y_{es}\in \{0, 1\},\, \forall e\in E,\, \forall s\in S$}
    \ForAll{$e \in E$}{
\eIf{$\frac{1}{|S|}\sum_{s\in S} y_{es} > 0.5$}{
    $w_e \gets 1$\;
    }{
        $w_e \gets -1$\;        
        }
        }
        \text{Find maximum weight forest for weights $w$}\;
        \KwResult{$y_e\in \{0, 1\},\, \forall e\in E$}
    \end{algorithm}
            
In the numerical experiments, we use the Lagrangian heuristic to build the training set. Practically, we stop the Lagrangian relaxation after 50.000 iterations or if the gap between the lower bound provided by the $G(\theta)$ and the upper bound obtained from the heuristic is non-grater than 0.1\%.

\subsubsection{Instance features}\label{subsec:two-stage-spanning-tree-features}

We detail here all the features used for training.

\paragraph{Basic feature set}
Basic features only contain twelve features. Let $e$ be an edge. The first one is the first stage cost $c_e$, and the other 11 are the quantiles of the second stage costs $(d_{es})_{s\in S}$ for values $\{0,\, 0.1,\, 0.2,\, 0.3,\, 0.4,\, 0.5,\, 0.6,\, 0.7,\, 0.8,\, 0.9,\, 1\}$. All quantiles in the advanced feature set use the same values.

\paragraph{Advanced feature set}
Advanced features contain 79 features in total. Let $e$ be an edge. The first 12 features are the same as the basic features. To these, we add quantiles of best stage costs (\textit{i.e} for each scenario $s$ the value $\min(c_e, d_{es})$), quantiles of neighbors first stage costs, and quantiles of neighbors second stage costs. Then, we compute several single stage minimum spanning trees with Kruskal's algorithm: one with first stage costs, one for each scenario with corresponding second stage costs, and one for each scenario with corresponding best stage costs $\min(c_e, d_{es})$. We use these spanning trees to compute new useful features: a binary features that tells if $e$ is in the first stage spanning tree, and quantile features on its presence in second stage trees. Finally, we use best stage spanning trees to compute quantiles for the cost of $e$ in the best stage costs spanning tree: one set of quantiles for first edge costs, and another set for second stage costs.

\end{document}